\documentclass[journal,onecolumn]{IEEEtran}

\makeatother
\usepackage{graphicx}
\usepackage{tikz}
\usepackage{color}
\usepackage{amsmath, amsfonts, amssymb, mathrsfs}
\usepackage[amsmath,thmmarks]{ntheorem}
\usepackage{subfigure}
\usepackage{epstopdf}
\usepackage{pgfplots}
\usepackage{arydshln}
\usepackage{mathabx}
\usepackage{nth}
\usepackage{cite}
\usepackage{bbm}
\usepackage{bm}
\usepackage{algorithm}
\usepackage{algorithmic}

\makeatletter
\renewcommand*\env@matrix[1][*\c@MaxMatrixCols c]{%
  \hskip -\arraycolsep
  \let\@ifnextchar\new@ifnextchar
  \array{#1}}

\theoremstyle{plain}
\theoremheaderfont{\itshape\bfseries}
\theorembodyfont{\normalfont\itshape}
\theoremseparator{:}
\theoremsymbol{}
\newtheorem{theorem}{Theorem}
\newtheorem{lemma}[theorem]{Lemma}

\newtheorem{proposition}[theorem]{Proposition}

\newtheorem{assum}{Assumptions}
\newtheorem{definition}{Definition}

\theoremstyle{nonumberplain}

\theoremstyle{plain}
\theoremheaderfont{\itshape\bfseries}
\theorembodyfont{\normalfont}
\theoremseparator{:}
\theoremsymbol{}
\newtheorem{remark}{Remark}

\theoremstyle{plain}
\theoremheaderfont{\itshape\bfseries}
\theorembodyfont{\normalfont}
\theoremseparator{:}
\newtheorem{example}{Example}

\theoremstyle{nonumberplain}
\theoremsymbol{\rule{1ex}{1ex}}

\newtheorem{proof}{Proof}
\newlength\fheight
\newlength\fwidth
\newcommand{\abs}[1]{\left| #1 \right|}
\newcommand{\norm}[1]{\| #1 \|}
\newcommand{\inn}[2]{\langle #1,#2 \rangle}
\newcommand{\lrbrace}[1]{\left\{ #1 \right\}}

\newcommand{\normi}[1]{{\left\vert\kern-0.25ex\left\vert\kern-0.25ex\left\vert #1 
    \right\vert\kern-0.25ex\right\vert\kern-0.25ex\right\vert}}
\newcommand{\inni}[2]{{\langle\kern-0.25ex\langle #1,#2
   \rangle\kern-0.25ex\rangle}}

\def\X{\mathcal{X}}

\def\Erw{\mathbb{E}}
\def\rmr{\mathrm{r}}

\def\rmD{\mathrm{D}}
\def\rmf{\mathrm{f}}
\def\tildrmr{\tilde{\rmr}}
\def\overrmr{\overline{\rmr}}
\def\tildrmT{\tilde{\rmT}}
\def\hatrmT{\hat{\rmT}}

\def\rmQ{\mathrm{Q}}
\def\tildrmQ{\tilde{\rmQ}}
\def\overrmQ{\overline{\rmQ}}

\def\rmP{\mathrm{P}}
\def\rmH{\mathrm{H}}
\def\tildrmP{\tilde{\rmP}}
\def\overrmP{\overline{\rmP}}

\def\rmV{\mathrm{V}}
\def\tildrmV{\tilde{\rmV}}

\def\tildrmQ{\tilde{\rmQ}}

\def\nat{ \mathbb{N} }								

\def\real{ \mathbb{R} }								

\def\A{ \mathcal{A} }
\def\G{ \mathcal{G} }

\def\S{ \mathcal{S} }               

\def\X{\mathcal{X}}
\def\Y{\mathcal{Y}}

\def\Loc{ \textnormal{lo} }

\def\GAQL{ \textnormal{GAQL} }
\def\LAQGI{ \textnormal{LAQGI} }	
\def\Glob{ \textnormal{gl} }
\def\LQ{ \textnormal{LQ} }
\def\GQ{ \textnormal{GQ} }
\def\EI{ \textnormal{EI} }

\def\rmT{\mathrm{T}}

\def\DMDP{\mathfrak{DM}}

\def\A{\mathcal{A}}

\DeclareMathOperator{\Unif}{\textnormal{Uni}}

\DeclareMathOperator{\supp}{\text{supp}}

\DeclareMathOperator{\Boltz}{\textnormal{Bolt}}
\DeclareMathOperator{\Softmax}{\textnormal{Soft}}
\DeclareMathOperator{\Greed}{\textnormal{Greed}}

\DeclareMathOperator*{\argmax}{arg\,max}

\newcommand\restr[2]{{
  \left.\kern-\nulldelimiterspace 
  #1 
  \right|_{#2} 
  }}

\newcommand{\F}{\mathcal{F}}
\newcommand{\tildF}{\tilde{\F}}

\title{On Information Asymmetry in Multi-Agent Reinforcement Learning: Convergence and Optimality}
\author{
Ezra~Tampubolon$^\dagger$, Haris~Ceribasic$^\dagger$
  Holger~Boche$^\dagger$*,\\
  $^\dagger$Technische Universit{\"a}t M{\"u}nchen, Lehrstuhl f{\"u}r Theoretische Informationstechnik\\
  *Munich Center for Quantum Science and Technology (MCQST)\\
  \{ezra.tampubolon,haris.ceribasic,boche\}@tum.de
}

\begin{document}

%
\maketitle
%
\begin{abstract}
In this work, we study the system of interacting non-cooperative two Q-learning agents, where one agent has the privilege of observing the other's actions. We show that this information asymmetry can lead to a stable outcome of population learning, which generally does not occur in an environment of general independent learners and that the resulted post-learning policies are almost optimal in the underlying game sense, i.e. they form a Nash equilibrium. Furthermore, we propose in this work a Q-learning algorithm, requiring predictive observation of two subsequent opponent's actions, yielding an optimal strategy given that the latter applies a stationary strategy, and discuss the existence of the Nash equilibrium in the underlying information asymmetrical game.
\end{abstract}

\section{Introduction}
\paragraph*{Information asymmetry in applications} In widespread multi-agent systems, information distribution is often asymmetrical, meaning that some agents have more or better information than the other. This property has been a subject of extensive study in economics resulting in the characterization of undesirable consequences such as market failure, moral hazards, monopoly of information, and adverse selection; and in the mechanisms avoiding those occurrences \cite{Aboody2000}. 
Likewise, information asymmetry arises in technical applications usually as an effect of hierarchical structures and cross-layer perspectives, which grows in importance with the systems' increasing complexity and growing interlinkage enabled by groundbreaking infrastructures, such as 5G and IoT. A specific example of an asymmetrical information relationship is that between the base stations (BSs) and the (mobile) users (USs) in a wireless communication system, where BSs often has (implicit) information about USs' service request, while USs might not know about the BS service allocation. 
Among USs themselves, information asymmetry might also occur, due to the decision-making order, such as in the setting of primary user (PU) and secondary users (SUs) in a cognitive radio network 
\cite{Adlakha2013,ZhangChen2017}.
 Another possible occurrence of information asymmetry is in the relation between defender and attacker in security systems \cite{AlpcanBasar2010,YangXue2013,XuRabinovich2015,HeIslam2017,EtesamiBasar2019}. Therein, the attacker might observe the defender's action, while the latter is unaware of the former's action but suffers the consequences. The reverse case might also occur in practice: the defender can observe the attacker's action while the attacker can only observe her action's impact on the defender.

\paragraph*{Reinforcement learning} 
In recent years, machine learning (ML) techniques have gained significant importance in academia and industry. 
Reinforcement learning (RL) \cite{SuttonBarto2018,Kaelbling1996} is a ML paradigm suited for dynamical applications. 
It allows a single agent to learn a reward maximizing policy in an unknown Markovian environment, arising naturally in various domains, such as robotics, telecommunications, economics. 
One fundamental technique in RL is the so-called Q-learning. Q-learning explores and exploits the state-action space and generate the so-called optimal Q-function, giving rise to the greedy deterministic strategy optimizing the accumulated discounted reward. Q-learning has been successfully adapted in several applications, reaching from single-device systems \cite{OrtizLi2016,Mastronade2011,Moody1998} to networked multi-device systems, found, e.g., in wireless communication \cite{BennisNiyato2010,Simsek2011,AmiriFridman2018,Ghadimi2017,Calabrese2018}, wireless sensor networks \cite{Sharma2020}, and edge computing \cite{ZengPan2019}. Furthermore, the Q-learning constitute a foundation of other RL algorithm, such as the policy iteration method, where the corresponding policy evaluation step is similar to the Q-learning iterate. Therefore, results concerning to the specific Q-learning paradigm enrich the general field of reinforcement learning. 

\paragraph*{Multi-Agent Learning} 
In the recent years, one can observe the increasing tendency of increasing interconnection between heterogeneous autonomous systems enabled by groundbreaking infrastructures, such as 5G and IoT, so that we literally have the case that no intellectual system is an island entire of itself. Therefore, although dramatic progresses in the field of AI also observable in the recent years, the applicability of AI techniques will still be limited until we understand the inter-agent interplay of the intelligent systems. The latter is far from being trivial, as properties of a single-agent intelligent system does in general not immediately transfer to its multi-agent extension. For instance, methods for a single-agent RL does in general not suit for multi-agent applications since the assumptions (stationarity of reward and state transition dynamic) are not longer valid (see e.g., \cite{zhang2019multiagent}), and thus desired algorithmic properties, such as the convergence of the learning, no longer hold in general. In this respect, the study of interconnected learning agents may help one to recognize possible pitfalls in the entire system, and inspire e.g., the design of an efficient mechanisms.


\paragraph*{Problem formulation}
 As motivated before, our interest is on the setting of competitive Q-learning multi-agent systems which is subject to asymmetrical information distribution. Specifically, we study the setting of two Q-learning non-cooperative agents, where one of them has the privilege of observing the other's actions.
The first question which one might ask is the following:
\begin{center}
\textit{How does the information asymmetry impact the outcome of multi-agent Q-learning?}
 \end{center}
In order to answer this question, there are two subsequent aspects needed to be investigated: First, the convergence of the non-cooperative learning schemes, and second, the behaviour of the limit Q-table and the resulted greedy strategy. The first aspect does not immediately follow from the well-known convergence of the single agent Q-learning is well-known, because for the general setting of independent Q-learners, i.e., Q-learners who can not observe the other's action, the learning outcome may not stable \cite{Tan1993,Claus1998}. This issue arises since the agents assume a stationary environment in the Q-learning phase, while other agents' influence makes her reward time-changing. In investigating the second aspect, our attention lies in the question of whether the informational advantage (resp. disadvantage) of the globalized (resp. localized) agent has a positive (resp. negative) impact on her. Moreover, we aim to investigate in face of the informational asymmetry, whether the learning outcome yields a reciprocal optimum strategy of the players, i.e., whether it is a Nash equilibrium of the underlying game.

\subsection*{Our Contributions}
In this work, we consider the information asymmetrical LA-GA Q-learning, where LA updates the Q-table as in the single-agent case with the possible influence of other agent in her obtained reward and the system state transition, and where GA updates the Q-table in dependence of LA's action. We call the corresponding method as LA Q-learning with globalized interference (LAQGI) and GA Q-learning (GAQL). 

At first, we show the convergence of this reciprocal Q-learning algorithm, and thus show the positive aspect of the informational asymmetry as the factor which can foster the stability of the two agent Q-learning. 

Furthermore, we provide theoretical and numerical analysis of the policies that result from the mentioned convergent multi-agent Q-learning. In this respect, our main result is the insight that the information asymmetry manifests itself in the outcome of the learning as follows: The LA generates via LAQGI greedy policy that is optimal given that the GA applies the long-term training policy, and the GA generates via GAQL a greedy policy that is optimal given that LA generates a greedy post-learning policy. Consequently, the informational disadvantage of LA causes this agent to choose a post-learning policy that is sub-optimal, as GA likely applies the corresponding greedy policy rather than the training policy. However, despite of the sub-optimality of LA's post-learning strategy due to the model mismatch, we are able to show that the greedy post-learning policy is almost optimal in the game sense, meaning that no agent have incentives to deviate from this strategy.

As the theory of single-agent MDP asserts one to apply a greedy strategy for optimal actions, one can expect that LA applies a greedy strategy in the post-learning phase. Therefore, GAQL would be fine enough for GA. However, we provide for the sake of completeness a learning algorithm, called extra information GA Q-learning (EIGAQL), which provides GA an optimal strategy given that LA applies a not-necessarily-greedy strategy. This advantage comes, as suggested by the name, with cost of additional information requirement: GA needs to observe the actual and next action of LA. Furthermore, as we are only able to generate an almost solution of the LA-GA game, we also provide for the sake of completeness in this work an existence Theorem for the indeed solution, i.e., Nash equilibrium, of this game.

Complementary to the theoretical results, we provide in this work some numerical simulations. Further, we numerically compare the performance of the proposed reciprocal information asymmetrical multi-agent learning with the independent Q-learning, where both agents is unaware of others' action, and the jointly cooperative Q-learning, where both agents know about the action of others and cooperatively updating the Q-table giving rise to the population-welfare post-learning policy.

\subsection*{Relation to Prior Works} 
\textbf{Learning in Games:} 
Our work is related to the works investigating the dynamic of agents in a competitive game setting. In particular, among them are those closely related to ours, which generate long-term results with different agent types. The latter includes no-learning agents, e.g., greedy agents with best-response dynamics, and learning agents, e.g., the fictitious playing, the gradient playing, and the online-learning playing (see e.g., \cite{Tampubolon2019Agg,Tampubolon2020Pricing,Tampubolon2020Coord} and the references therein).
For a further comprehensive review of the literature on those topics, we refer to \cite{Fudenberg1998,zhang2019multiagent}. 

\textbf{Multi-Agent Reinforcement Learning:} Particular research field of learning in games relevant to our work is the field of multi-agent reinforcement learning (for an excellent overview see \cite{Busoniu2008,zhang2019multiagent,LuYan2020}). Similarly to the single-agent case, the solutions proposed in this field can basically be categorized into value-based, policy-based, and linear programming based solution. Respective to this distinction, our work can be classified into value-based approach. As the body of the literature on multi-agent reinforcement learning is enormous huge, we review in the following only some value-based approaches closely related to our work.  Most of the works on this kind of approach (or more generally multi-agent reinforcement learning) concern with the setting where the iterate of a learner requires fully knowledge about other learners' action and even their further characteristics such as the obtained rewards and the value function iterates. One prototypical examples of such works are those \cite{Shapley1953,Watkins1992,Littman1994,Hu2003} proposing the minimax Q-learning ensuring the convergence of the iterate of the competitive reinforcemet learners to the Nash equilibrium. Several extensions of the minimax Q-learning have been given in the literature. To name a few: \cite{HuGao2015} proposes a negotiation process to reduce the amount of exchanged information in the minimax Q-learning; \cite{ConitzerSandholm2007} proposes an algorithm which is optimal against eventually stationary agents; \cite{Huang2020ModelAR} proposes a minimax algorithm for risk-averse reinforcement learners. Above mentioned multi-agent Q-learning methods require high degree of coordination of the agents. Therefore, one can clearly not expect that such solutions is realizable in practice. Another drawback of the solutions proposed in the above mentioned work is that they usually make use of non-elementary operation, such as computing at each time step the Nash equilibrium of an additional stage game induces by the Q-iterate. However, there is one principal difference between such works and ours: While the focus of the mentioned works is on designing algorithms converging to the solution concept of the Nash equilibrium of the underlying Markov game, our focus is rather to investigate the repeated game outcome of Q-learning agents with information asymmetry. Nevertheless, we are able to show that in some sense the outcome of the learning Q-learning is an almost solution of the underlying game.

\textbf{Independent Learning:} The most basic approach in multi-agent reinforcement learning is the independent learner approach which concern with learners each applying the single-agent reinforcement learning. This kind of approach is relevant to ours as the proposed method here is an extension of this learning paradigm for the information asymmetric case where one agent can observe the action of others. For the general setting of non-cooperative independent Q-learners, the learning outcome is believed to not be stable \cite{Tan1993,Claus1998}, although it can be proven that a stable solution of the underlying game exists \cite{Fink1964}. This issue arises since the agent assumes a stationary environment in the Q-learning phase, while other agents' influence makes her reward time-changing. One way to avoid this occurence is by considering two time-scale method, where in one scale, the agents each apply constantly a policy aiming at the first place to learn the best reply given the opponent strategy, and where in the other scale the agents adapt their strategy by the learning experience  \cite{ArslanYuksel2017,ArslanYuksel2019}. This kind of method clearly requires, due to the different time-scale, a higher degree of coordination than the independent learning, which is in general not given in the applications.

\textbf{Asymmetry in Multi-Agent Reinforcement Learning:} The asymmetrical information structure as considered in our work induces a particular ordering of the agents' game play, i.e., GA acts after LA. A game with such kind of asymmetry is called the Stackelberg game. In the context of the multi-agent reinforcement learning, the Stackelberg game has first been studied in \cite{kononen2004asymmetric}, proposing Q-learning method based on minimax approaches given by \cite{Shapley1953,Watkins1992,Littman1994,Hu2003} for achieving the Nash equilibrium of the underlying game. The Stackelberg game in multi-agent reinforcement learning has been successfully applied in several applications, such as robotics\cite{pinto2017asymmetric,warrington2020robust}, security \cite{vorobeychik2012computing,vasal2020model}, and wireless network \cite{jia2018stackelberg,xiao2018two}.
In Stackelberg game, it is usually assumed that all agents have fully knowledge of others' action contrasting to our work as we consider the additional asymmetry in action's observability. Lastly, we mention that there are another type of information asymmetry considered in the literature, such as the case where an agent in contrast to others has (local) information about the underlying system \cite{HeDai2016,vasal2020model}. However, such approaches, in contrast to ours, usually requires additional assumption on the system dynamic, such as the existence of post-decision state where some system characteristics are revealed, or additional prediction step.
\subsection*{Structure of Our Work}
Our work is structured as follows:
\begin{itemize}
    \item In Section \ref{Sec:Model}, we provide for sake of completeness basic notions on Markov decision process and Q-learning. Moreover, we also provide in this section a formal description of the LA-GA information asymmetric stochastic game of our interest.
    \item In Section \ref{Sec:ConvThm}, we introduce LA Q-learning with globalized interference (LAQGI) and GA Q-learning (GAQL) constituting canonical extensions of the single-agent Q-learning to the aforementioned LA-GA game. Moreover, we show that joint application of LAQGI and GAQL leads to convergent learning process both agents.
    \item In Section \ref{Sec:OptProof}, we discuss the outcome of the joint application of LAQGI and GAQL. Specifically, we analyze theoretically the joint performance of the greedy algorithms resulted from those algorithms. Here, we show that LAQGI and GAQL yields an almost solution concept, i.e. Nash equilibrium, of the underlying game.
    \item Motivated by the fact that GAQL only provides an optimal greedy strategy for GA provided that LA applies a stationary deterministic strategy, our aim in Section \ref{Sec:ajajahshsgsgssssss} is to provide an alternative GA Q-learning method yielding an optimal greedy strategy for GA even if LA applies a not-necessarily deterministic but stationary strategy. There, we are able to show that the latter can be generated by allowing GA to have an additional observation of LA's action, and propose the so-called extra information GA Q-learning (EIGAQL).  
    \item As LAQGI and GAQL only yield an almost Nash equilibrium of the underlying LA-GA game, we provide in Section \eqref{Sec: Nash}, the answer to the question whether an indeed solution concept exists.
    \item Finally, we provide and discuss in Section \ref{Sec:Num} some numerical simulations  which support our theoretical findings and also give additional insight into the mechanism introduced in this work.
\end{itemize}
\subsection*{Basic Notions and Notations}
Let $\X$ and $\Y$ be a finite sets. We denote the set of probability density on $\X$ by $\Delta(\X)$
, i.e.:
\begin{equation*}
\Delta(\X):=\lrbrace{p:\X\rightarrow [0,1]:~\sum_{x\in\X}p(x)=1}.
\end{equation*}  
We write the set of Markov kernel with source $\X$ and target $\Y$ by $\Delta_{\X}(\Y)$, i.e.:
\begin{equation*}
\Delta_{\X}(\Y):=\lrbrace{	\begin{aligned}
		&p:(\X,\Y)\rightarrow[0,1],\\
	&(x,y)\mapsto p(y|x)
		\end{aligned}	:~\sum_{y\in\Y}p(y|x)=1,~\forall x\in\X}
\end{equation*}
Given two vectors (or matrices) $x,y$ having the same dimensions, we denote the entrywise multiplication of $x$ and $y$ by $x\odot y$

\section{Model Description}
\label{Sec:Model}
In this section, we recall some basic notions for the setting of our investigations, i.e., the setting of the local-global competitive reinforcement learners. As the reinforcement learning generally concerns with the setting of the so-called Markov decision process (MDP), we briefly recall, for sake of completeness, the latter in Subsection \ref{Subsec:MDP}. Subsequently, we extend the notion of MDP in \ref{Subsec:aajajsshshsggsgshddddssss} to the information asymmetrical game setting of our interest.   
\subsection{Single Agent Markov Decision Process}
\label{Subsec:MDP}

To begin with, we first recall the setting of a \textit{\textbf{Markov decision process (MDP)}} of a single agent \cite{Puterman1994}:
\begin{definition}
A Markov decision process is defined as the tuple $(\S,\A,\rmr,\rmP)$, where $\S$ and $\A$ are finite sets, $\rmr:\S\times \A\rightarrow\real$, and $\rmP\in\Delta_{\S \times\A}(\S)$.
\end{definition}
 MDP serves as a model for decision making of an agent in a dynamical uncertain system. In this regard, $\S$ stands for the set of all possible system \textit{\textbf{states}} and $\A$ for the set of all possible executable \textit{\textbf{actions}} of the agent in the system. Moreover, $\rmr(s,a)$ stands for the reward received by the agent after executing the action $a\in\A$ given that the system is in the state $s\in\S$. This interpretation of $\rmr$ gives rise to the naming "\textit{\textbf{reward function}}". Lastly, $\rmP$ describes the dynamic, by assigning the probability $\rmP(s'|s,a)$ for the change of the system state to $s'\in\S$ given that the agent executes action $a\in\A$, and that the actual system state is $s\in\S$ .
 
 In the repeated setting, the aim of the agent in MDP is to determine a \textit{\textbf{policy}} $\pi\in\Delta_{\S}(\A)$ optimizing the obtained reward. Specifically, the agent uses the policy $\pi$ as follows: $\pi(a_{t}|s_{t})$ stands for the probability for choosing the action $a_{t}\in\A$ at time $t\in\nat_{0}$ given that the system is in the state $s_{t}\in\S$. Of course, one may, in investigating a MDP, consider a more general class of policy, such as the time-varying policies, which in contrast to the policies in $\Delta_{\S}(\A)$ depend on the state-action history. However, it is usually enough to consider the class of (stationary Markov) policy $\Delta_{\S}(\A)$ (see, e.g., \cite{Puterman1994}). 
Now, let us further discuss the policies in $\Delta_{\S}(\A)$.
One of the important subclass of such policies is the class of deterministic policy: A \textit{\textbf{deterministic policy}} $\pi\in\Delta_{\S}(\A)$ is a policy satisfying for every $s\in\S$ $\pi(a|s)=1$ for an $a\in\A$. A deterministic policy $\pi$ can be seen, with slight abuse of notation, as a deterministic function $\pi:\S\rightarrow\A$. A type of deterministic policy of particular interest in MDP is the \textit{\textbf{greedy policy}} $\pi$ w.r.t. $\rmQ\in\real^{\S\times\A}$ defined as $\pi(s)\in\argmax_{a\in\A}Q(s,a)$.

We measure the performance of a policy $\pi\in\Delta_{\S}(\A)$ in a MDP $(\S,\A,\rmr,\rmP)$ by the so-called \textit{\textbf{value function}} $\rmV_{\pi}:\S\rightarrow\real$ defined as:
\begin{equation*}
\rmV_{\pi}(s)=\Erw\left[\left. \sum_{t=0}^{\infty}\beta^{t}\rmr(S_{t},A_{t})\right| \begin{aligned}&S_{0}=s,S_{t+1}\sim P(\cdot|S_{t},A_{t}),\\
& A_{t}\sim\pi(\cdot|S_{t}) \end{aligned}\right],
\end{equation*}
where $\beta\in(0,1)$ is a chosen \textit{\textbf{discount factor}}. As described above, $\rmV_{\pi}(s)$ specifies the expected discounted accumulated reward of the agent in the infinite horizon provided that she follows the strategy $\pi$ to choose her action and that the initial system state is $s$.
We refer the MDP with the discount factor $\DMDP:=(\S,\A,\rmr,\rmP,\beta)$ to as the \textit{\textbf{discounted MPD}}, and $\rmV_{\pi}$ as to the value function of $\pi$ in $\DMDP$. Closely
related to the value function, is the following quantity called the Q-function of the policy $\pi$ in $\DMDP$, defined as 
\begin{equation*}
\rmQ_{\pi}(s,a)=\Erw\left[\left. \sum_{t=0}^{\infty}\beta^{t}\rmr(S_{t},A_{t})\right| \begin{aligned}&S_{0}=s,~A_{0}=a,\\
&A_{t}\sim\pi(\cdot|S_{t}) ,\\
&S_{t+1}\sim P(\cdot|S_{t},A_{t})\end{aligned}\right]. 
\end{equation*}
$\rmQ_{\pi}(s,a)$ computes the expected discounted accumulated reward of the agent in the infinite horizon provided that she applies the initial action $a$ and follows the strategy $\pi$ to choose her subsequent actions, and that the initial system state is $s$. The specific relation between value function and Q-function is given by $\rmV_{\pi}(s)=\Erw_{A\sim\pi(\cdot|s)}[\rmQ_{\pi}(s,A)]$.

The goal of the discounted MDP is to optimize the value function w.r.t. the policy. In this direction, it is convenient to consider the so called \textit{\textbf{optimal value function}} of $\DMDP$ given by:
\begin{equation*}
 \rmV_{*}(s):=\max_{\pi\in\Delta_{\S}(\A)}\rmV_{\pi}(s).
 \end{equation*}
It is sometimes convenient to consider Q-function corresponds to $\rmV_{*}$, i.e., the function $\rmQ_{*}:\real^{\S\times\A}\rightarrow\real$ given by:
\begin{equation*} 
\rmQ_{*}(s,a):=\max_{\pi\in\Delta_{\S}(\A)}\rmQ_{\pi}(s,a),
\end{equation*}
We refer $\rmQ_{*}$ to as the \textit{\textbf{optimal Q-function}} for $\DMDP$.
If we know $\rmQ_{*}$, then the \textit{\textbf{optimal policy}} $\pi_{*}$ for $\DMDP$, i.e. the policy satisfying $\rmQ_{*}=\rmQ_{\pi_{*}}$ and $\rmV_{*}=\rmV_{\pi_{*}}$
is the greedy policy w.r.t. $\rmQ_{*}$. 

\subsection{MDP and the Bellman Equations}
Given a discounted MDP $\DMDP=(\S,\A,\rmr,\rmP,\beta)$. For analysis of the value function $\rmV_{\pi}$ of a policy $\pi\in\Delta_{\S}(\A)$ it is useful to describe it implicitly as a solution of the so-called Bellman equation. Specifically, one can show (see Theorem 6.1.1 in \cite{Puterman1994}), that 
$\rmV_{\pi}$ is the unique solution of: 
\begin{equation}
\label{Eq:aajshssgsgsgsffsfss}
\rmV_{\pi}(s)=\Erw_{A\sim \pi(\cdot|s)}\left[ \rmr(s,A)+\beta\Erw_{S'\sim P(\cdot|s,A)}[\rmV_{\pi}(s')\right].
\end{equation}
Similarly, the $Q_{\pi}$-function of $\pi$ is the unique solution of equation:
\begin{equation}
\label{Eq:aaksjsjshhsgsgsggshsgsgsggsgs}
\begin{split}
\rmQ_{\pi}(s,a)=\rmr(s,a)+\beta \Erw_{S'\sim \rmP(\cdot|s,a)}\left[ \Erw_{A'\sim \pi(\cdot|S')}\left[ \rmQ_{\pi}(S',A')\right] \right]. 
\end{split}
\end{equation}

For optimal value function $\rmV_{*}$ of $\DMDP$ we have also implicit description similar to the previous one for the value function of a policy. Specifically, it holds that $\rmV_{*}$ is the unique solution of the equation:
\begin{equation}
\label{Eq:ajajajsgsgsffssddafaaddaadsssdss}
\rmV_{*}(s)=\max_{a\in\A}\left[ \rmr(s,a)+\gamma\sum_{s'\in\S}\rmP(s'|s,a)\rmV_{*}(s')\right]. 
\end{equation}
Moreover for the corresponding Q-function $Q_{*}$, it holds that it is the unique solution of the equation:
\begin{equation}
\label{Eq:akjajsjshgdgdgdggdgdgdgdd}
\rmQ_{*}(s,a)= \rmr(s,a)+\gamma\Erw_{S'\sim\rmP(S'|s,a)}\left[ \max_{a'\in\A}\rmQ_{*}(S',a')\right].
\end{equation}

Working with a discounted MDP $\DMDP$ and a policy $\pi\in\Delta_{\S}(\A)$, it is convenient to utilize the operator $\rmT_{\pi}:\real^{\S\times\A}\rightarrow \real^{\S\times\A}$ given by:
\begin{equation*}
(\rmT_{\pi}\rmQ)(s,a):=\rmr(s,a)+\gamma \Erw_{S'}[\Erw_{A'\sim\pi}[\rmQ(S',A')]],
\end{equation*}
called the \textit{\textbf{Bellman operator}} of $\pi$ in $\DMDP$. Furthermore, the operator $\rmT^{*}:\real^{\S\times\A}\rightarrow \real^{\S\times\A}$,
\begin{equation*}
(\rmT_{*}\rmQ)(s,a):=\rmr(s,a)+\gamma \Erw_{S'}[\max_{a'\in\A}\rmQ(S',a')],
\end{equation*} 
called the \textit{\textbf{optimal Bellman operator}} of $\DMDP$, is also useful for MDP analysis. It directly follows from the discussion in the previous paragraph that the Q-function of the policy $\pi$ is the unique fixed point of the Bellman operator $\rmT_{\pi}$ of the policy $\pi$ in $\DMDP$. The same relation holds also between the optimal Bellman operator and the optimal Q-function for $\DMDP$. Formally, $\rmQ_{\pi}$ and $\rmQ_{*}$ are the unique solution of:
\begin{equation}
\label{Eq:akakshsgsggsffsgsgsggsgsssss}
\rmT_{\pi} Q_{\pi}=Q_{\pi}\quad\text{and}\quad \rmT_{*} Q_{*}=Q_{*}
\end{equation}
One property of the Bellman operator useful for our later approach, is that both $\rmT_{\pi}$ and $\rmT_{*}$ are \textit{\textbf{$\gamma$-contractions}} (w.r.t. $\norm{\cdot}_{\infty}$), i.e.:
\begin{equation}
\label{Eq:akaksjsssgggsgsssssss}
\norm{\rmT \rmQ-\rmT \rmQ'}_{\infty}\leq \gamma\norm{\rmQ-\rmQ'}_{\infty},\quad\forall \rmQ,\rmQ'\in\real^{\S\times\A},
\end{equation} 
where $\rmT$ is either $\rmT_{\pi}$ or $\rmT_{*}$.

\subsection{Single Agent Q-Learning}

In many practical applications, the agent in an MDP $(\S,\A,\rmr,\rmP)$ has no knowledge about the reward $\rmr$ and the transition probability $\rmP$. Consequently, she cannot simply solve the Bellman equations \eqref{Eq:ajajajsgsgsffssddafaaddaadsssdss} and \eqref{Eq:akjajsjshgdgdgdggdgdgdgdd}, or \eqref{Eq:akakshsgsggsffsgsgsggsgsssss}. 
One way to do this is by the so-called \textit{\textbf{Q-learning}} algorithm. Starting from an initial system state $S_{0}$, this algorithm maintains at each time step $t\in\nat_{0}$, the so-called Q-table $\rmQ_{t}$ which is a $\real^{\S\times \A}$-valued random variable, serving as an approximation of the optimal Q-function. For each time step $t\in\nat_{0}$, the update is as follows: First, the learner takes the action $A_{t}\sim\eta_{t}(\cdot|S_{t})$, where $\eta_{t}$ is a $\Delta_{\S}(\A)$-valued random variable,  experiences the reward $R_{t}:=\rmr(S_{t},A_{t})$, and observes the new state of the system $S_{t+1}\sim\rmP(\cdot|S_{t},A_{t})$. By those information, the agent finally update the Q-table at time $t\in\nat$ as follows: 
\begin{equation*}
Q_{t+1}(S_{t},A_{t})=(1-\gamma_{t}) Q_{t}(S_{t},A_{t})+\gamma_{t}[R_{t}+\beta\max_{a\in\A }Q_{t}(S_{t+1},a)],
\end{equation*}
and leaves the remaining entries unupdated, i.e., $Q_{t+1}(s,a)=Q_{t}(s,a)$, for all $(s,a)\neq (S_{t},A_{t})$.     

The sequence of $\Delta_{\S}(\A)$-valued random variable $(\eta_{t})_{t\in\nat_{0}}$ specifies the interaction of the learner with the system in the training phase, as it provides the choice of instantaneous action given the system state. By this reason, $(\eta_{t})_{t\in\nat_{0}}$ is called the \textbf{learning policy}. Generally, the learning policy depends on the historical learning data. One way to realize the learning policy is by generating it from the Q-table update $Q_{t}$ by means of a mapping $\Psi:\real^{\S\times\A}\rightarrow\Delta_{\S}(\A)$, i.e., $\eta_{t}=\Psi Q_{t}$. We refer $\Psi$ as to the \textit{\textbf{policy generator}}. In practice, one chooses the policy generator such that the corresponding learning policy follows the famous trade-off principle of exploration and exploitation. The first principle means that the probability of choosing any action should be non-zero, and the second principle means that the learning policy should be concentrated on the set of the best actions respective to the historical data. The latter is contained implicitly in the Q-table update as it is generated by historical rewards.  
In the following we provide some popular choices of learning policies:
\begin{example}[Boltzmann policy]
\label{Ex:Boltz}
An instance of a learning policy satisfying this principle is the \textit{\textbf{Boltzmann policy/strategy}} $\eta_{t}^{\Boltz,\tau}$ with the temperature $\tau$, which is generated by the policy generator $\Psi_{\Boltz,\tau}$ given by:
\begin{equation}
\label{Eq:ajajhhsgsgsggssssss}
(\Psi_{\Boltz,\tau} Q)(\cdot|s)=\Phi_{\Softmax,\tau} Q(\cdot,s),\quad\text{where }\Phi_{\Softmax,\tau}:\real^{\A}\rightarrow\Delta(\A),~(\Phi_{\Softmax,\tau}f)(a)\propto\exp\left(\frac{f(a)}{\tau}\right).
\end{equation}
The mapping $\Phi_{\Softmax,\tau}$ is called the \textbf{softmax} and plays an important role in the theory of decision making, i.a., machine learning. The Boltzmann policy $\eta_{t}^{\Boltz,\tau}$ takes specifically the following specific form: 
\begin{equation*}
 \eta_{t}^{\Boltz,\tau}(a|s)\propto\exp\left( \frac{Q_{t}(s,a)}{\tau}\right)
 \end{equation*} 
For $\tau\rightarrow 0$, the Boltzmann strategy tends to be the greedy w.r.t. $Q_{t}$  (exploitation), and for $\tau\rightarrow \infty$, it tends to choose the action with equal probability (exploration). 
\end{example}
\begin{example}
Another popular learning policy is the so-called \textbf{$\epsilon$-greedy policy} $\eta_{t}^{\Greed,\epsilon}$ ($\epsilon\in [0,1]$) generated by the policy generator $\Psi_{\Greed,\epsilon}$ given by:
 \begin{equation*}
 (\Psi_{\Greed,\epsilon} Q)(\cdot|s)=(1-\epsilon)+\epsilon\frac{\delta_{\argmax_{a'\in\A}Q_{t}(s,a')}(\cdot)}{\abs{\argmax_{a'\in\A}Q_{t}(s,a')}},
 \end{equation*}
\end{example}
In this work, we mostly consider the Boltzmann policy as a learning policy. One reason for this is the analytical convenience of the Boltzmann policy, which is founded by the fact that this learning policy possesses nice properties such as Lipschitz continuity. Furthermore, Boltzmann policy is based on the softmax selection rule which is a plausible model for a natural decision-making. For instance, it is shown in \cite{LeeConroy2004,KimHwang2009} that the behaviour of monkeys during reinforcement learning experiments is consistent with the Boltzman rule for learning policy. Furthermore, there are vast connections between the softmax rule and the neurophysiology of the decision-making (see e.g., \cite{daw2006cortical,Lee2006,Cohen2007,BOSSAERTS2015}) 

\subsection{Information Asymmetrical LA-GA Game Setting}
\label{Subsec:aajajsshshsggsgshddddssss}
\subsubsection{LA-GA Markov Game}
In this work, we consider the setting of Markov game (see, e.g., \cite{zhang2019multiagent}) with two players: The \textit{\textbf{localized agent (LA)}} and the  \textit{\textbf{globalized agent (GA)}}. The (finite) state space of system containing those agents is denoted by $\S$. $\A_{\Loc}$ stands for the (finite) action space of the LA, and $\A_{\Glob}$ for the action space of the GA. The function $\rmr^{\Loc}:\S\times\A_{\Loc}\times\A_{\Glob}\rightarrow\real$ specifying the LA's reward depends on the state of the system, the action of the GA, and her own action.
Similarly, GA's reward function is given by $\rmr^{\Glob}:\S\times\A_{\Loc}\times\A_{\Glob}\rightarrow\real$. Throughout, we assume that both agents are unaware of the reward functions. Assuming that the system is in the state $s\in\S$, and that the agents apply the joint action $(a_{\Loc},a_{\Glob})\in\A_{\Loc}\times\A_{\Glob}$, the system state changes in Markovian manner as in a MDP described by a probability distribution $\rmP(\cdot|s,a_{\Loc},a_{\Glob})\in\Delta(\S)$, unknown to both agents. 

In our study, we assume that LA can choose an action-deciding strategy for deciding her action from the set $\Delta_{\S}(\A_{\Loc})$, where for a $\pi_{\Loc}\in\Delta_{\S}(\A_{\Loc})$, $\pi_{\Loc}(a_{\Loc}|s)$ stands for the probability that LA chooses $a_{\Loc}\in\A_{\Loc}$ given that the system is in the state $s\in\S$. Moreover, we assume that GA can choose an action deciding strategy from the set $\Delta_{\S\times\A_{\Loc}}(\A_{\Glob})$. For a $\pi_{\Glob}\in\Delta_{\S\times\A_{\Loc}}(\A_{\Glob})$, $\pi_{\Glob}(a_{\Glob}|s,a_{\Loc})$ stands for the probability that GA chooses $a_{\Glob}\in\A_{\Glob}$ given that the system is in the state $s\in\S$ and that LA chooses the action $a_{\Loc}\in\A_{\Loc}$. This class of LA-GA strategy reflects the local-global information asymmetrical aspect of our interest, since it models the fact that GA can observe LA's action while the latter cannot observe the former's action. 

Align with the MDP setting, our focus in this paper is on the aforementioned modeled LA-GA Markov game in the repeated infinite horizon setting. Accordingly, we extend the notion of value function for a single agent MDP to the LA-GA setting by defining the LA's value function of the LA-GA strategy tuple $(\pi_{\Loc},\pi_{\Glob})$ as:
\begin{equation*}
V_{\pi_{\Loc},\pi_{\Glob}}^{\Loc}(s)=\Erw\left[\left. \sum_{t=0}^{\infty}\beta_{\Loc}^{t}\rmr^{\Loc}(S_{t},A^{\Loc}_{t},A^{\Glob}_{t})\right| \begin{aligned}&S_{0}=s,S_{t+1}\sim P(\cdot|S_{t},A_{t}),\\
& A^{\Loc}_{t}\sim\pi_{\Loc}(\cdot|S_{t}),~A^{\Glob}_{t}\sim\pi_{\Glob}(\cdot|S_{t},A_{\Loc,t})  \end{aligned}\right],
\end{equation*} 
where $\beta_{\Loc}\in (0,1)$ denotes the discount factor of LA.
The GA's value function $V_{\pi_{\Loc},\pi_{\Glob}}^{\Glob}$ is defined similar as above with the discount factor and the reward replaced by $\beta_{\Glob}\in(0,1)$ and $\rmr^{\Glob}$.
\subsubsection{Game Learning Model}
Since the underlying game is unknown to both agents, they have to learn in order to deduce the (agent-subjective) optimal strategy. Therefore, we also study the LA-GA training phase. In this phase, we consider a slight modification of the information asymmetrical Markov game described before with the following specification. 
Starting with the initial state $S_{0}=s\in\S$, the agents execute the following procedure for each time $t\in\nat_{0}$: 
\begin{itemize}
	\item First, LA chooses the (randomized) action $A_{t}^{\Loc}$ possibly utilizing the historical and actual system dynamic $(S_{\tau})_{\tau\in [t]}$, and some implicit information about the historical GA actions $(A^{\Glob}_{\tau})_{\tau\in [t-1]}$ (and also her actions $(A^{\Loc}_{\tau})_{\tau\in [t-1]}$)
	\item Afterwards GA observes LA's action $A_{t}^{\Loc}$. By utilizing this information, the historical and actual system dynamic $(S_{\tau})_{\tau\in [t]}$, and implicit information about her and LA's past actions $(A_{\tau}^{\Loc},A_{\tau}^{\Glob})_{\tau\in [t-1]}$, GA chooses the action $A_{t}^{\Glob}$.
	\item Finally LA (resp. GA) obtain the reward  $\rmr^{\Loc}(S_{t},A_{t}^{\Loc},A_{t}^{\Glob})$ 
	(resp. $\rmr^{\Glob}(S_{t},A_{t}^{\Loc},A_{t}^{\Glob})$) and subsequently the system transits to the state $S_{t+1}\sim \rmP(\cdot|S_{t},A_{t}^{\Loc},A_{t}^{\Glob})$. 
\end{itemize}
Our main interest is on the Q-learning-based LA-GA training behaviour which we will specify in the next section. 


For analysis purposes, we can represent the scope of LA information by means of the filtration $(\F_{t})_{t\in\nat_{0}}$, where $\F_{t}$ is the sigma-algebra generated by $(S_{\tau})_{\tau\in [t]_{0}}$, $(A^{\Loc}_{\tau})_{\tau\in [t]_{0}}$, $(A^{\Glob}_{\tau})_{\tau\in [t-1]_{0}}$:
\begin{equation*}
\F_{t}:=\sigma((S_{\tau})_{\tau\in [t]_{0}}, (A^{\Loc}_{\tau})_{\tau\in [t]_{0}}, (A^{\Glob}_{\tau})_{\tau\in [t-1]_{0}}).
\end{equation*}
$\F_{t}$ represents the (implicit and explicit) information LA have at time $t\in\nat$ after choosing an action $A_{t}^{\Loc}$.  As modeled in the previous paragraph, the scope of GA information in the training phase is different than LA. We represent this by a different filtration $(\tildF_{t})_{t\in\nat_{0}}$, where $\tildF_{t}$ denotes the $\sigma$-algebra generated by $\F_{t}$ and $A^{\Glob}_{t}$, i.e.:
\begin{equation*}
    \tilde{\mathcal{F}}_{t}:=\sigma(\F_{t},A_{t}^{\Glob})=\sigma((S_{\tau})_{\tau\in [t]_{0}}, (A^{\Loc}_{\tau})_{\tau\in [t]_{0}}, (A^{\Glob}_{\tau})_{\tau\in [t]_{0}}).
\end{equation*}
The $\sigma$-algebra $\tildF_{t}$ represents the (implicit and explicit) information LA have at time $t$ after choosing an action $A_{t}^{\Glob}$.  In this model, the information that GA can utilize at time $t\in\nat_{0}$ for choosing the action $A_{t}^{\Glob}$ can be symbolized by $\F_{t}$. 

\section{LA-GA Q-learning -- Convergence Result}
\label{Sec:ConvThm}
\begin{algorithm}[htbp]
	\caption{LA Q-learning with Globalized Interference (LAQGI)}
	\begin{algorithmic}
		\STATE \textbf{Extrinsic Parameter:} LA reward $\rmr^{\Loc}$, system dynamic $\rmP$, GA policy $(\eta_{t}^{\Glob})_{t\in\nat_{0}}$, with 
		\STATE \textbf{LA parameter:} Policy generator $\Psi^{\Loc}:\real^{S\times\A_{\Loc}}\rightarrow \Delta_{\S}(\A_{\Loc})$, learning rate $\gamma_{t}^{\Loc}>0$, discount factor $\beta^{\Loc}$ 
		\STATE \textbf{Initialization:} $Q_{0}^{\LQ}\in\real^{\S\times\A}$
		%
		\FOR{$t=0,1,2,\ldots$}
		
		\STATE Execute the action $A_{t}^{\Loc}\sim\eta_{t}^{\LQ}(\cdot|S_{t})$, according to the LA learning policy $\eta_{t}^{\LQ}:=\Psi^{\Loc}Q^{\LQ}_{t}$
		\STATE Experience $R^{\Loc}_{t}:=\rmr^{\Loc}(S_{t},A^{\Glob}_{t},A^{\Loc}_{t})$, where $A^{\Glob}_{t}\sim \eta^{\Glob}_{t}(\cdot|\F_{t})$
		\STATE Query system state information $S_{t+1}\sim \rmP(\cdot|S_{t},A^{\Loc}_{t},A^{\Glob}_{t})$
		\STATE Update: 
		$$Q^{\LQ}_{t+1}(S_{t},A^{\Loc}_{t})=(1-\gamma^{\Loc}_{t}) Q^{\LQ}_{t}(S_{t},A^{\Loc}_{t})+\gamma^{\Loc}_{t}H^{\Loc}_{t+1},$$
		where:
		$$H^{\Loc}_{t+1}=R^{\Loc}_{t}+\beta^{\Loc}\max_{a_{\Loc}^{'}\in \A_{\Loc}}Q^{\LQ}_{t}(S_{t+1},a_{\Loc}^{'}).$$
		\FOR{all $(s,a_{\Loc})\neq (S_{t},A^{\Loc}_{t})$}
		\STATE Update $Q^{\LQ}_{t+1}(s,a_{\Loc})=Q^{\LQ}_{t}(s,a_{\Loc})$.
		\ENDFOR
		
		\ENDFOR
	\end{algorithmic}
	\label{Alg:ohgfdfddrrdrdrdrdrdghhgggdsa}
\end{algorithm}
\setlength{\textfloatsep}{0pt}
In this section, we extend the single agent Q-learning paradigm  to the informational asymmetrical Markov game setting given in Subsection \ref{Subsec:aajajsshshsggsgshddddssss}. Our particular interest is on the convergence behaviour of the given Q-learning extension. 
\subsection{LA Q-learning}

First, we model the Q-learning iterate for LA by straightforwardly extend the single agent Q-learning to the LA reward structure having additional dependency on the GA action. The specific description is given in Algorithm \ref{Alg:ohgfdfddrrdrdrdrdrdghhgggdsa}. In our LA Q-learning model, we assume in this subsection that the GA chooses her action according to a general time-varying policy $(\eta_{t}^{\Glob})_{t\in\nat_{0}}$, where $\eta_{t}^{\Glob}(s|\F_{t})$ is a $\Delta(\A_{\Glob})$-valued random variable, which might depend on the (implicit information of the) state-action history according to the GA scope of information (see Subsection \ref{Subsec:aajajsshshsggsgshddddssss}). We formalize the latter, by assuming that for every $t\in\nat_{0}$, $\eta_{t}^{\Glob}$ is $\F_{t}$-measureable. To emphasize the latter dependency, we sometimes use the notation $\eta_{t}^{\Glob}(\cdot|\F_{t})$ instead of $\eta_{t}$. The analytical use of $\eta_{t}^{\Glob}$ is to be understood as:
\begin{equation}
\label{Eq:ajkajjhssghsgsgsgsssss}
\Erw[f(A_{t}^{\Glob})|\F_{t}]=\sum_{a_{\Glob}\in \A_{\Glob}}\eta_{t}(a_{\Glob}|\F_{t})f(a_{\Glob}),\quad A_{t}^{\Glob}\sim\eta_{t}(\cdot|\F_{t}),~f:\A_{\Glob}\rightarrow\real.
\end{equation}	

In the following, we specify some conditions leading to the convergence of LAQGI and determine the corresponding limit:
\begin{theorem}
	\label{Thm:akakashshssggdhdhdggdgdhdhd}
Suppose that there exists $\eta^{\Glob}_{\infty}(\cdot)\in\Delta_{\S\times\A_{ \Loc}}(\A_{\Glob})$ s.t.:
\begin{equation}
\label{Eq:jjaajjsgsgsfsffsfsgsgsss}
\lim_{t\rightarrow\infty}\norm{\eta^{\Glob}_{t}(\cdot|\F_{t})-\eta^{\Glob}_{\infty}(\cdot|S_{t},A^{\Loc}_{t})}=0,\quad \text{a.s.}
\end{equation}
and that:
\begin{equation}
\label{Eq:ajajshshsggshhsgsgsgssssss}
\sum_{t=0}^{\infty}\psi_{t}^{\Loc}=\infty~ \text{and}~ \sum_{t=0}^{\infty}\psi_{t}^{\Loc,2}<\infty~\text{a.s.,}\quad \text{where}\quad \psi^{\Loc}_{t}(s,a)=\mathbf{1}_{\lrbrace{S_{t}=s,A^{\Loc}_{t}=a}}\gamma^{\Loc}_{t}.
\end{equation}

Then the a.s. limit $\tildrmQ^{\Loc}_{\LAQGI}$ of LAQGI's iterate $(Q^{\LQ})_{t\in\nat_{0}}$ is the optimal $Q$-function of the discounted MDP $(\S,\A_{\Loc},\tildrmr^{\Loc},\tildrmP^{\Loc},\beta_{\Loc})$, where for $(s,a_{\Loc})\in\S\times \A_{\Loc}$:
\begin{equation}
\label{Eq:ajjahsgsgsgsffffsffsssssss}
\tildrmr^{\Loc}(s,a_{\Loc}):=\Erw_{A_{\Glob}\sim\eta_{\infty}^{\Glob}(\cdot|s,a_{\Loc})}[\rmr^{\Loc}(s,a_{\Loc},A_{\Glob})]\quad\text{and}\quad\tildrmP^{\Loc}(\cdot|s,a_{\Loc}):=\Erw_{A_{\Glob}\sim\eta^{\Glob}_{\infty}(\cdot|s,a_{\Loc})}[\rmP(\cdot|s,a_{\Loc},A_{\Glob})].
\end{equation}
\end{theorem}
The proof of above Theorem is given in Subsubsection \ref{Subsubsec:LAQGI}.
\begin{remark}
\label{Rem:ajajssgggffsssss}
Above theorem gives hint that LA learns via Q-learning how to act optimally in expectation given GA's stationary strategy (see Lemma 6), remarkably without knowing the latter. This property is interesting for e.g., security applications, where LA is a defender and GA is an attacker, since it implies that Q-learning helps the defender to learn optimal defend policy. However, it is not yet clear, whether, by applying the greedy policy resulted from learning phase, LA has indeed an optimal discounted cumulative reward. the discounted yields of the LA. We will clarify this aspect in the next section. 
\end{remark}



\subsection{GA Q-Learning}
Our actual interest is on the behaviour of a Q-learning applying GA. As GA has informational advantage over LA by knowing the latter's instantaneous action, we assume that she utilizes this information in the learning phase and executes Q-table update for each observed LA action. Our proposal of GA Q-learning is specifically given in Algorithm 	\ref{Alg:ohgfdfddrrdrdrdrdrdghhgggdsadxxx}.

\begin{algorithm}[htbp]
	\caption{GA Q-Learning (GAQL)}
	\begin{algorithmic}
		\STATE \textbf{Extrinsic Parameter:} GA reward $\rmr^{\Glob}$, system dynamic $\rmP$, LA policy $(\eta_{t}^{\Loc})_{t\in\nat_{0}}$
		\STATE \textbf{GA parameter:} Policy generator $\Psi^{\Glob}:\real^{\S\times\A_{\Glob}\times\A_{\Loc}}\rightarrow\Delta_{\S\times\A_{\Loc}}(\A_{\Glob})$, learning rate $\gamma_{t}^{\Loc}$, discount factor $\beta^{\Loc}$ 
		\STATE \textbf{Initialization:} $Q^{\GQ}_{0}\in\real^{\S\times\A_{\Glob}\times\A_{\Loc}}$, $S_{0}\in\S$.
		
		\FOR{$t=0,2,\ldots$}
		\STATE Query LA's action $A_{t}^{\Loc}\sim\eta^{\Loc}_{t}(\cdot|S_{t})$
		\STATE Execute the action $A_{t}^{\Glob}\sim\eta_{t}^{\GQ}(\cdot|S_{t},A_{t}^{\Loc})$, according to the GA learning policy $\eta_{t}^{\GQ}:=\Psi^{\Glob}Q^{\GQ}_{t}$
		
		\STATE  Experience the reward 	$R^{\Glob}_{t}:=\rmr^{\Glob}(S_{t},A^{\Glob}_{t},A^{\Loc}_{t})$
		\STATE Query system state information $S_{t+1}\sim \rmP(\cdot|S_{t},A^{\Loc}_{t},A^{\Glob}_{t})$
		\STATE Update:
		\begin{equation*}
		 Q^{\GQ}_{A^{\Loc}_{t},t+1}(S_{t},A^{\Glob}_{t})=(1-\gamma^{\Glob}_{t}) Q^{\GQ}_{A^{\Loc}_{t},t}(S_{t},A^{\Glob}_{t})+\gamma^{\Glob}_{t}H^{\Glob}_{t+1}  
		\end{equation*}
		where 	$H^{\Glob}_{t+1}=R^{\Glob}_{t}+\beta_{\Glob}\max_{a^{'}_{\Glob}\in \A_{\Glob}}Q^{\GQ}_{A^{\Loc}_{t},t}(S_{t+1},a_{\Glob})$.
		\FOR{all $(s,a_{\Glob},a_{\Loc})\neq (S_{t},A^{\Glob}_{t},A^{\Loc}_{t})$}
		\STATE Update $Q^{\GQ}_{a_{\Loc},t+1}(s,a_{\Glob})=Q^{\GQ}_{a_{\Loc},t}(s,a_{\Glob})$
		\ENDFOR
		\ENDFOR
	\end{algorithmic}
	\label{Alg:ohgfdfddrrdrdrdrdrdghhgggdsadxxx}
\end{algorithm}

The following Theorem gives sufficient conditions for the convergence of GAQL: 
\begin{theorem}
	\label{Thm:kiaiaahshsgsgsgsfffsfsfssss}
Suppose that:
\begin{equation}
\label{Eq:jajaajasggsgsfdffdfddddd}
\sum_{t=0}^{\infty}\psi_{t}^{\Glob}=\infty\quad\text{and}\quad \sum_{t=0}^{\infty}\psi_{t}^{\Glob,2}<\infty,   
\end{equation}
where:
\begin{equation*}
\psi_{t}^{\Glob}(s,a_{\Loc},a_{\Glob}):=\mathbf{1}_{\lrbrace{S_{t}=s,A_{t}^{\Loc}=a_{\Loc},A_{t}^{\Glob}=a_{\Glob}}}\gamma_{t}.
\end{equation*}  
Then $\forall a_{\Loc}\in \A_{\Loc}$, the a.s. limit 
$\tildrmQ^{\Glob}_{\GAQL,a_{\Loc}}$ 
of GAQL's iterates $(Q^{\GQ}_{a_{\Loc},t})_{t\in\nat 0}$ is the optimal Q-function of the discounted MDP $(\S,\A_{\Glob},\tildrmr^{\Glob}_{a_{\Loc}},\tildrmP^{\Glob}_{a_{\Loc}},\beta_{\Loc})$, where:
\begin{equation*}
\tildrmr_{a_{\Loc}}^{\Glob}(s,a_{\Glob}):=\rmr^{\Glob}(s,a_{\Glob},a_{\Loc})\quad\text{and}\quad \tildrmP_{a_{\Loc}}^{\Glob}(s'|s,a_{\Glob}):=\rmP(s'|s,a_{\Glob},a_{\Loc}).
\end{equation*}
\end{theorem}
The proof of above Theorem is given in Subsubsection \ref{Subsubsec:GAQL}.

\begin{remark}
\label{Rem:ajajajashhhsgsgsgss}
Above theorem gives the hint that GA learns via GAQL the optimal strategies given that LA executes a constant action, it learns the optimal Q-function of the relevant MDP. At the first sight, this might affect adversely GA performance since LA's action rather changes over the time. However, we will see later in the next section (Lemma \ref{Lem:ahhasgsgsfsfsfsffssss}) that this is not true: GA learns via GAQL greedy policy given that LA applies a strategy from the class of deterministic strategies including LA optimal policy according to MDP theory.
\end{remark}


\subsection{Joint LA-GA Q-Learning}
In this subsection, we jointly consider the Q-learning LA applying LAQGI (Algorithm \ref{Alg:ohgfdfddrrdrdrdrdrdghhgggdsa}) and the Q-learning GA applying the GAQL (Algorithm \ref{Alg:ohgfdfddrrdrdrdrdrdghhgggdsadxxx}). We specifically link both algorithms, by setting the extrinsic GA policy in Algorithm \ref{Alg:ohgfdfddrrdrdrdrdrdghhgggdsa} (LAQGI) as the GA learning policy in Algorithm \ref{Alg:ohgfdfddrrdrdrdrdrdghhgggdsadxxx} given the actual state and LA action, and by setting the extrinsic LA policy in Algorithm \ref{Alg:ohgfdfddrrdrdrdrdrdghhgggdsadxxx} as the LA learning policy $\eta^{\LQ}_{t}$ in Algorithm \ref{Alg:ohgfdfddrrdrdrdrdrdghhgggdsa}, i.e.: 
\begin{assum}
\label{Ass:Joint}
Consider Algorithms \ref{Alg:ohgfdfddrrdrdrdrdrdghhgggdsa} and \ref{Alg:ohgfdfddrrdrdrdrdrdghhgggdsadxxx}. For each time $t\in\nat$, we set:
\begin{equation*}
\eta^{\Glob}_{t}(\cdot|\F_{t})=\eta^{\GQ}_{t}(\cdot|S_{t},A_{t}^{\Loc})\quad\text{and}\quad\eta_{t}^{\Loc}=\eta^{\LQ}_{t}
\end{equation*}
\end{assum}
We refer the above specified joint LA-GA Q-learning as \textbf{LAQGI-GAQL}. In the following, we provide the convergence guarantee of this joint Q-learning paradigm:
\begin{theorem}
\label{Thm:JointQLearn}
Suppose that \eqref{Eq:ajajshshsggshhsgsgsgssssss}, \eqref{Eq:jajaajasggsgsfdffdfddddd}, and Assumption \ref{Ass:Joint} hold. Moreover, suppose that $\Psi^{\Glob}$ is continuous. Then:
\begin{enumerate}
\item the Q-iterates $(Q^{\GQ}_{t})_{t\in\nat_{0}}$ of GAQL converges to the optimal Q-function $\tildrmQ^{\Glob}_{\LAQGI}$ of the discounted MDP $(\S,\A_{\Glob},\tildrmr^{\Glob}_{a_{\Loc}},\tildrmP^{\Glob}_{a_{\Loc}},\beta_{\Loc})$, where:
\begin{equation*}
\tildrmr_{a_{\Loc}}^{\Glob}(s,a_{\Glob}):=\rmr^{\Glob}(s,a_{\Glob},a_{\Loc})\quad\text{and}\quad \tildrmP_{a_{\Loc}}^{\Glob}(s'|s,a_{\Glob}):=\rmP(s'|s,a_{\Glob},a_{\Loc}).
\end{equation*}
\item for any $a_{\Loc}\in\A_{\Loc}$, the Q-iterates $(Q^{\LQ}_{a_{\Loc},t})_{t\in\nat_{0}}$ of LAQGI converges a.s. to the optimal $Q$-function $\tildrmQ^{\Loc}_{\GAQL,a_{\Loc}}$ of the discounted MDP $(\S,\A_{\Loc},\tildrmr_{\Loc},\tildrmP_{\Loc},\beta_{\Loc})$, where for $(s,a_{\Loc})\in\S\times \A_{\Loc}$:
\begin{equation}
\label{Eq:ajajahshssggssgshhsgsgshss}
\tildrmr^{\Loc}(s,a_{\Loc}):=\Erw_{A_{\Glob}\sim(\Psi^{\Glob}\tildrmQ^{\Glob}_{\GAQL})(\cdot|s,a_{\Loc})}[\rmr^{\Loc}(s,a_{\Loc},A_{\Glob})]\quad\text{and}\quad\tildrmP^{\Loc}(\cdot|s,a_{\Loc}):=\Erw_{A_{\Glob}\sim(\Psi^{\Glob}\tildrmQ^{\Glob}_{\GAQL})(\cdot|s,a_{\Loc})}[\rmP(\cdot|s,a,A_{\Glob})].
\end{equation}
\end{enumerate}
\end{theorem}
\begin{proof}
The first statement follows from Theorem \ref{Thm:kiaiaahshsgsgsgsfffsfsfssss}. To show the second statement, notice that from the first statement, we have that a.s. $Q^{\GQ}_{t}\rightarrow \tildrmQ^{\GAQL}$ as $t\rightarrow\infty$. Consequently by the continuity of $\Psi^{\Loc}$, we have a.s. $\Psi^{\Glob}Q_{t}^{\GQ}\rightarrow \Psi^{\Glob}\tilde{\rmQ}^{\GAQL}$ and thus:
\begin{equation*}
\norm{\eta^{\Glob}_{t}(\cdot|\mathcal{F}_{t})-\eta^{\Glob}_{\infty}(\cdot|S_{t},A_{t}^{\Loc})}=\norm{(\Psi^{\Glob}Q_{t}^{\GQ})(\cdot|S_{t},A_{t}^{\Loc})-(\Psi^{\Glob}\tilde{\rmQ}^{\GAQL})(\cdot|S_{t},A_{t}^{\Loc})}\xrightarrow{t\rightarrow\infty}0,\quad\text{a.s.}
\end{equation*}
This shows that \eqref{Eq:jjaajjsgsgsfsffsfsgsgsss} is fulfilled, and thus the second statement follows from Theorem \ref{Thm:kiaiaahshsgsgsgsfffsfsfssss}.
\end{proof}
\section{Optimality Analysis of LAQGI-GAQL}
\label{Sec:OptProof}
Our main aim in this section is to formally investigate the performance of both agents respective to the policies yielded from the joint training by the Q-learning algorithms (LAQGI and GAQL) introduced in the previous section.
Specifically, assuming that the coupling of both Q-learning algorithms is given in Assumption \ref{Ass:Joint}, we aim to analyze the discounted rewards $\rmV^{(i)}_{\pi^{\LAQGI}_{\Loc},\pi^{\GAQL}_{\Glob}}$, $i\in\lrbrace{\Loc,\Glob}$, of LA and GA, if LA applies the greedy policy $\pi_{\Loc}^{\LAQGI}$ (w.r.t. the limit $\tildrmQ^{\Loc}_{\LAQGI}$ of LAQGI's iterate) and GA applies the greedy strategy $\pi_{\Glob}^{\GAQL}$ (w.r.t. the limit $\tildrmQ^{\Glob}_{\GAQL}$ of GAQL's iterate) of GA. Specifically those strategies is given by:
\begin{equation*}
\pi_{\Loc}^{\LAQGI}(s)\in\argmax_{a_{\Loc}}\tildrmQ^{\Loc}_{\LAQGI}(s,a_{\Loc})\quad\text{and}\quad \pi_{\Glob}^{\GAQL}(s,a_{\Loc})\in\argmax_{a_{\Glob}}\tildrmQ^{\Glob}_{\GAQL}(s,a_{\Loc}).      
\end{equation*}
In doing this, we investigate the performance of the greedy strategy resulted from the corresponding Q-learning algorithm for each agent separately.

Let us first consider LA. Our result is that the greedy policy of the LA Q-learning is indeed optimal for LA given that the GA applies the asymptotic training policy (see Remark \ref{Rem:ajajssgggffsssss}). The formal statement is as follows:
\begin{lemma} 
\label{Lem:ajajsshssgsgssfsfss}
Suppose that the assumptions of Theorem 	\ref{Thm:akakashshssggdhdhdggdgdhdhd} holds, and let $\pi^{\LAQGI}_{\Loc}$ be the greedy policy (w.r.t. $\tildrmQ_{\LAQGI}^{\Loc}$) of LA resulted from LAQGI (Algorithm \ref{Alg:ohgfdfddrrdrdrdrdrdghhgggdsa}) for a given sequence $(\eta_{t}^{\Glob})_{t\in\nat_{0}}$ of extrinsic GA's policies having the limit policy $\eta^{\Glob}_{\infty}$. Then, we have:
\begin{equation*}
\rmV^{\Loc}_{\pi^{\LAQGI}_{\Loc},\eta^{\Glob}_{\infty}}\geq \rmV^{\Loc}_{\pi_{\Loc},\eta^{\Glob}_{\infty}},\quad\forall\pi_{\Loc}\in\Delta_{\S}(\A_{\Loc}).
\end{equation*}

\end{lemma}
The proof of this statement is given in Subsection \ref{Subsec:ajajajhshsggssfsfsfsssss}. 

In contrast to the LA, we have that the greedy policy of the GA Q-learning is optimal for GA given that LA applies deterministic policy. Formally, we have:
\begin{lemma}
	\label{Lem:ahhasgsgsfsfsfsffssss}
	Suppose that the condition in Theorem \ref{Thm:kiaiaahshsgsgsgsfffsfsfssss} holds. Let $\pi_{\Loc}\in\Delta_{S}(\A_{\Loc})$ be a deterministic LA policy, and $\pi^{\GAQL}_{\Glob}$ be the greedy strategy (w.r.t. $\tildrmQ^{\Glob}_{(\cdot),\GAQL}$) resulted from GA Q-learning (Algorithm 	\ref{Alg:ohgfdfddrrdrdrdrdrdghhgggdsadxxx}) for a given LA policy $(\eta_{t}^{\Loc})_{t\in\nat_{0}}$. Then it holds:
	\begin{equation*}
	\rmV^{\Glob}_{\pi_{\Loc},\pi^{\GAQL}_{\Glob}}\geq \rmV^{\Glob}_{\pi_{\Loc},\pi_{\Glob}},\quad\forall\pi_{\Glob}\in\Delta_{\S\times\A_{\Loc}}(\A_{\Glob}).
	\end{equation*}
\end{lemma}
The proof of this result is given in Subsection \ref{Subsec:ajajajhshsggssfsfsfsssss}. Above Lemma is the formal continuation of the discussion given in Remark \ref{Rem:ajajajashhhsgsgsgss}. 

\begin{remark}
Above Lemma guarantees the optimality of the GA greedy policy resulted from GAQL given that the LA applies a deterministic policy. At first sight, the latter condition seems to be restrictive. However, we would expect that the Q-learning LA would utilize the MDP theory, and apply this sort of policy.  
\end{remark}
\begin{remark}
Another interesting point of above Lemma is that the optimality of the GAQL greedy strategy is not directly dependent of the LA policy $(\eta^{\Loc}_{t})_{t\in\nat_{0}}$, in the sense that the guarantee does not require the existence of a coupling between $\pi_{\Loc}$ and $(\eta^{\Loc}_{t})_{t\in\nat_{0}}$. Merely, LA policy $(\eta^{\Loc}_{t})_{t\in\nat_{0}}$ is one necessary factor, which ensures sufficient exploration of the state-action space by GA sufficiently and thus the existence of $\pi_{\Glob}^{\GAQL}$.
\end{remark}

To sum up we have from above results that the GA learning anticipates LA's post-learning strategy, while LA learning results in the best response strategy respective to long-term GA learning strategy. As a consequence, we have that assuming the training of both agents are coupled by Assumption \ref{Ass:Joint}, the tuple $(\pi_{\Loc}^{\LAQGI},\pi_{\Glob}^{\GAQL})$ of post-Q-learning policies can in general not be the solution concept of the underlying game, since LA might be better off by applying another strategy, as GA applies in the post-learning phase the greedy strategy $\pi_{\Glob}^{\GAQL}$ which differs in general to the asymptotic GA learning strategy:
\begin{equation*}
\eta^{\GQ}_{\infty}=\Psi^{\Glob}\tildrmQ^{\Glob}_{\GAQL}=\Psi^{\Glob}(\lim_{t\rightarrow\infty}Q^{\GQ}_{t}).
\end{equation*}However, if the GA's long-term learning strategy is approximately equal to GA's post-learning greedy strategy, it is likely that the latter tuple is an (almost) solution concept. To ensure the former, GA can use the Boltzmann strategy (see Example \ref{Ex:Boltz})with low temperature as the learning policy:
\begin{theorem}
	\label{Thm:ajajhshssggsgssgsgffsfssssss}
Let be $\tau>0$. Suppose that GA  applies in Algorithm \ref{Alg:ohgfdfddrrdrdrdrdrdghhgggdsadxxx} the Boltzmann strategy with temperature $\tau$ as the learning policy, given by:
\begin{equation}
\label{Eq:jajajjashhshsgsgsgshshsss}
(\Psi^{\Glob}Q)(a_{\Glob}|s,a_{\Loc})=(\Psi_{\Softmax,\tau}Q_{a_{\Loc}}(s,(\cdot)))(a_{\Glob}),
\end{equation}
where $\Psi_{\Softmax,\tau}$ is defined in \eqref{Eq:ajajhhsgsgsggssssss}. Furthermore,
suppose that the condition in Theorem \ref{Thm:JointQLearn} is fulfilled. Then the tuple $(\pi^{\LAQGI}_{\Loc},\pi^{\GAQL}_{\Glob})\in\Delta_{\S}(\A_{\Loc})\times\Delta_{\S\times\A_{\Loc}}(\A_{\Glob})$ is an almost Nash-equilibrium of the local-global Markov game, in the sense that:
\begin{align}
&\rmV_{\pi^{\LAQGI}_{\Loc},\pi^{\GAQL}_{\Glob}}^{\Loc}\geq \rmV_{\pi_{\Loc},\pi^{\GAQL}_{\Glob}}^{\Loc}-\epsilon,\quad\forall ~\text{deterministic } \pi_{\Loc}\in\Delta_{\S}(\A_{\Loc}) \label{Eq:jajjssgsggsgsfffsssss}\\
&\rmV_{\pi^{\LAQGI}_{\Loc},\pi^{\GAQL}_{\Glob}}^{\Glob}\geq \rmV_{\pi_{\Loc}^{\LAQGI},\pi_{\Glob}}^{\Glob},\quad\forall \pi_{\Glob}\in\Delta_{\S\times\A_{\Loc}}(\A_{\Glob})\label{Eq:jajjssgsggsgsfffsssss2},
\end{align}
where:
\begin{equation*}
\epsilon\leq \frac{2\norm{\rmr^{\Loc}}_{\infty}D}{(1-\beta_{\Loc})^{2}}\exp\left(-\frac{C}{\tau} \right),
\end{equation*}
with a certain constants $C:=\max_{s,a_{\Loc}}C_{s,a_{\Loc}}$ and $D:=\max_{s,a_{\Loc}}D_{s,a_{\Loc}}$ given by:
	\begin{equation*}
	D_{s,a_{\Loc}}:=\sqrt{2\frac{\abs{\A_{\Glob}}-\abs{\argmax_{a_{\Glob}\in\A_{\Glob}}\tildrmQ^{\Glob}_{a_{\Loc}}(s,a_{\Glob})}}{\abs{\argmax_{a_{\Glob}\in\A_{\Glob}}\tildrmQ^{\Glob}_{a_{\Loc}}(s,a_{\Glob})}}},
	\end{equation*}
	and:
	\begin{equation*}
	C_{s,a_{\Loc}}:=\min_{a_{\Glob}\notin\argmax_{a_{\Glob}\in\A_{\Glob}}\tildrmQ^{\Glob}_{a_{\Loc}}(s,a_{\Glob})}\left(\max_{a_{\Glob}\in\A_{\Glob}}\tildrmQ^{\Glob}_{a_{\Loc}}(s,a'_{\Glob})-\tildrmQ^{\Glob}_{a_{\Loc}}(s,a_{\Glob}) \right). 
	\end{equation*}
\end{theorem}
The proof of this theorem can be found in Subsection \ref{Subsec:ajajajhshsggssfsfsfsssss}. So from above Theorem, we have that, up to a deviation $\epsilon$ decreasing exponentially with the temperature of the in-training Boltzmann policy, no agent applying the post-Q-learning greedy strategy has incentives to change her strategy. 
\begin{remark}
One thing which is unusual in above Theorem is that the statement is respective to deterministic LA strategies and not general strategies. This occurance is caused not quite by the same reason than that in Lemma \ref{Lem:ahhasgsgsfsfsfsffssss}, as the former is a sufficient condition in order to ensure the equality $\rmV^{\Loc}_{\pi_{\Loc},\pi^{\GAQL}_{\Glob}}=\rmV^{\Loc}_{\pi_{\Loc},\tilde{\pi}^{\GAQL}_{\Glob}}$, where $\tilde{\pi}^{\GAQL}_{\Glob}(\cdot|s,a_{\Loc})$ is uniformly distributed in $\argmax_{a_{\Glob}\in\A_{\Glob}}\tildrmQ_{a_{\Loc}}(s,a_{\Glob})$ used to derive above theorem. 
Nevertheless, since the optimal strategy in a (single-agent) MDP is deterministic, one can expect that LA applies this kind of strategy. However, in case that GA applies $\tilde{\pi}^{\GAQL}_{\Glob}$ instead of the corresponding greedy strategy, we can replace in above theorem, the condition that $\pi_{\Loc}\in\Delta_{\S}(\A_{\Loc})$ is deterministic.   \end{remark}
\section{Optimal of GA policy - EIQGL}
\label{Sec:ajajahshsgsgssssss}
Recall that from Lemma \ref{Lem:ahhasgsgsfsfsfsffssss}, we know that GAQL yields an optimal policy for GA given that LA applies deterministic stationary strategy. In this section, we aim to find a method for GA to find an optimal policy given that the LA possibly applies general stationary strategy. Such a method can be used, e.g., in security application, for equipping the defender which can observe attacker's action optimal strategy to reduce the latter's effect providing that the attacker applies a stationary strategy.

For the sake of finding above discussed GA policy, we first consider the maximum of GA's value function given that LA applies the stationary strategy $\pi_{\Loc}\in\Delta_{\S}(\A_{\Loc})$:
\begin{equation}
\label{Eq:sjsjsgsgsgfsfssssss}
\rmV^{*,\Glob}_{\pi_{\Loc}}=\max_{\pi_{\Glob}\in\Delta_{\S\times\A_{\Loc}}(\A_{\Glob})}\rmV_{\pi_{\Loc},\pi_{\Glob}}.
\end{equation}
Our approach is to link $\rmV^{*,\Glob}_{\pi_{\Loc}}$ with an appropriate Bellman equation. Analogous to the standard Q-learning method, we subsequently design an iterative method to find the latter's solution giving rise to the desired optimal GA strategy. for which we can derive an iterative method. 

First, we derive the corresponding Bellman equation. For this sake, we define the operator $\rmH^{\Glob}_{\pi_{\Loc},\pi_{\Glob}}:\S\mapsto\S$ by:
\begin{equation*}
(\rmH^{\Glob}_{\pi_{\Loc},\pi_{\Glob}}\rmV)(s):=\Erw_{\substack{A_{\Loc}\sim\pi_{\Loc}(\cdot|s)\\A_{\Glob}\sim\pi_{\Glob}(\cdot|s,A_{\Loc})}}\left[\rmr^{\Glob}(s,A_{\Loc},A_{\Glob})+\beta_{\Glob} \Erw_{S'\sim\rmP(\cdot|s,A_{\Loc},A_{\Glob})}\left[\rmV(S') \right] \right], 
\end{equation*}
and the operator $\rmH^{\Glob}_{\pi_{\Glob},\pi_{\Loc}}:\S\mapsto\S$ by:
\begin{equation*}
(\rmH^{\Glob}_{\pi_{\Loc},*}\rmV)(s):=\Erw_{A_{\Loc}\sim\pi_{\Loc}(\cdot|s)}\left[\max_{a_{\Glob}}\left( \rmr^{\Glob}(s,A_{\Loc},a_{\Glob})+\beta_{\Glob} \Erw_{S'\sim\rmP(\cdot|s,A_{\Loc},a_{\Glob})}\left[\rmV(S') \right]\right)  \right] 
\end{equation*}
The following characterization is useful for our approach:
\begin{lemma}
	\label{Lem:ajajgsgsghsgffsfsfdfssss}
It holds:
\begin{enumerate}
	\item For all $\pi_{\Loc}\in\Delta_{\S}(\A_{\Loc})$, $\rmH^{\Glob}_{\pi_{\Loc},*}$ is a $\beta_{\Glob}$-contraction.
	\item For all $\pi_{\Loc}\in\Delta_{\S}(\A_{\Loc})$ and $\pi_{\Glob}\in\Delta_{\S\times\A_{\Loc}}(\A_{\Glob})$,  $\rmH^{\Glob}_{\pi_{\Loc},\pi_{\Glob}}\leq\rmH^{\Glob}_{\pi_{\Loc},*}$.
\end{enumerate}
\end{lemma}
\begin{proof}
The first statement follows by the following computation:
\begin{equation*}
\begin{split}
\abs{(\rmH^{\Glob}_{\pi_{\Loc},*}\rmV_{1})(s)-(\rmH^{\Glob}_{\pi_{\Loc},*}\rmV_{2})(s)}&\leq\beta_{\Glob} \Erw_{A_{\Loc}\sim\pi_{\Loc}(\cdot|s)}\left[\max_{a_{\Glob}\in\A_{\Glob}} \abs{\Erw_{S'\sim\rmP(\cdot|s,A_{\Loc},a_{\Glob})}\left[\rmV_{1}(S')-\rmV_{2}(S') \right]} \right]\\
&\leq\beta_{\Glob} \Erw_{A_{\Loc}\sim\pi_{\Loc}(\cdot|s)}\left[\max_{a_{\Glob}\in\A_{\Glob}} \Erw_{S'\sim\rmP(\cdot|s,A_{\Loc},a_{\Glob})}\left[\abs{\rmV_{1}(S')-\rmV_{2}(S')} \right] \right]\\
&\leq \beta_{\Glob} \norm{\rmV_{1}-\rmV_{2}}_{\infty},
\end{split}
\end{equation*}
where we use the basic inequality $\abs{\max_{a}f_{1}(a)-\max_{a'}f_{2}(a')}\le\max_{a}\abs{f_{1}(a)-f_{2}(a)}$. Taking the maximum over $s$ on the L.H.S. of above inequality, we obtain the desired statement. For the second statement, notice that for arbitrary $V,a_{\Loc},a_{\Glob}$:
\begin{equation*}
\rmr^{\Glob}(s,a_{\Loc},a_{\Glob})+\beta_{\Glob}\Erw_{S'\sim\rmP(\cdot|s,a_{\Loc},a_{\Glob})}\left[\rmV(S') \right] \leq \argmax_{a'_{\Glob}}\left(\rmr^{\Glob}(s,a_{\Loc},a'_{\Glob})+\beta_{\Glob}\Erw_{S'\sim\rmP(\cdot|s,a_{\Loc},a'_{\Glob})}\left[\rmV(S') \right] \right). 
\end{equation*}
Since $a_{\Loc}$ and $a_{\Glob}$ are arbitrary, we have:
Consequently:
\begin{equation*}
\rmr^{\Glob}(s,A_{\Loc},A_{\Glob})+\beta_{\Glob}\Erw_{S'\sim\rmP(\cdot|s,A_{\Loc},A_{\Glob})}\left[\rmV(S') \right] \leq \argmax_{a'_{\Glob}}\left(\rmr^{\Glob}(s,A_{\Loc},a'_{\Glob})+\beta_{\Glob}\Erw_{S'\sim\rmP(\cdot|s,A_{\Loc},a'_{\Glob})}\left[\rmV(S') \right] \right). 
\end{equation*}
Taking the expectation w.r.t. $A_{\Glob}\sim\pi_{\Glob}(\cdot|s,A_{\Loc})$ and subsequently w.r.t. $A_{\Loc}\sim\pi_{\Loc}(\cdot|s)$, we obtain the desired statement.
\end{proof}
The Bellman equation of our interest takes the form:
\begin{equation}
\label{Eq:sjsjsgsgsgfsfssssss2}
\tildrmV=\rmH^{\Glob}_{\pi_{\Loc},*}\tildrmV.
\end{equation}
Let us define $\tildrmV^{\Glob}_{\pi_{\Loc}}$ as the unique solution of the above fixed point equation.
The fact that the solution of above equation uniquely exists follows from the fact that $\rmH^{\Glob}_{\pi_{\Loc},*}$ is a contraction mapping (Lemma \ref{Lem:ajajgsgsghsgffsfsfdfssss}). The following theorem gives the desired description of $\rmV^{*,\Glob}_{\pi_{\Loc}}$ by means of a Bellman equation:
\begin{theorem}
\label{Thm:jjajggggshsgsgsssss}
The value $V^{*,\Glob}_{\pi_{\Loc}}$ of the optimization problem \eqref{Eq:sjsjsgsgsgfsfssssss} coincides with the solution of the fixed point equation \eqref{Eq:sjsjsgsgsgfsfssssss2}
\end{theorem}
\begin{proof}
By the second statement in Lemma \ref{Lem:ajajgsgsghsgffsfsfdfssss}, we have for any $n\in\nat$:
\begin{equation*}
(\rmH^{\Glob}_{\pi_{\Loc},\pi_{\Glob}})^{n}\tildrmV^{\Glob}_{\pi_{\Loc}}\leq(\rmH^{\Glob}_{\pi_{\Loc},*})^{n}\tildrmV^{\Glob}_{\pi_{\Loc}}=\tildrmV^{\Glob}_{\pi_{\Loc}},
\end{equation*}
where the equality follows from the definition of $\tildrmV^{\Glob}_{\pi_{\Loc}}$ as the solution of the fixed point equation \eqref{Eq:sjsjsgsgsgfsfssssss2}. As $n\rightarrow\infty$, we have that $(\rmH^{\Glob}_{\pi_{\Loc},\pi_{\Glob}})^{n}\tildrmV^{\Glob}_{\pi_{\Loc}}\rightarrow\rmV^{\Glob}_{\pi_{\Loc},\pi_{\Glob}}$ since $\rmH^{\Glob}_{\pi_{\Loc},\pi_{\Glob}}$ is a contraction, and since $\rmV^{\Glob}_{\pi_{\Loc},\pi_{\Glob}}$ is the solution of the fixed point equation with $\rmH^{\Glob}_{\pi_{\Loc},\pi_{\Glob}}$. Since $\pi_{\Glob}$ is arbitrary, we have as a consequence:
\begin{equation}
\label{Eq:ajajsshsggsgsfsfsss}
\rmV^{*,\Glob}_{\pi_{\Loc}}=\max_{\pi_{\Glob}}\rmV^{\Glob}_{\pi_{\Loc},\pi_{\Glob}}\leq \tildrmV^{\Glob}_{\pi_{\Loc}}.
\end{equation}

Now, for the reverse inequality, take a policy $\pi^{*}_{\Glob}\in\Delta_{\S\times\A_{\Loc}}(\A_{\Glob})$ with:
\begin{equation}
\label{Eq:ajjahsgsgsgsgssssss}
\supp(\pi^{*}_{\Glob}(\cdot|s,a_{\Loc}))\subseteq\argmax_{a_{\Glob}}\left( \rmr^{\Glob}(s,a_{\Loc},a_{\Glob})+\beta_{\Glob}\Erw_{S'\sim\rmP(\cdot|s,a_{\Loc},a_{\Glob})}\left[\tildrmV^{\Glob}_{\pi_{\Loc}}(S') \right]\right). 
\end{equation}
By this definition, we have that $\tildrmV^{\Glob}_{\pi_{\Loc}}$ is the value function of the policy $(\pi_{\Loc},\pi^{*}_{\Glob})$ in the discounted MDP $(\S,\A_{\Loc}\times\A_{\Glob},\rmr^{\Glob},\rmP)$. Consequently:
\begin{equation}
\label{Eq:ajajsshsggsgsfsfsss2}
\tildrmV^{\Glob}_{\pi_{\Loc}}=\rmV_{\pi_{\Loc},\pi_{\Glob}^{*}}\leq \max_{\pi_{\Glob}}\rmV_{\pi_{\Loc},\pi_{\Glob}}=\rmV^{*,\Glob}_{\pi_{\Loc}}
\end{equation}
\end{proof}

Now, our aim is to compute the policy solving the optimization problem \eqref{Eq:sjsjsgsgsgfsfssssss}. For this task, Theorem \ref{Thm:jjajggggshsgsgsssss} provides the tool, since it gives the hint that one can analyze the fixed point equation \eqref{Eq:sjsjsgsgsgfsfssssss2} instead of \eqref{Eq:sjsjsgsgsgfsfssssss}. We continue for our actual purpose by defining:
\begin{equation*}
(\rmT^{\Glob}_{\pi_{\Loc},*}\rmQ)_{a_{\Loc}}(s,a_{\Glob}):=\rmr^{\Glob}(s,a_{\Loc},a_{\Glob})+\beta_{\Glob}\Erw_{S'\sim\rmP(\cdot|s,a_{\Loc},a_{\Glob})}\left[\Erw_{A_{\Loc}\sim\pi_{\Loc}(\cdot|S')}\left[\max_{a_{\Glob}'\in\A_{\Glob}} \rmQ_{A_{\Loc}}(S',a_{\Glob}')\right]  \right], 
\end{equation*}
and the Q-function $\tilde{\rmQ}^{\Glob}_{\pi_{\Loc}}$ by: 
\begin{equation}
\label{Eq:ahjahshshsgsgsfsfsfsggsss}
\tilde{\rmQ}^{\Glob}_{\pi_{\Loc},a_{\Loc}}(s,a_{\Glob}):=\rmr^{\Glob}(s,a_{\Loc},a_{\Glob})+\beta_{\Glob}\Erw_{S'\sim\rmP(\cdot|s,a_{\Loc},a_{\Glob})}\left[\tildrmV^{\Glob}_{\pi_{\Loc}}(S') \right]. 
\end{equation}
It holds:
\begin{equation*}
\tildrmV^{\Glob}_{\pi_{\Loc}}(s)=\Erw_{A_{\Loc}\sim\pi_{\Loc}(\cdot|s)}\left[\max_{a_{\Glob}\in \A_{\Glob}}\tilde{\rmQ}^{\Glob}_{\pi_{\Loc},A_{\Loc}}(s,a_{\Glob}) \right], 
\end{equation*}
and consequently, $\tilde{\rmQ}^{\Glob}_{\pi_{\Loc}}$ is the unique solution of:
\begin{equation}
\label{Eq:jajajahshshsgsgsg}
\tilde{\rmQ}^{\Glob}_{\pi_{\Loc}}=\rmT^{\Glob}_{\pi_{\Loc},*}\tilde{\rmQ}^{\Glob}_{\pi_{\Loc}}.
\end{equation} 
From here, we can infer the following statement on the solution of \eqref{Eq:sjsjsgsgsgfsfssssss}: 
\begin{lemma}
\label{Lem:ajajajagsgsgsfsfsffsfssgsfssss}
Let $\tilde{\rmQ}^{\Glob}_{\pi_{\Loc}}$ be the unique solution of the Bellman equation \eqref{Eq:jajajahshshsgsgsg}. Then $\pi_{\Glob}^{*}\in\Delta_{\S}(\A_{\Loc})$ is a solution of the optimization problem \eqref{Eq:sjsjsgsgsgfsfssssss} if 
:
\begin{equation}
\label{Eq:ajjahshsggsgsggsfsfssss}
\supp(\pi_{\Glob}^{*}(\cdot|s,a_{\Loc}))\subseteq\argmax_{a_{\Glob}\in\A_{\Glob}}\tilde{\rmQ}^{\Glob}_{\pi_{\Loc},a_{\Loc}}(s,a_{\Glob})
\end{equation}
is a solution of the optimization problem \eqref{Eq:sjsjsgsgsgfsfssssss} is 
\end{lemma}
\begin{proof}
We first show that $\pi_{\Glob}^{*}$ satisfying \eqref{Eq:ajjahshsggsgsggsfsfssss} is a solution of \eqref{Eq:sjsjsgsgsgfsfssssss}. By \eqref{Eq:ahjahshshsgsgsfsfsfsggsss}, it follows that $\pi_{\Glob}^{*}$ is a policy satisfying \eqref{Eq:ajjahsgsgsgsgssssss}. Consequently we have from \eqref{Eq:ajajsshsggsgsfsfsss2} and \eqref{Eq:ajajsshsggsgsfsfsss} as desired $\rmV^{*,\Glob}_{\pi_{\Loc}}=\rmV^{\Glob}_{\pi_{\Loc},\pi_{\Glob}^{*}}$.
\end{proof}
According to above lemma, we can find the desired optimal strategy for GA by solving the Bellman equation \eqref{Eq:jajajahshshsgsgsg}.  

To find the solution of \eqref{Eq:jajajahshshsgsgsg}, we proposed the method described in Algorithm 	\ref{Alg:ohgfdfddrrdrdrdrdrdghhgggdsadxxx5}. In contrast to the previous GA Q-learning algorithm (GAQL), Algorithm \ref{Alg:ohgfdfddrrdrdrdrdrdghhgggdsadxxx5} requires on each step more observation of LA's action, i.e. two consecutive LA's actions. Therefore we call Algorithm 	\ref{Alg:ohgfdfddrrdrdrdrdrdghhgggdsadxxx5} as extra information GA Q-learning (EIGAQL). However, this additional feedback effectuates in the optimality of GA policy, not only given that LA applies deterministic stationary strategy, but more general: given that LA applies arbitrary stationary strategy.  
\begin{algorithm}[htbp]
	\caption{Extra Information GA Q-Learning (EIGAQL)}
	\begin{algorithmic}
		\STATE \textbf{Extrinsic Parameter:} GA reward $\rmr^{\Glob}$, system dynamic $\rmP$, LA stationary policy $\pi_{\Loc}$
		\STATE \textbf{GA parameter:} Policy generator $\Psi^{\Glob}:\real^{\S\times\A_{\Loc}\times\A_{\Glob}}\rightarrow \Delta_{\S\times\A_{\Loc}}(\A_{\Glob})$, learning rate $\gamma_{t}^{\Loc}$, discount factor $\beta^{\Loc}$ 
		\STATE \textbf{Initialization:}$Q^{\EI}_{0}\in\real^{\S\times\A_{\Loc}\times\A_{\Glob}}$ and $S_{0}\in\mathcal{S}$
		
		\FOR{$t=0,2,\ldots$}
		
	\STATE Execute the action $A_{t}^{\Glob}\sim\eta_{t}^{\EI}(\cdot|S_{t},A_{t}^{\Loc})$, according to the GA learning policy $\eta_{t}^{\EI}:=\Psi^{\Glob}Q^{\EI}_{t}$
		\STATE  Experience the reward 	$R^{\Glob}_{t}:=\rmr^{\Glob}(S_{t},A^{\Glob}_{t},A^{\Loc}_{t})$
		\STATE Query LAs' action $A^{\Loc}_{t}\sim\pi_{\Loc}(\cdot|S_{t})$ and the system state update $S_{t+1}\sim\rmP(\cdot|s,A_{t}^{\Loc},A_{t}^{\Glob})$
		\STATE Query next LAs' action $A^{\Loc}_{t+1}\sim\pi_{\Loc}(\cdot|S_{t+1})$
		\STATE Update $$Q^{\EI}_{A^{\Loc}_{t},t+1}(S_{t},A^{\Glob}_{t})=(1-\gamma^{\Glob}_{t}) Q^{\EI}_{A^{\Loc}_{t},t}(S_{t},A^{\Glob}_{t})+\gamma^{\Glob}_{t}H^{\Glob}_{t+1},$$

		where:
		 	$$H^{\Glob}_{t+1}=R^{\Glob}_{t}+\beta_{\Glob}\max_{a_{\Glob}\in \A_{\Glob}}Q^{\EI}_{A^{\Loc}_{t+1},t}(S_{t+1},a_{\Glob}).$$
		\FOR{all $(s,a_{\Glob},a_{\Loc})\neq (S_{t},A^{\Glob}_{t},A^{\Loc}_{t})$}
		\STATE Update $Q^{\EI}_{a_{\Loc},t+1}(s,a_{\Glob})=Q^{\EI}_{a_{\Loc},t}(s,a_{\Glob})$
		\ENDFOR
		\ENDFOR
	\end{algorithmic}
	\label{Alg:ohgfdfddrrdrdrdrdrdghhgggdsadxxx5}
\end{algorithm}
The fact that indeed Algorithm \ref{Alg:ohgfdfddrrdrdrdrdrdghhgggdsadxxx5} yields the solution of the Bellman equation \eqref{Eq:jajajahshshsgsgsg} and therefore the desired optimal GA policy is given in the following theorem:
\begin{theorem}[Convergence of EIGAQL]
Suppose that:
\begin{equation}
\label{Eq:jajaajasggsgsfdffdfddddd4}
\sum_{t=0}^{\infty}\psi_{t}^{\Glob}=\infty\quad\text{and}\quad \sum_{t=0}^{\infty}\psi_{t}^{\Glob,2}<\infty,   
\end{equation}
where:
\begin{equation*}
\psi_{t}^{\Glob}(s,a_{\Loc},a_{\Glob}):=\mathbf{1}_{\lrbrace{S_{t}=s,A_{t}^{\Loc}=a_{\Loc},A_{t}^{\Glob}=a_{\Glob}}}\gamma_{t}.
\end{equation*}
Then a.s. the iterate $(Q^{\EI}_{t})_{t\in\nat_{0}}$ of EIQGL converges and we have that $\pi_{\Glob}^{*}\in\Delta_{\S\times\A_\Loc}(\A_{\Glob})$ satisfying:
\begin{equation*}
\supp(\pi_{\Glob}^{*}(\cdot|s,a_{\Loc}))\subseteq\argmax_{a_{\Glob}\in\A_{\Glob}}\lim_{t\rightarrow\infty }Q^{\EI}_{a_{\Loc},t}(s,a_{\Glob}).
\end{equation*}
Is a solution of the optimization problem \eqref{Eq:sjsjsgsgsgfsfssssss}. 
\end{theorem}
\begin{proof}
	Let $\hatrmT^{\Glob}_{t}:\real^{\S\times\A_{\Loc}\times\A_{\Glob}}\rightarrow\real^{\S\times\A_{\Loc}\times\A_{\Glob}}$ be the operator given by:
	\begin{equation*}
	(\hatrmT^{\Glob}_{t}\rmQ)_{a_{\Loc}}(s,a_{\Glob}):=\tildrmr_{a_{\Loc}}^{\Glob}(s,a_{\Glob})+\beta_{\Glob}\max_{a_{\Glob}\in\A_{\Glob}}\rmQ_{A^{\Loc}_{t+1}}(S_{t+1},a_{\Glob}),
	\end{equation*}
	We can write the iterate of Algorithm \ref{Alg:ohgfdfddrrdrdrdrdrdghhgggdsadxxx5} in the form:
	\begin{equation}
	\label{Eq:aoasjsshhssggsggsgsgsshhssfaaaddddt}
	Q^{\EI}_{t+1}=(1-\psi^{\Glob}_{t})\odot Q^{\EI}_{t}+\psi^{\Glob}_{t}\odot\left[ \rmT^{\Glob}_{\pi_{\Loc},*}Q^{\EI}_{t}+W^{\Glob}_{t+1}\right] ,
	\end{equation}
	where:
	\begin{equation*}
	W^{\Glob}_{a_{\Loc},t+1}(s,a_{\Glob}):=\mathbf{1}_{\lrbrace{S_{t}=s,A^{\Loc}_{t}=a_{\Loc},A^{\Glob}_{t}=a_{\Glob}}}\left((  \hatrmT^{\Glob}_{t}Q^{\EI}_{t})_{a_{\Loc}}(s,a_{\Glob})-(\rmT^{\Glob}_{\pi_{\Loc},*}Q^{\EI}_{t})_{a_{\Loc}}(s,a_{\Glob})\right),
	\end{equation*}
	and where:
	\begin{equation*}
	\psi_{a_{\Loc},t}^{\Glob}(s,a_{\Loc})=\mathbf{1}_{\lrbrace{S_{t}=s,A^{\Loc}_{t}=a_{\Loc},A^{\Glob}_{t}=a^{\Glob}_{t}}}\gamma_{t}^{\Glob}.
	\end{equation*}
	\eqref{Eq:aoasjsshhssggsggsgsgsshhssfaaaddddt} shows that the iterate of the EIGAQL has the form
	\eqref{Eq:aahshsggsgsfsfsfsssssssss}, it is sufficient to show the desired statement by checking the conditions of Proposition \ref{Prop:ajjahhssggsgdddddd}. First, it holds $\Erw[W^{\Glob}_{t+1}|\tildF_{t}]=0$. Indeed, we have:
\begin{equation*}
\begin{split}
\Erw[(\hatrmT^{\Glob}_{t}Q_{t}^{\EI})_{A_{t}^{\Loc}}(S_{t},A_{t}^{\Glob})|\tildF_{t}]&=
\Erw\left[\left.  \tildrmr_{A^{\Loc}_{t}}^{\Glob}(S_{t},A^{\Glob}_{t})+\beta_{\Glob}\max_{a_{\Glob\in\A_{\Glob}}'}\rmQ^{\EI}_{A^{\Loc}_{t+1}}(S_{t+1},a_{\Glob}')\right| \tildF_{t}\right] \\
&= \tildrmr_{A^{\Loc}_{t}}^{\Glob}(S_{t},A^{\Glob}_{t})+\beta_{\Glob}\Erw\left[\left.  \max_{a^{'}_{\Glob\in\A_{\Glob}}}\rmQ^{\EI}_{A^{\Loc}_{t+1}}(S_{t+1},a^{'}_{\Glob})\right| \tildF_{t}\right]
\end{split}
\end{equation*}
Furthermore:
\begin{equation*}
\begin{split}
&\Erw\left[\left.  \max_{a^{'}_{\Glob\in\A_{\Glob}}}\rmQ^{\EI}_{A^{\Loc}_{t+1}}(S_{t+1},a^{'}_{\Glob})\right| \tildF_{t}\right]=\Erw\left[\left.\Erw\left[\left.  \max_{a^{'}_{\Glob\in\A_{\Glob}}}\rmQ^{\EI}_{A^{\Loc}_{t+1}}(S_{t+1},a^{'}_{\Glob})\right| \tildF_{t},S_{t+1}\right]\right| \tildF_{t}\right]\\
&=\left[\left.\Erw_{A_{\Loc}\sim\pi_{\Loc}(\cdot|S_{t+1})}\left[  \max_{a^{'}_{\Glob\in\A_{\Glob}}}\rmQ^{\EI}_{A^{\Loc}_{t+1}}(S_{t+1},a^{'}_{\Glob})\right]\right| \tildF_{t}\right]=\Erw_{S'\sim\rmP(\cdot|S_{t},A^{\Loc}_{t},A^{\Glob}_{t})}\left[\Erw_{A^{'}_{\Loc}\sim\pi_{\Loc}(\cdot|S')}\left[\max_{a_{\Glob}^{'}\in\A_{\Glob}}Q^{\EI}_{A'_{\Loc}}(S',a_{\Glob}) \right]  \right]. 
\end{split}
\end{equation*}
Therefore:
\begin{equation*}
\begin{split}
\Erw[(\hatrmT^{\Glob}_{t}Q_{t}^{\EI})_{A_{t}^{\Loc}}(S_{t},A_{t}^{\Glob})|\tildF_{t}]&=\tildrmr_{A^{\Loc}_{t}}^{\Glob}(S_{t},A^{\Glob}_{t})+\beta_{\Glob}\Erw_{S'\sim\rmP(\cdot|S_{t},A^{\Loc}_{t},A^{\Glob}_{t})}\left[\Erw_{A^{'}_{\Loc}\sim\pi_{\Loc}(\cdot|S')}\left[\max_{a_{\Glob}^{'}\in\A_{\Glob}}Q^{\EI}_{A'_{\Loc}}(S',a_{\Glob}) \right]  \right]=(\rmT^{\Glob}_{\pi_{\Loc},*}Q^{\EI}_{t})_{A_{t}^{\Loc}}(S_{t},A_{t}^{\Glob}) \\
&=\Erw[(\rmT^{\Glob}_{\pi_{\Loc},*}Q_{t}^{\EI})_{A_{t}^{\Loc}}(S_{t},A_{t}^{\Glob})|\tildF_{t}].
\end{split}
\end{equation*}
This consequences as desired in:
\begin{equation*}
\begin{split}
\Erw[W^{\Glob}_{a_{\Loc},t+1}(s,a_{\Glob})|\tildF_{t}]&=\Erw\left[\left.\mathbf{1}_{\lrbrace{S_{t}=s,A^{\Loc}_{t}=a_{\Loc},A^{\Glob}_{t}=a_{\Glob}}}\left(  (\hatrmT^{\Glob}_{t}Q_{t}^{\EI})_{a_{\Loc}}(s,a_{\Glob})-(\rmT^{\Glob}_{\pi_{\Loc},*}Q_{t}^{\EI})_{a_{\Loc}}(s,a_{\Glob})\right)\right|\tildF_{t} \right]\\
&=\Erw\left[\left.\mathbf{1}_{\lrbrace{S_{t}=s,A^{\Loc}_{t}=a_{\Loc},A^{\Glob}_{t}=a_{\Glob}}}\left(  (\hatrmT^{\Glob}_{t}Q_{t}^{\EI})_{A_{t}^{\Loc}}(S_{t},A_{t}^{\Glob})-(\rmT^{\Glob}_{\pi_{\Loc},*}Q_{t}^{\EI})_{A_{t}^{\Loc}}(S_{t},A_{t}^{\Glob})\right)\right|\tildF_{t} \right]\\
&=\mathbf{1}_{\lrbrace{S_{t}=s,A^{\Loc}_{t}=a_{\Loc},A^{\Glob}_{t}=a_{\Glob}}}\Erw\left[\left.\left(  (\hatrmT^{\Glob}_{t}Q_{t}^{\EI})_{A_{t}^{\Loc}}(S_{t},A_{t}^{\Glob})-(\rmT^{\Glob}_{\pi_{\Loc},*}Q_{t}^{\EI})_{A_{t}^{\Loc}}(S_{t},A_{t}^{\Glob})\right)\right|\tildF_{t} \right] =0.
\end{split}
\end{equation*}
Now, similar argumentation as in the proof of Lemma \ref{Lem:aajsjshhsggdddhhdd} yields:
\begin{equation*}
\Erw[(W_{a_{\Loc},t+1}^{\Glob}(s,a_{\Glob}))^{2}|\tildF_{t}]\leq 2\left( \norm{\rmr^{\Glob}}_{\infty}^{2}+\beta_{\Loc}^{2}\norm{Q^{\EI}_{t}}_{\infty}\right). 
\end{equation*}
The remaining condition which we need to show is the third condition of Proposition \ref{Prop:ajjahhssggsgdddddd}. This is shown by the following computation:
\begin{equation*}
\begin{split}
\norm{\rmT^{\Glob}_{\pi_{\Loc},*}\rmQ-\tildrmQ^{\Glob}_{\pi_{\Loc}}}_{\infty}=\norm{\rmT^{\Glob}_{\pi_{\Loc},*}\rmQ-\rmT^{\Glob}_{\pi_{\Loc},*}\tildrmQ^{\Glob}_{\pi_{\Loc}}}_{\infty}\leq\beta_{\Glob}\norm{\rmQ-\tildrmQ^{\Glob}_{\pi_{\Loc}}}_{\infty},
\end{split}
\end{equation*}
where the equality follows from the fixed point definition of $\tildrmQ^{\Glob}_{\pi_{\Loc}}$, and the inequality follows from similar computation as done before in the proof of Theorem \ref{Thm:akakashshssggdhdhdggdgdhdhd} (see \eqref{Eq:akakashshssggdhdhdggdgdhdhd}). Consequently, we obtain by Proposition \ref{Prop:ajjahhssggsgdddddd} the fact that the iterate of EIGAQL converges to $\tildrmQ^{\Glob}_{\pi_{\Loc}}$. Finally, we obtain the remaining statement from Lemma \ref{Lem:ajajajagsgsgsfsfsffsfssgsfssss}.  
\end{proof}
\section{Existence of Nash Equilibrium of Optimality of LA-GA Game}
\label{Sec: Nash}
As the joint application of LAQGI and GAQL only yields an almost Nash equilibrium (Theorem 	\ref{Thm:ajajhshssggsgssgsgffsfssssss}), it is natural to ask whether an indeed one exists for the underlying LA-GA stochastic game. Formally, a Nash equilibrium $(\pi_{(i)}^*,\pi_{(-i)}^*)$, where $i \in \{\Loc,\Glob\}$, for LA-GA stochastic game is a tuple of strategy satisfying:
\begin{equation}
    \rmV^{(i)}_{\pi_{(i)}^*,\pi_{(-i)}^*} = \max_{\pi_{(i)}}  \rmV^{(i)}_{\pi_{(i)},\pi_{(-i)}^*}.
    \label{NEQ}
\end{equation}
Our aim in this section is to show the existence of such object:
\begin{theorem}
\label{Thm:shhsgsgfsffddddddd}
There exists a Nash equilibrium for the local-global stochastic game.
\end{theorem}
The proof of this theorem follows the approach in \cite{Fink1964}, and is divided into two steps which is given in the following. First, we characterize the Nash equilibrium of this particular game and show that it can be described by a suitable fixed point equation using the previously defined operators (Theorem \ref{Thm:RBENash}). Thus, if a fixed point exists, it coincides with the Nash equilibrium. In the second step we prove the existence of a fixed point using Kakutanis fixed point theorem (Theorem \ref{Thm:RBEExists}). The formal proof of Theorem \ref{Thm:shhsgsgfsffddddddd} is given in the following:

First, we characterize the Nash equilibrium by a suitable fixed point equation using the operator $\rmH_{\pi_{\Loc},\pi_{\Glob}}: \real^\S\times\real^\S \to \real^\S\times\real^\S$ defined as:
\begin{equation}
\label{Eq:OpOp}
	(\rmH_{\pi_{\Loc},\pi_{\Glob}} \rmV)(s) := \begin{bmatrix}
		(\rmH_{\pi_{\Loc},\pi_{\Glob}}^{\Loc} \rmV_{\Loc})(s)\\
		(\rmH_{\pi_{\Loc},\pi_{\Glob}}^{\Glob} \rmV_{\Glob})(s)
	\end{bmatrix}(s) := \Erw_{\substack{A_{\Loc}\sim\pi_{\Loc}(\cdot|s)\\A_{\Glob}\sim\pi_{\Glob}(\cdot|s,A_{\Loc})}} \begin{bmatrix}
	\rmr^{\Loc}(s,A_{\Loc},A_{\Glob}) + \beta_{\Loc}\Erw_{S'\sim \rmP(\cdot|s,A_{\Loc},A_{\Glob})}[\rmV_{\Loc}(S')] \\
		\rmr^{\Glob}_s(s,A_{\Loc},A_{\Glob}) + \beta_{\Glob}\Erw_{S'\sim \rmP(\cdot|s,A_{\Loc},A_{\Glob})}[\rmV_{\Glob}(S')]
	\end{bmatrix}
\end{equation}
and the operator $\rmH_{\pi_{\Loc},\pi_{\Glob}}^*:\mathcal{S}^2 \to \mathcal{S}^2$ defined as:
\begin{equation*}
\begin{split}
	(\rmH_{\pi_{\Loc},\pi_{\Glob}}^* \rmV)(s) &:= \begin{bmatrix}
		(\rmH_{*,\pi_{\Glob}}^{\Loc} \rmV_{\Loc})(s)\\
		(\rmH_{\pi_{\Loc},*}^{\Glob} \rmV_{\Glob})(s)
	\end{bmatrix}\\
	&:= \begin{bmatrix}
		\max_{\pi_{\Loc} \in \Delta_{\S}(\A_{\Loc})} \Erw_{\substack{A_{\Loc}\sim\pi_{\Loc}(\cdot|s)\\A_{\Glob}\sim\pi_{\Glob}(\cdot|s,A_{\Loc})}}\left[\rmr^{\Loc}(s,A_{\Loc},A_{\Glob}) + \beta_{\Loc}\Erw_{S'\sim \rmP(\cdot|s,A_{\Loc},A_{\Glob})}[\rmV_{\Loc}(S')]\right] \\
		\Erw_{A_{\Loc}\sim\pi_{\Loc}(\cdot|s)}\left[\max_{\pi_{\Glob} \in \Delta_{\S\times\A_{\Loc}}(\A_{\Glob})} \Erw_{A_{\Glob}\sim\pi_{\Glob}(\cdot|s,A_{\Loc})}[\rmr^{\Glob}(s,A_{\Loc},A_{\Glob}) + \beta_{\Glob}\Erw_{S'\sim \rmP(\cdot|s,A_{\Loc},A_{\Glob})}[\rmV_{\Glob}(S')]]\right].
	\end{bmatrix}
	\end{split}
\end{equation*}
Here, we use the notation $\rmV=(\rmV_{\Loc},\rmV_{\Glob})$.
Thus, if a fixed point exists, it coincides with the Nash equilibrium. In the next step we prove the existence of a fixed point using Kakutani's fixed point theorem (Theorem 17). First, we provide in the following some basic properties of the operators introduced above. For better readability, we give the corresponding proof in the Appendix (Subsection \ref{Subsec:ProofNash}).
\begin{lemma}
\label{Lem:ajhajasggffsfsfsfsss}
    For all $\pi_{\Loc}\in\Delta_{\S}(\A_{\Loc})$ and $\pi_{\Glob}\in\Delta_{\S\times\A_{\Loc}}(\A_{\Glob})$, $\rmH^{\Loc}_{\pi_{\Loc},\pi_{\Glob}}\leq\rmH^{\Loc}_{*,\pi_{\Glob}}$ and   $\rmH^{\Glob}_{\pi_{\Loc},\pi_{\Glob}}\leq\rmH^{\Glob}_{\pi_{\Loc},*}$.
\end{lemma}
\begin{lemma} 
\label{Lem:aidgfgfgefefeeeeeeeeee}
The solution $\rmV$ of the Bellman equation $\rmV=\rmH_{\pi_{\Loc},\pi_{\Glob}}\rmV$ is bounded in the sense that:
    \begin{equation*}
        \|\rmV\|_{\infty} \leq \frac{\|\rmr\|_{\infty}}{1-\beta},
    \end{equation*}
    where $\beta = \max_i \beta_i$, and $\norm{\rmr}_{\infty}:=\max_{s\in\S,i\in\lrbrace{\Loc,\Glob},(a_{\Loc},a_{\Glob})\in\A_{\Loc}\times\A_{\Glob}} |\rmr^{(i)}(s,a_{\Loc},a_{\Glob})|$ 
\end{lemma}
\begin{lemma}
\label{Lem:ahjahshsgsgsfffsgfgssssss}
	The operator $\rmH_{\pi_{\Loc},\pi_{\Glob}}\rmV$ is Lipschitz continuous in each argument $(\pi_{\Loc},\pi_{\Glob})\in\Delta_{\S}(\A_{\Loc})\times \Delta_{\S\times\A_{\Loc}}(\A_{\Glob})$ and $\rmV\in\real^{\S}\times\real^{\S}$.
\end{lemma}
In the following, we define the fixed point equation of our interest:
\begin{definition}
We say the triple $(\rmV,\pi_{\Loc}^{*},\pi_{\Glob}^{*})\in\real^{\S^{2}}\times\Delta_{\S}(\A_{\Loc})\times\Delta_{\S\times\A_{\Loc}}(\A_{\Glob})$ satisfies the \textit{reciprocal Bellman equation} (RBE) for the local-global stochastic game if:
\begin{equation}
\rmV^* = \rmH_{\pi_{\Loc}^*,\pi_{\Glob}^*}\rmV^*\quad\text{and}\quad\rmV^* = \rmH_{\pi_{\Loc}^{*},\pi_{\Glob}^{*}}^* \rmV^*
\label{RBE}
\end{equation}
\end{definition}
The following Theorem gives the connection between above concept and the concept of Nash equilibrium of the local-global stochastic game:
\begin{theorem}
\label{Thm:RBENash}
Let $(\rmV^{*},\pi_{\Loc}^*,\pi_{\Glob}^*)$ be a triple satisfying the RBE of the local-global stochastic game. Then $(\pi_{\Loc}^*,\pi_{\Glob}^*) \in\Delta_{\S}(\A_{\Loc})\times\Delta_{\S\times\A_{\Loc}}(\A_{\Glob})$ is a Nash equilibrium for the local-global stochastic game. 
\end{theorem}
\begin{proof}
Denote $\rmV^{*}=(\rmV^{*}_{\Loc},\rmV^{*}_{\Glob})$. Let be $i\in \lrbrace{\Loc,\Glob}$. For any $n \in \mathbb{N}$ and $\pi_{(i)}$, it holds by Lemma \ref{Lem:ajhajasggffsfsfsfsss}:
\begin{equation}
\label{Eq:ajkajajsgsgsgsffsssssss}
    \underbrace{(\rmH_{\pi_{(i)},\pi_{(-i)}^*}^{(i)})^n \rmV_i^*}_{\rightarrow \rmV_{\pi_{(i)},\pi_{(-i)}^*}^{(i)}} \leq \underbrace{(\rmH_{*,\pi_{(-i)}^*}^{(i)})^n \rmV_i^*}_{=\rmV_i^*}
\end{equation}
Notice that $\rmH_{\pi_{(i)},\pi_{(-i)}^*}^{(i)}$ is a contraction (c.f. \eqref{Eq:akaksjsssgggsgsssssss}). Now, the value function $\rmV_{\pi_{(i)},\pi_{(-i)}^*}^{(i)}$ of $i$ given the population strategy $(\pi_{(i)},\pi_{(-i)}^*)$ is the fixed point of the Bellman equation specified by $\rmH_{\pi_{(i)},\pi_{(-i)}^*}^{(i)}$. Moreover by definition of RBE, $\rmV_i^*$ is the solution of the Bellman equation specified by $\rmH_{*,\pi_{(-i)}^*}^{(i)}$. Those observations yield $(\rmH_{\pi_{(i)},\pi_{(-i)}^*}^{(i)})^n \rmV_i^*\rightarrow \rmV_{\pi_{(i)},\pi_{(-i)}^*}^{(i)}$ as $n\rightarrow\infty$ and $(\rmH_{*,\pi_{(-i)}^*}^{(i)})^n \rmV_i^{*}=\rmV_{i}^{*}$. Setting this into \eqref{Eq:ajkajajsgsgsgsffsssssss} and since $\pi_{(i)}$ is arbitrary we have:
\begin{equation*}
    \max_{\pi_{(i)}}  \rmV^{(i)}_{\pi_{(i)},\pi_{(-i)}^*} \leq \rmV_i^*
\end{equation*}
For the reverse inequality, notice that since $\rmV_i^*$ satisfies $\rmV_i^* = \rmH_{\pi_{\Loc}^*,\pi_{\Glob}^*}^{(i)}\rmV_i^*$, we have: 
\begin{equation*}
    \rmV_i^* =  \rmV^{(i)}_{\pi_{(i)}^*,\pi_{(-i)}^*} \leq \max_{\pi_{(i)}}  \rmV^{(i)}_{\pi_{(i)},\pi_{(-i)}^*},
\end{equation*}
where the equality follows by the uniqueness of the solution of the Bellman equation specified by the contractive operator $\rmH_{\pi_{\Loc}^*,\pi_{\Glob}^*}^{(i)}$ having the solution $\rmV^{(i)}_{\pi_{(i)}^*,\pi_{(-i)}^*}$, as desired.
\end{proof}

What remains for the proof of Theorem \ref{Thm:shhsgsgfsffddddddd} is to show the existence of a fixed point, i.e. a solution $(\rmV^{*},\pi_{\Loc}^*,\pi_{\Glob}^*)$ of the RBE.
We aim to proof the existence of such a solution using Kakutanis fixed point theorem, following a similar argumentation as \cite{Fink1964}. To do so, we define first the following set:
\begin{equation}
	\Sigma(\pi_{\Loc},\pi_{\Glob}) 
	:= \left\{ \left.\rmV=\begin{bmatrix}
		\rmV_{\Loc}\\
		\rmV_{\Glob}
	\end{bmatrix}\in\real^{\S^{2}} \right|  \rmV = \rmH_{\pi_{\Loc},\pi_{\Glob}}^* \rmV
	\right\},\quad (\pi_{\Loc},\pi_{\Glob})\in\Delta_{\S}(\A_{\Loc})\times\Delta_{\S\times\A_{\Loc}}(\A_{\Glob}).
	\label{sigmadef}
\end{equation}
Additionally, we define in the following the a point-set mapping $\phi(\pi_{\Loc},\pi_{\Glob}):\Delta\rightarrow 2^{\Delta}$, where $\Delta = \Delta_{\mathcal{S}}\times\Delta_{\mathcal{S}\times\mathcal{\mathcal{A}_{\Loc}}}$ and $2^{\Delta}$ is the power set of $\Delta$:
\begin{equation}
	\phi(\pi_{\Loc},\pi_{\Glob})
	:= \left\{ (\pi_{\Loc}^*,\pi_{\Glob}^*)\in\Delta_{\S}(\A_{\Loc})\times\Delta_{\S\times\A_{\Loc}}(\A_{\Glob}) \left| \right. \Sigma(\pi_{\Loc},\pi_{\Glob}) = \rmH_{(\pi_{\Loc},\pi_{\Glob}),(\pi_{\Loc}^*,\pi_{\Glob}^*)} \Sigma(\pi_{\Loc},\pi_{\Glob}) \right\}
	\label{phidef} 
\end{equation}
with the coupling operator 
\begin{equation*}
    \rmH_{(\pi_{\Loc},\pi_{\Glob}),(\pi_{\Loc}^*,\pi_{\Glob}^*)}  := \begin{bmatrix}
        \rmH_{\pi_{\Loc}^*,\pi_{\Glob}}^{\Loc}\\
        \rmH_{\pi_{\Loc},\pi_{\Glob}^*}^{\Glob}.
    \end{bmatrix}
\end{equation*}
The following property of $\phi$, whose proof is given in the appendix (Subsection \ref{Subsec:ProofNash}), is useful for our approach:
\begin{lemma}
\label{Lem:Uppersem}
$\phi: \Delta \to 2^{\Delta}$ is upper semi-continuous, i.e. if a sequence $\{(\pi^n_{\Loc},\pi_{\Glob}^n)\}_{n \in \mathbb{N}}$ in $\Delta$ converges to $(\pi_{\Loc},\pi_{\Glob}) \in \Delta$ and a sequence $\{(\pi^{*,n}_{\Loc},\pi^{*,n}_{\Glob})\}_{n \in \mathbb{N}}$ in $\wp(\Delta)$ with $(\pi^{*,n}_{\Loc},\pi^{*,n}_{\Glob}) \in \phi(\pi^n_{\Loc},\pi_{\Glob}^n)$ converges to $(\pi^{*}_{\Loc},\pi^{*}_{\Glob})$, then $(\pi^{*}_{\Loc},\pi^{*}_{\Glob}) \in \phi(\pi_{\Loc},\pi_{\Glob})$.
\end{lemma}
Now, we are ready to show the remaining step for the proof of Theorem \ref{Thm:shhsgsgfsffddddddd}:
\begin{theorem}
\label{Thm:RBEExists}
There exists a triple $(\rmV^{*},\pi_{\Loc}^*,\pi_{\Glob}^*)$ satisfying the RBE as defined in \eqref{RBE}.
\end{theorem}
\begin{proof}
We have established in Lemma \ref{Lem:Uppersem} the fact that the correspondence $\phi$ is an upper semi-continuous point-set mapping, which maps from a compact convex set $\Delta_{\mathcal{S}}\times\Delta_{\mathcal{S}\times\mathcal{\mathcal{A}_{\Loc}}}$ to the power set $2^{\Delta_{\mathcal{S}}\times\Delta_{\mathcal{S}\times\mathcal{\mathcal{A}_{\Loc}}}}$. By Kakutani's fixed point theorem, it follows that there exists a $(\pi_{\Loc}^*,\pi_{\Glob}^*) \in \Delta_{\mathcal{S}}\times\Delta_{\mathcal{S}\times\mathcal{\mathcal{A}_{\Loc}}}$ with $\rmV^* \in \Sigma(\pi_{\Loc}^*,\pi_{\Glob}^*) $, such that $(\pi_{\Loc}^*,\pi_{\Glob}^*) \in \phi(\pi_{\Loc}^*,\pi_{\Glob}^*)$ In particular, there exists a tuple $(\rmV^*,\pi_{\Loc}^*,\pi_{\Glob}^*)$, such that $\rmV^*= \rmH_{\pi^*,\pi^*} \rmV^*$
where $\pi^{*}=(\pi_{\Loc}^{*},\pi_{\Glob}^{*})$. Now,
we observe that the coupling operator $\rmH_{\pi^*,\pi^*}$ coincides with the operator $ \rmH_{\pi_{\Loc}^{*},\pi_{\Glob}^{*}}$ defined as in \eqref{Eq:OpOp}. Moreover, since $\rmV^* \in \Sigma(\pi_{\Loc}^*,\pi_{\Glob}^*) $ and in conjunction with the definition \eqref{sigmadef}, we have as desired:
\begin{equation*}
\rmV^* = \rmH_{\pi_{\Loc}^*,\pi_{\Glob}^*}\rmV^*\quad\text{and}\quad\rmV^* = \rmH_{\pi_{\Loc}^{*},\pi_{\Glob}^{*}}^* \rmV^*,
\end{equation*}
\end{proof}

At last, we summarize in the following the proof of Theorem \ref{Thm:shhsgsgfsffddddddd}:
\begin{proof}[Proof of Theorem \ref{Thm:shhsgsgfsffddddddd}]
Theorem \ref{Thm:RBEExists} asserts the existence of a triple $(\rmV^*,\pi_{\Loc}^*,\pi_{\Glob}^*)$ satisfying RBE of the local-global stochastic game. Finally, Theorem \ref{Thm:RBENash}, asserts that $(\pi_{\Loc}^*,\pi_{\Glob}^*)$ is a Nash equilibrium of the local-global stochastic game. 
\end{proof}
\section{Numerical Simulations}
\label{Sec:Num}

\begin{figure}[htbp]
	\begin{minipage}[b]{0.5\textwidth}
		\centering
		\scalebox{.6}{
%
%
\definecolor{mycolor1}{rgb}{0.00000,0.44700,0.74100}%
\definecolor{mycolor2}{rgb}{0.85000,0.32500,0.09800}%
\definecolor{mycolor3}{rgb}{0.92900,0.69400,0.12500}%
\definecolor{mycolor4}{rgb}{0.49400,0.18400,0.55600}%
\begin{tikzpicture}

\begin{axis}[%
width=4.521in,
height=4in,
at={(0.5in,0in)},
scale only axis,
xmin=0,
xmax=50,
xlabel style={font=\color{white!15!black}},
xlabel={Time/Round},
ymin=0.5,
ymax=4,
ylabel style={font=\color{white!15!black}},
ylabel={\textbf{Value}},
axis background/.style={fill=white},
xmajorgrids,
ymajorgrids,
legend style={at={(0.97,0.03)}, anchor=south east, legend cell align=left, align=left, draw=white!15!black}
]
\addplot [color=mycolor1, line width=1.5pt]
  table[row sep=crcr]{%
0	0.777734644168963\\
1	1.40873884120953\\
2	1.89784812133326\\
3	2.28541535860117\\
4	2.59651817419071\\
5	2.84277023837577\\
6	3.0405557502618\\
7	3.20147348121109\\
8	3.32956988110993\\
9	3.43124382937642\\
10	3.51281803076393\\
11	3.57775464194448\\
12	3.62953590822565\\
13	3.67124696739961\\
14	3.70480836177951\\
15	3.73151928504455\\
16	3.75274675248544\\
17	3.76977517748263\\
18	3.78339189514728\\
19	3.79428463242944\\
20	3.8029764510938\\
21	3.81002906395284\\
22	3.81562280025984\\
23	3.82008667777972\\
24	3.82369209246282\\
25	3.82656219128833\\
26	3.82885646676115\\
27	3.83068603328007\\
28	3.83216170316826\\
29	3.83333435151022\\
30	3.83426365656322\\
31	3.83501491121403\\
32	3.83561770323221\\
33	3.8360958328756\\
34	3.83648086936035\\
35	3.8367894656205\\
36	3.83703552946098\\
37	3.83723280178615\\
38	3.8373907980967\\
39	3.83751708696817\\
40	3.8376175528979\\
41	3.83769778315716\\
42	3.83776242449184\\
43	3.83781387395302\\
44	3.83785518073019\\
45	3.83788808046815\\
46	3.83791446571999\\
47	3.83793559150637\\
48	3.83795246551923\\
49	3.83796603122046\\
};
\addlegendentry{\LARGE Greedy (LA) and Greedy (GA)}

\addplot [color=mycolor2, dotted, line width=1.5pt]
  table[row sep=crcr]{%
0	0.740443074722927\\
1	1.33101117606344\\
2	1.7917970713819\\
3	2.15691905604883\\
4	2.44939679127435\\
5	2.68342780157329\\
6	2.87258788474357\\
7	3.02464628062819\\
8	3.14500724023186\\
9	3.24193551040314\\
10	3.31922958708393\\
11	3.38118963050209\\
12	3.43014459826721\\
13	3.46963318994788\\
14	3.50129336468901\\
15	3.52673265496079\\
16	3.54692885626277\\
17	3.56297054298881\\
18	3.57583267389127\\
19	3.58618953045674\\
20	3.59451315715714\\
21	3.60113635018412\\
22	3.60639412908431\\
23	3.61063009511537\\
24	3.61401343069723\\
25	3.61670936586427\\
26	3.61886772848186\\
27	3.6206166728446\\
28	3.62199453171728\\
29	3.62309311529119\\
30	3.62397333712297\\
31	3.62467513987781\\
32	3.62524103291341\\
33	3.62569650582045\\
34	3.62606150697834\\
35	3.62635410028833\\
36	3.62658704309959\\
37	3.62677206533984\\
38	3.62691988148818\\
39	3.62703978964915\\
40	3.62713509027813\\
41	3.6272110987687\\
42	3.6272728591388\\
43	3.62732178624506\\
44	3.62736084281378\\
45	3.62739229047914\\
46	3.6274171996439\\
47	3.62743719382572\\
48	3.62745318143504\\
49	3.62746596299538\\
};
\addlegendentry{\LARGE Boltzman (LA) and Greedy (GA)}

\addplot [color=mycolor3, dashed, line width=1.5pt]
  table[row sep=crcr]{%
0	0.634129610364244\\
1	1.13027727350143\\
2	1.52356545095917\\
3	1.83694745407074\\
4	2.08641742699324\\
5	2.28684645435438\\
6	2.44364392662313\\
7	2.57132319990197\\
8	2.67275504317437\\
9	2.75430945385379\\
10	2.8193725994559\\
11	2.8715089391413\\
12	2.91372797873319\\
13	2.94695130781614\\
14	2.97364057884114\\
15	2.99538978803767\\
16	3.01272346644365\\
17	3.02648280304854\\
18	3.03717297590764\\
19	3.04600964932118\\
20	3.05307796741321\\
21	3.05880880065542\\
22	3.06332507044554\\
23	3.06694883023963\\
24	3.0698456832625\\
25	3.07215944428554\\
26	3.07404300589803\\
27	3.07551731609392\\
28	3.07671779524157\\
29	3.07765814918063\\
30	3.0784109932844\\
31	3.07900825304339\\
32	3.07949415862622\\
33	3.0798873819947\\
34	3.08019696987074\\
35	3.0804434638177\\
36	3.08064513080786\\
37	3.0808058051674\\
38	3.080931738867\\
39	3.08103249111661\\
40	3.08111143195317\\
41	3.08117593988769\\
42	3.08122893486673\\
43	3.08127132916005\\
44	3.08130483406475\\
45	3.08133168800885\\
46	3.08135287576118\\
47	3.08136991900261\\
48	3.08138365267772\\
49	3.08139455078073\\
};
\addlegendentry{\LARGE Greedy (LA) and Boltzman (GA)}

\addplot [color=mycolor4, dashdotted, line width=1.5pt]
  table[row sep=crcr]{%
0	0.544386381914012\\
1	0.993709768113072\\
2	1.33912470841565\\
3	1.61658092803017\\
4	1.83935502339686\\
5	2.01684684206761\\
6	2.15781424137111\\
7	2.27095914898843\\
8	2.35820906439647\\
9	2.42957053421043\\
10	2.48873809204665\\
11	2.53557925313193\\
12	2.57272196714973\\
13	2.60261148964912\\
14	2.62639767321241\\
15	2.6456133262544\\
16	2.66056895216612\\
17	2.67258429045744\\
18	2.6823051535591\\
19	2.69008821573655\\
20	2.69639053079258\\
21	2.70139495534065\\
22	2.70541373061669\\
23	2.70865140397956\\
24	2.71118518779038\\
25	2.71323087848653\\
26	2.71486817441161\\
27	2.71619332650255\\
28	2.71723438823561\\
29	2.71806284613564\\
30	2.71873918475662\\
31	2.71928798896305\\
32	2.71971213587943\\
33	2.72005019475933\\
34	2.72032410020511\\
35	2.72054544704902\\
36	2.7207202673538\\
37	2.72086006767757\\
38	2.72097242159863\\
39	2.72106229064028\\
40	2.72113449844197\\
41	2.72119262759156\\
42	2.72123927302208\\
43	2.72127627862975\\
44	2.7213055235743\\
45	2.72132921565375\\
46	2.72134804306538\\
47	2.72136326182115\\
48	2.72137533902328\\
49	2.7213849977476\\
};
\addlegendentry{\LARGE Boltzman (LA) and Boltzman (GA)}

\end{axis}

\begin{axis}[%
width=4in,
height=4in,
at={(0in,0in)},
scale only axis,
xmin=0,
xmax=1,
ymin=0,
ymax=1,
axis line style={draw=none},
ticks=none,
axis x line*=bottom,
axis y line*=left,
legend style={legend cell align=left, align=left, draw=white!15!black}
]
\end{axis}
\end{tikzpicture}%
}
		\centerline{(a) $\tau = 1.3$}
	\end{minipage}
	\begin{minipage}[b]{0.5\textwidth}
		\centering
		\scalebox{.6}{
%
%
\definecolor{mycolor1}{rgb}{0.00000,0.44700,0.74100}%
\definecolor{mycolor2}{rgb}{0.85000,0.32500,0.09800}%
\definecolor{mycolor3}{rgb}{0.92900,0.69400,0.12500}%
\definecolor{mycolor4}{rgb}{0.49400,0.18400,0.55600}%
\begin{tikzpicture}

\begin{axis}[%
width=4.521in,
height=4in,
at={(0.5in,0in)},
scale only axis,
xmin=0,
xmax=50,
xlabel style={font=\color{white!15!black}},
xlabel={Time/Round},
ymin=0.5,
ymax=4,
axis background/.style={fill=white},
xmajorgrids,
ymajorgrids,
legend style={at={(0.97,0.03)}, anchor=south east, legend cell align=left, align=left, draw=white!15!black}
]
\addplot [color=mycolor1, line width=1.5pt]
  table[row sep=crcr]{%
0	0.725500377838466\\
1	1.31237997960411\\
2	1.77378259645308\\
3	2.14171206145536\\
4	2.43652073052064\\
5	2.6719616511734\\
6	2.86080998679473\\
7	3.01158577020201\\
8	3.13245502945674\\
9	3.22940697165294\\
10	3.30615874511862\\
11	3.36775588397569\\
12	3.41705404102403\\
13	3.45635829738038\\
14	3.48792095986857\\
15	3.51340821409494\\
16	3.53377877796383\\
17	3.54995005626664\\
18	3.56286604685799\\
19	3.57322288388398\\
20	3.5814378601706\\
21	3.58805628343923\\
22	3.59339819408025\\
23	3.59764574544854\\
24	3.60104363532128\\
25	3.60376370446955\\
26	3.6059197722837\\
27	3.60763572545057\\
28	3.60902216872343\\
29	3.61013361657204\\
30	3.61102363506231\\
31	3.61173379800864\\
32	3.61230734548773\\
33	3.6127667838332\\
34	3.61312966351204\\
35	3.61341905672359\\
36	3.61365210756361\\
37	3.61383864063219\\
38	3.61398842283458\\
39	3.6141076910515\\
40	3.61420301751708\\
41	3.61427994789566\\
42	3.61434128186641\\
43	3.61439024177861\\
44	3.61442939545705\\
45	3.6144607723684\\
46	3.61448571222208\\
47	3.61450572924696\\
48	3.61452179645158\\
49	3.61453454862416\\
};

\addplot [color=mycolor2, dotted, line width=1.5pt]
  table[row sep=crcr]{%
0	0.725001389447786\\
1	1.30575426470531\\
2	1.76639641774001\\
3	2.13770899250484\\
4	2.43299355088032\\
5	2.66885406737265\\
6	2.8576012162313\\
7	3.00800933848142\\
8	3.12864734092497\\
9	3.22485900305454\\
10	3.30188211784533\\
11	3.36347812505332\\
12	3.41294903523047\\
13	3.45239573925421\\
14	3.48409093261487\\
15	3.50917006073536\\
16	3.52933046492038\\
17	3.54560791899496\\
18	3.55855742469589\\
19	3.56890047524927\\
20	3.57717808541459\\
21	3.58381142799649\\
22	3.58911501762363\\
23	3.5933512087061\\
24	3.59675667505416\\
25	3.59948183066537\\
26	3.6016492239169\\
27	3.60338921672848\\
28	3.60477859945526\\
29	3.60588412918918\\
30	3.60677307068963\\
31	3.6074819207248\\
32	3.60805321878686\\
33	3.60850784241316\\
34	3.60887231187441\\
35	3.60916329142872\\
36	3.60939601962368\\
37	3.60958129923468\\
38	3.60973063911201\\
39	3.60985015137374\\
40	3.60994583978116\\
41	3.61002205750432\\
42	3.61008290673191\\
43	3.6101315247265\\
44	3.6101706722531\\
45	3.61020213175713\\
46	3.61022720927113\\
47	3.61024734939048\\
48	3.61026336061712\\
49	3.61027616966534\\
};

\addplot [color=mycolor3, dashed, line width=1.5pt]
  table[row sep=crcr]{%
0	0.724632316664036\\
1	1.30697284246556\\
2	1.77054117862148\\
3	2.13790758778067\\
4	2.43165572719923\\
5	2.66406636453265\\
6	2.85202468871842\\
7	3.00262197574038\\
8	3.12325839706965\\
9	3.22015574358969\\
10	3.29751329470929\\
11	3.35926571946604\\
12	3.40842334527326\\
13	3.44788982118834\\
14	3.47972510432446\\
15	3.50499749731557\\
16	3.52507246045875\\
17	3.54141474861738\\
18	3.55439169766938\\
19	3.56468994360751\\
20	3.57297982023517\\
21	3.57961321515252\\
22	3.58489410110085\\
23	3.58912274995883\\
24	3.59252755829337\\
25	3.59523116554743\\
26	3.59739932142353\\
27	3.59912611059994\\
28	3.60051974393431\\
29	3.60163042543855\\
30	3.60252242767424\\
31	3.60323806777844\\
32	3.60380993386337\\
33	3.60426722834937\\
34	3.60463385760096\\
35	3.60492649263461\\
36	3.60515919027837\\
37	3.60534419025299\\
38	3.60549327434566\\
39	3.60561381346804\\
40	3.60570993038691\\
41	3.60578616594027\\
42	3.60584714812762\\
43	3.60589599703075\\
44	3.60593510385578\\
45	3.60596637028753\\
46	3.60599148749828\\
47	3.60601155491631\\
48	3.60602758865167\\
49	3.60604031505337\\
};

\addplot [color=mycolor4, dashdotted, line width=1.5pt]
  table[row sep=crcr]{%
0	0.728353241295573\\
1	1.31096359862208\\
2	1.76854345832814\\
3	2.13663482857412\\
4	2.43067991629376\\
5	2.66485158860468\\
6	2.85236961793345\\
7	3.00229042952994\\
8	3.12276667803322\\
9	3.21895660520066\\
10	3.29607252470211\\
11	3.35801597765369\\
12	3.40776056324083\\
13	3.44723532332872\\
14	3.47889173642232\\
15	3.50416757372163\\
16	3.52432903897449\\
17	3.54041275520972\\
18	3.55328789640484\\
19	3.56361653686545\\
20	3.57189560162172\\
21	3.57854731994791\\
22	3.58385675915827\\
23	3.58809745223881\\
24	3.59147334199926\\
25	3.59419997730201\\
26	3.59635763284876\\
27	3.59810175522619\\
28	3.59948528099621\\
29	3.60059654077447\\
30	3.60149016081599\\
31	3.60220283636178\\
32	3.60277036696647\\
33	3.60322583920601\\
34	3.60359177325842\\
35	3.60388340535035\\
36	3.60411642869748\\
37	3.60430284982184\\
38	3.60445180681125\\
39	3.60457097115665\\
40	3.60466611740552\\
41	3.60474271788116\\
42	3.604803741078\\
43	3.60485272854255\\
44	3.60489177918432\\
45	3.60492307719303\\
46	3.60494817025803\\
47	3.60496812990943\\
48	3.60498412271203\\
49	3.60499699270189\\
};

\end{axis}

\begin{axis}[%
width=5.833in,
height=4in,
at={(0in,0in)},
scale only axis,
xmin=0,
xmax=1,
ymin=0,
ymax=1,
axis line style={draw=none},
ticks=none,
axis x line*=bottom,
axis y line*=left,
legend style={legend cell align=left, align=left, draw=white!15!black}
]
\end{axis}
\end{tikzpicture}%
}
		\centerline{(b) $\tau = 0.1$}
	\end{minipage}
	\caption{Cum. disc. reward for different policies and temperatures.}
	\label{Fig:Temp}
\end{figure}
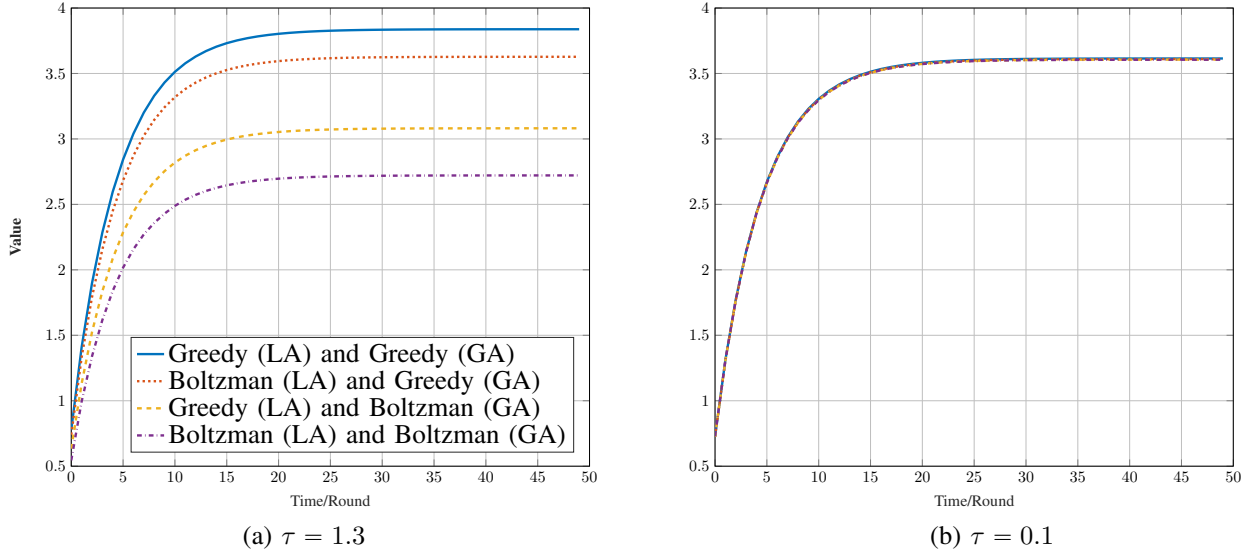
For our numerical analysis we first consider a practical example, where the aim is to maximize the capacity of a network, while exhausting the previously set power constraints at each state. We consider the state space $\mathcal{S} = \{1,2,3\}$ and set $\A_{\Loc} = \{1,2,3\}$, $\mathcal{A}_{\Glob} = \{1,2,3,4\}$, which can be thought of as allocated signal power levels.    
We set the reward functions of both agents equal to $\rmr(s,a_{\Loc},a_{\Glob}) = \log_2(\text{det}(\mathbf{I}_s - \rho(s,a_{\Loc},a_{\Glob})\mathrm{h}(s)\mathrm{h}(s)^T)) - c\abs{\mathrm{p}(s) - (a_{\Loc} + a_{\Glob})}$, composed by the capacity term (where $\mathrm{h}(s)$ denotes the randomly generated state-dependent gain) and by the scaled (with factor $c>0$) penalization of over/under-use of the power respective to the given state-dependent power constraint $\mathrm{p}(s) = [2,5,3]$. To construct the state-transition model, we calculate the Signal to Noise Ratio (SNR) at each state by $\rho(s,a_{\Loc},a_{\Loc}) = f(s,a_{\Loc},a_{\Glob})/N(s)$,
where $f$ specifies the signal power and $N(s) = [2.2,9,4.5]$ the state-dependent noise power.
We then model the transition probabilities by $P(s'|s,a_1,a_2)=\textnormal{erfc}(\sqrt{\rho(s,a_1,a_2)/2})$ if $s = s'$, and $P(s'|s,a_1,a_2) = 1-\textnormal{erfc}(\sqrt{\rho(s,a_1,a_2)/2})/2$ otherwise. In the previous equations, $\textnormal{erfc}$ denotes the Gauss complementary error function. In the Q-learning phase, we choose the Boltzmann strategy as the training policy.

\begin{figure}
\centering
\begin{tabular}{||c | c | c | c | c | c ||} 
 \hline
 $\tau$ & $\beta$ & $c$ & $T_{\text{training}}$ & $T_{\text{testing}}$ & Samples\\ [0.5ex] 
 \hline\hline
 1.30 & 0.80 & 0.25 & 40000 & 5000 & 1000\\ 
 \hline
\end{tabular}
\caption{Simulation Parameters}
\end{figure}

 \begin{figure}[htbp]
	\centering
	\scalebox{.7}{
%
%
\definecolor{mycolor1}{rgb}{0.00000,0.44700,0.74100}%
\definecolor{mycolor2}{rgb}{0.85000,0.32500,0.09800}%
\begin{tikzpicture}

\begin{axis}[%
width=5in,
height=2in,
at={(0in,0in)},
scale only axis,
xmin=0,
xmax=50,
xlabel style={font=\color{white!15!black}},
xlabel={Time/Round},
ymin=2,
ymax=4,
ylabel style={font=\color{white!15!black}},
ylabel={Value},
axis background/.style={fill=white},
xmajorgrids,
ymajorgrids,
legend style={at={(0.97,0.03)}, anchor=south east, legend cell align=left, align=left, draw=white!15!black}
]
\addplot [color=mycolor1, dashdotted, line width=1.5pt]
  table[row sep=crcr]{%
0	0.675413874117271\\
1	1.20932048934214\\
2	1.63542121506164\\
3	1.97511642476719\\
4	2.24614641120982\\
5	2.46252608361275\\
6	2.63474057388346\\
7	2.77398714301778\\
8	2.88520914467661\\
9	2.97395505903193\\
10	3.04542403742941\\
11	3.10281422895587\\
12	3.14824732318673\\
13	3.18443630876749\\
14	3.21356126793966\\
15	3.23680946474379\\
16	3.2554094754387\\
17	3.27037102870509\\
18	3.28234724035648\\
19	3.29191004068789\\
20	3.29956585991034\\
21	3.30566780225074\\
22	3.31051918753554\\
23	3.31443201812134\\
24	3.31756772282018\\
25	3.32007198537083\\
26	3.32207298250962\\
27	3.32366685374583\\
28	3.32495286094632\\
29	3.32597889476478\\
30	3.32679687407297\\
31	3.3274507831719\\
32	3.32797835398798\\
33	3.32840021329682\\
34	3.3287360912034\\
35	3.32900586792402\\
36	3.32922115302876\\
37	3.3293933815845\\
38	3.3295312812586\\
39	3.32964125132828\\
40	3.32972936313044\\
41	3.32979974469245\\
42	3.32985583502121\\
43	3.32990098437458\\
44	3.329936840625\\
45	3.3299656478361\\
46	3.32998862500239\\
47	3.33000706675276\\
48	3.33002185360889\\
49	3.33003362075558\\
};
\addlegendentry{Non-Cooperative (NC)}

\addplot [color=mycolor2, dashed, line width=1.5pt]
  table[row sep=crcr]{%
0	0.794742180424895\\
1	1.42551050406822\\
2	1.91387391213191\\
3	2.30385746953391\\
4	2.61723062280284\\
5	2.86676436761106\\
6	3.06884039094711\\
7	3.2298889983458\\
8	3.3574678320327\\
9	3.46060900001259\\
10	3.54188418176333\\
11	3.60679567588168\\
12	3.65926456594476\\
13	3.70148544747098\\
14	3.73520658778228\\
15	3.76236309591588\\
16	3.78394789989256\\
17	3.80121960599088\\
18	3.81523004230255\\
19	3.82638726963571\\
20	3.83517751738077\\
21	3.8423058724079\\
22	3.84802513095864\\
23	3.85257786440569\\
24	3.85619105088399\\
25	3.85912665997747\\
26	3.86144190294934\\
27	3.86328284865394\\
28	3.86476292517413\\
29	3.8659567410332\\
30	3.86691705085837\\
31	3.86768240181866\\
32	3.86829645644952\\
33	3.86878384838744\\
34	3.86917079700069\\
35	3.86948493880746\\
36	3.86973078035059\\
37	3.8699277076091\\
38	3.87008589226414\\
39	3.87021239240139\\
40	3.87031395020771\\
41	3.87039612945345\\
42	3.87046149216564\\
43	3.87051426220693\\
44	3.87055647515327\\
45	3.87059003591169\\
46	3.87061701851194\\
47	3.87063811709328\\
48	3.87065510453717\\
49	3.87066886419879\\
};
\addlegendentry{Jointly Cooperative (JC)}

\addplot [color=green, line width=1.5pt]
  table[row sep=crcr]{%
0	0.777734644168963\\
1	1.40873884120953\\
2	1.89784812133326\\
3	2.28541535860117\\
4	2.59651817419071\\
5	2.84277023837577\\
6	3.0405557502618\\
7	3.20147348121109\\
8	3.32956988110993\\
9	3.43124382937642\\
10	3.51281803076393\\
11	3.57775464194448\\
12	3.62953590822565\\
13	3.67124696739961\\
14	3.70480836177951\\
15	3.73151928504455\\
16	3.75274675248544\\
17	3.76977517748263\\
18	3.78339189514728\\
19	3.79428463242944\\
20	3.8029764510938\\
21	3.81002906395284\\
22	3.81562280025984\\
23	3.82008667777972\\
24	3.82369209246282\\
25	3.82656219128833\\
26	3.82885646676115\\
27	3.83068603328007\\
28	3.83216170316826\\
29	3.83333435151022\\
30	3.83426365656322\\
31	3.83501491121403\\
32	3.83561770323221\\
33	3.8360958328756\\
34	3.83648086936035\\
35	3.8367894656205\\
36	3.83703552946098\\
37	3.83723280178615\\
38	3.8373907980967\\
39	3.83751708696817\\
40	3.8376175528979\\
41	3.83769778315716\\
42	3.83776242449184\\
43	3.83781387395302\\
44	3.83785518073019\\
45	3.83788808046815\\
46	3.83791446571999\\
47	3.83793559150637\\
48	3.83795246551923\\
49	3.83796603122046\\
};
\addlegendentry{Asymmetrical (AS)}

\end{axis}

\begin{axis}[%
width=4.44in,
height=2in,
at={(0in,0in)},
scale only axis,
xmin=0,
xmax=1,
ymin=0,
ymax=1,
axis line style={draw=none},
ticks=none,
axis x line*=bottom,
axis y line*=left,
legend style={legend cell align=left, align=left, draw=white!15!black}
]
\end{axis}
\end{tikzpicture}%
}
	\caption{Joint. coop. vs. asymm. (this work) vs. fully non-cooperative}
	\label{Strategy_Comp}
\end{figure}
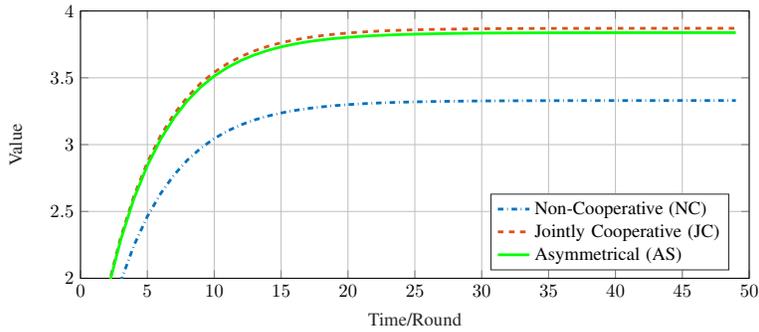 

\begin{figure}[htbp]
\centering
 \begin{tabular}{||c | c || c | c || c ||} 
 \hline
$s$ & $\mathrm{p}(s)$ & NC & AS & JC \\ [0.5ex] 
 \hline\hline
 1 & 2 & (\textbf{3},\textbf{1}) & (\textbf{2},\textbf{1}) & (3,1)\\ 
 \hline
 2 & 5 & (3,\textbf{4})  & (3,\textbf{2}) & (3,2)\\
 \hline
 3 & 3 & (2,\textbf{2})  & (2,\textbf{1}) & (2,1)\\
 \hline
\end{tabular}
\caption{Greedy Strategy Profile}
\label{Fig:TabProf}
\end{figure}

For Boltzmann temperature $\tau=1.3$, figure \ref{Fig:Temp} (a) compares the cumulative discounted reward over time for both (local and global) agents different strategy choices, i.e., the post-learning greedy strategies ($\pi_{\Loc}^{\LAQGI}$ and $\pi_{\Glob}^{\GAQL}$) and the long-term Boltzmann learning strategy ($\eta_{\infty}^{\Loc}$ and $\eta_{\infty}^{\Glob}$). We observe, that if GA applies the Boltzmann strategy, it is better for LA to apply the greedy strategy, and that if LA applies the greedy strategy, it is also better for LA to apply the greedy strategy. This observation supports in particular the claims in Lemmas \ref{Lem:ajajsshssgsgssfsfss} and \ref{Lem:ahhasgsgsfsfsfsffssss}. Moreover, we see that best overall performance yields if both agents acts greedily. This observation is not surprising, since it follows from the fact that the agents' rewards (and therefore the value function) are the same and from our analysis (Lemmas \ref{Lem:ajajsshssgsgssfsfss} and \ref{Lem:ahhasgsgsfsfsfsffssss}).  


In Figure \ref{Fig:Temp} (b), we compare the same policy tuples, however with smaller $\tau=1.3$.
We observe that the cumulative discounted rewards are the approximately the same for any strategy choice, which is the effect of the fact that the Boltzmann strategy morphs into a greedy like strategy (c.f. the discussion above the Theorem \ref{Thm:ajajhshssggsgssgsgffsfssssss}). With increasing $\tau$, we observe in our simulation that the discrepancy between the strategy tuples' performances becomes larger. These observations gives in particular insight into the Theorem \ref{Thm:ajajhshssggsgssgsgffsfssssss}. Moreover, we observe that to small $\tau$ results in a lack of state-action exploration, giving a sub-optimal solution. One can see the latter effect in Figure \ref{Fig:Temp}, which shows that the best possible value in case $\tau=0.1$ is dominated by the best possible value in case $\tau=1.3$.

In Figure \ref{Strategy_Comp} we compare the performance of our asymmetrical Q-learning (Asymmetrical (AS)) with the jointly cooperative Q-learning (Jointly Cooperative (JC)), i.e., the single-agent Q-learning in the MDP $(\S,\A_{\Loc}\times\A_{\Glob},\rmr,\rmP)$), and (fully) non-cooperative Q-learning (Non-Cooperative (NC)), i.e. the Q-learning where GA has no knowledge about LA action. In particular, we compare the corresponding post-learning greedy policies. We observe, that the JC has the best performance, which is to be expected due to the knowledge of the agents. However, it is remarkable to see that AS greatly outperforms the non-cooperative case, and its performance is only marginally worse than the jointly cooperative one. This leads to the belief, that even under asymmetry of information, the agents are able to approach an almost fully cooperative amount of reward, as well as outperform NC case. We can further investigate the reasoning behind this result by analyzing the different strategy profiles of the agents for each case. Figure \ref{Fig:TabProf} shows the greedy strategy profile of the LA and GA as tuples $(a_{\Loc},a_{\Glob})$. As highlighted in the table, of particular interest is the change of behaviour from the GA, when given additional knowledge of the LAs action (second entry of the tuple). In the non-cooperative case, both agents act in a selfish and greedy manner, therefore violating the power constraints and decreasing prosperity as measured by the value function. Given additional knowledge of the LAs action, i.e. for AS, we observe a more conservative and sophisticated usage of power levels from the GA, resulting in a considerable increase in prosperity. Moreover, we see that the GA for AS additionally influences the LA to be more conservative with his power usage as seen for $s = 1$, where the LA now chooses $a_{\Loc} = 2$ instead of $a_{\Loc} = 3$. Here we observe the limits of AS, as in this particular case, when we compare the AS to the optimal JC case, it is indeed more advantageous to violate the power constraints to maximize reward, since capacity maximization appears to yield a greater reward overall. Therefore in our example the AS seems to incentivize a more conservative approach, which, while being an improvement to the uncoordinated selfish approach (NC), might yield a suboptimal solution overall.

In our second numerical example we compare the performance of the Extra Information GA Q-Learning (EIGAQL) algorithm with the standard GAQL algorithm. As before, the performance is measured using the value function. To better illustrate the difference in performance between both algorithms, we consider a slightly bigger action and state space, i.e. $\mathcal{S} = \{1,2,...,7\}$, $\mathcal{A}_{\Loc} = \{1,2,...,4\}$, $\mathcal{A}_{\Glob} = \{1,2,...,5\}$. Since we have stated, that EIGAQL applies to arbitrary stationary LA strategies, the LA strategy is randomly generated beforehand by generating a normalized random matrix $M \in \mathbb{R}^{\mathcal{S} \times \mathcal{A}_{\Loc}}$ of uniformly distributed entries between 0 and 1 whose row-entries sum up to 1. Furthermore, the state and action dependent reward is also randomly generated and sampled from the standard normal distribution $\mathcal{N}(0,1)$. The simulation parameters are shown in Figure \ref{Fig:SimPar}. We note, that the increases in $\tau$ and $T_{\text{training}}$ are due to the higher dimensionality of our problem, which consequently demands a higher exploration factor to ensure that all states have been visited enough times.

\begin{figure}[h!]
\label{Fig:SimPar}
\centering
\begin{tabular}{||c | c | c | c | c ||} 
 \hline
 $\tau$ & $\beta$  & $T_{\text{training}}$ & $T_{\text{testing}}$ & Samples\\ [0.5ex] 
 \hline\hline
 1.60 & 0.80  & 300000 & 5000 & 100\\ 
 \hline
\end{tabular}
\end{figure}
\begin{figure}[htbp]
	\centering
	\scalebox{.7}{
%
%
\definecolor{mycolor1}{rgb}{0.00000,0.44700,0.74100}%
\definecolor{mycolor2}{rgb}{0.85000,0.32500,0.09800}%
\begin{tikzpicture}

\begin{axis}[%
width=4.602in,
height=3.566in,
at={(0.772in,0.481in)},
scale only axis,
xmin=0,
xmax=50,
xlabel style={font=\color{white!15!black}},
xlabel={Time/Round},
ymin=2,
ymax=4.5,
ylabel style={font=\color{white!15!black}},
ylabel={Value},
axis background/.style={fill=white},
xmajorgrids,
ymajorgrids,
legend style={at={(0.97,0.03)}, anchor=south east, legend cell align=left, align=left, draw=white!15!black}
]
\addplot [color=mycolor1, line width=1.5pt]
  table[row sep=crcr]{%
0	0.824791284720935\\
1	1.5228393116057\\
2	2.11412194137072\\
3	2.57623962772494\\
4	2.95156175052115\\
5	3.29460738020194\\
6	3.5193040113649\\
7	3.71837194694747\\
8	3.84876674252175\\
9	3.96712232254347\\
10	4.0584692797689\\
11	4.13153404911738\\
12	4.19488236959463\\
13	4.23891472257904\\
14	4.2755824754179\\
15	4.30679581453769\\
16	4.3332040303942\\
17	4.35475595514474\\
18	4.3697851524144\\
19	4.38265653363152\\
20	4.39258236030443\\
21	4.40077988898044\\
22	4.40724990930538\\
23	4.41247203660703\\
24	4.41675808243143\\
25	4.42014261234568\\
26	4.42275164255861\\
27	4.4246659475166\\
28	4.42635563352845\\
29	4.42764528015847\\
30	4.42877301231891\\
31	4.4296928536102\\
32	4.4304251550785\\
33	4.4310243858108\\
34	4.43146294999822\\
35	4.43182639970501\\
36	4.43209568748437\\
37	4.43233390207205\\
38	4.43251984943328\\
39	4.43268525393229\\
40	4.43280557689579\\
41	4.43289649539361\\
42	4.43297882364545\\
43	4.43304183394045\\
44	4.4330820410472\\
45	4.43311545156326\\
46	4.43314799637929\\
47	4.43317177643923\\
48	4.43319219496211\\
49	4.43320799435633\\
};
\addlegendentry{\LARGE Stationary (LA) and EIGAQL (GA)}

\addplot [color=mycolor2, dashdotted, line width=1.5pt]
  table[row sep=crcr]{%
0	0.789122724923229\\
1	1.41985198372179\\
2	1.95857434592784\\
3	2.38562405298032\\
4	2.74235991292938\\
5	3.04461817941435\\
6	3.27972976669374\\
7	3.46242630534197\\
8	3.59963975961357\\
9	3.71508980458794\\
10	3.80786568265967\\
11	3.87624166659817\\
12	3.93682593823462\\
13	3.99292341018803\\
14	4.0281622140241\\
15	4.06099994900823\\
16	4.08424923439088\\
17	4.10356846385192\\
18	4.11894998365622\\
19	4.13108713965015\\
20	4.14091819829629\\
21	4.1497448944739\\
22	4.15585194334633\\
23	4.16072899977328\\
24	4.16467159480623\\
25	4.16770296460047\\
26	4.17049641935931\\
27	4.17259732139998\\
28	4.17412292952788\\
29	4.17541969834857\\
30	4.17659442226682\\
31	4.17750155194959\\
32	4.17824811334383\\
33	4.17879954720002\\
34	4.17925542601246\\
35	4.17963003751131\\
36	4.17992173395996\\
37	4.18016005498529\\
38	4.1803419786115\\
39	4.18048421502709\\
40	4.18059859771548\\
41	4.1807058702187\\
42	4.18078962316214\\
43	4.1808541396109\\
44	4.18090211992603\\
45	4.18093625142008\\
46	4.18096511408144\\
47	4.18099025079667\\
48	4.18100686656183\\
49	4.1810238287348\\
};
\addlegendentry{\LARGE Stationary (LA) and GAQL (GA)}

\end{axis}

\begin{axis}[%
width=5.938in,
height=4.375in,
at={(0in,0in)},
scale only axis,
xmin=0,
xmax=1,
ymin=0,
ymax=1,
axis line style={draw=none},
ticks=none,
axis x line*=bottom,
axis y line*=left,
legend style={legend cell align=left, align=left, draw=white!15!black}
]
\end{axis}
\end{tikzpicture}%
}
	\caption{Standard GA learning vs predictive GA learning}
	\label{Predictive}
\end{figure}
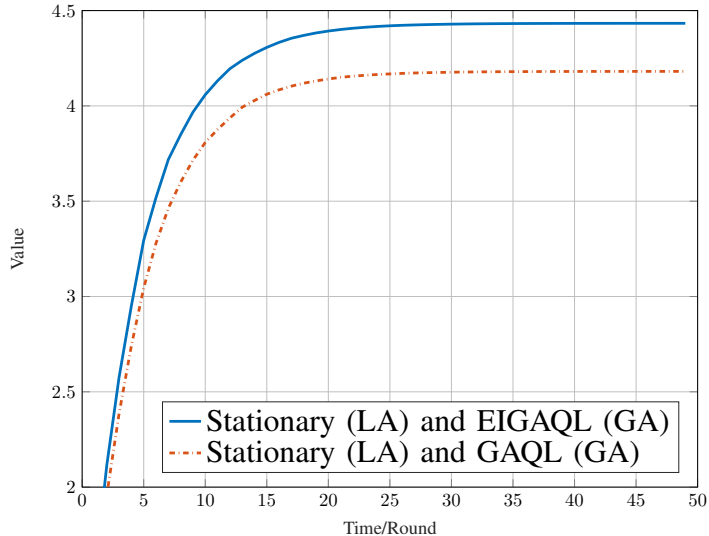 

As we can see in Figure \ref{Predictive}, the additional information provided in EIGAQL yields an improvement in performance compared to the usual GAQL. These changes are due to differences in the greedy strategy profile of the GA. Figure \ref{Tab:Compare} shows the tuple $(a_{\text{GAQL}},a_{\text{EIGAQL}})$ of greedy actions by the GA. Even though most of them are the same, those who differ, do so drastically. This suggests that additional information does provide some insight to the GA, which highly influences his behaviour.

\begin{figure}[h!]
\label{Tab:Compare}
\centering
\begin{tabular}{||c || c | c | c | c ||} 
 \hline
  &  $a_{\Loc} = 1$  &  $a_{\Loc} = 2$ &  $a_{\Loc} = 3$ &  $a_{\Loc} = 4$\\ [0.5ex] 
 \hline\hline
 $s = 1$ & (3,3)  & (4,4) & \textbf{(1,3)} & (3,3)\\ 
 \hline
 $s = 2$ & (4,4)  & (5,5) & (1,1) & (3,3)\\ 
 \hline
 $s = 3$ & (5,5)  & (4,4) & (4,4) & (5,5)\\ 
 \hline
 $s = 4$ & (5,5)  & (1,1) & (2,2) & (5,5)\\ 
 \hline
 $s = 5$ & (2,2)  & (5,5) & (4,4) & (3,3)\\ 
 \hline
 $s = 6$ & \textbf{(5,2)}  & (2,2) & (5,5) & (1,1)\\ 
 \hline
 $s = 7$ & \textbf{(1,2)}  & (5,5) & (4,4) & (1,1)\\ 
 \hline
\end{tabular}
\end{figure}

We note, that due to the inherent random nature of the example, this simulation has been performed multiple times using the same parameters and stationary strategy of the LA. The results remained generally the same, i.e. EIGAQL outperforms GAQL, though the margin of improvement can range from minimal ($>0.1\%$) to considerable (up to $10\%$). For illustration purposes we chose to show an example where a bigger improvement is noticeable. In summary, it certainly does provide an advantage for the GA to have additional information of the LAs actions.

\section{Conclusion and Future Work}
We have studied the long-term outcome of multi-agent (independent) Q-learning with information asymmetry. We have shown that the latter can foster the stability of the learning method. Despite of the information asymmetry, we have shown that the post-learning joint strategy of the agents is an almost solution concept. For sake of completeness, we have also provided the existence theorem for the indeed Nash equilibrium of the underlying game. Furthermore, as the proposed algorithm (GAQL) for GA is only optimal in case that LA applies the greedy strategy, we have provided also in this work a way for GA to gain optimality in case that LA applies a general stationary strategy. Requirement for this achievement is that GA can observe subsequent LA's action.
A point worth for further discussion is the summability-condition $\psi^{\Loc}$ and $\psi^{\Glob}$ given in the corresponding convergence theorems. 
The achievement of this depends not only on the model's transition probability and the considered agent itself, but also on other extrinsic factor: One agent's policy has to allow other's to explore the MDP. We leave the detailed treatment of this aspect for the future.   
%

\section{Appendix}

\subsection{Missing proofs in Section \ref{Sec:ConvThm}}
Our convergence proof is based on the following well-known statement \cite{Puterman1994}:
\begin{proposition}
	\label{Prop:ajjahhssggsgdddddd}
	Given a filtration $\G:=(\G_{t})_{t\in\nat_{0}}$.
	Let be $H:\real^{D}\rightarrow\real^{D}$, $(\gamma_{t})_{t\in\nat_{0}}\subset\real^{D}$, and $(U_{t})_{t\in\nat_{0}},(W_{t})_{t\in\nat}$ are sequences of $\real^{D}$-valued RV. Let 
	$(X_{t})_{t\in\nat_{0}}\subset\real^{D}$ be a sequence generated by the iteration:
	\begin{equation}
	\label{Eq:aahshsggsgsfsfsfsssssssss}
	X_{t+1}(i)=(1-\gamma_{t}(i))X_{t}(i)+\gamma_{t}(i)\left[(HX_{t})(i)+U_{t}(i)+W_{t+1}(i) \right]. 
	\end{equation}
	Suppose that:
	\begin{enumerate}
		\item $(W_{t})_{t\in\nat}$ is $\G$-adapted and fulfills:
		\begin{equation*}
		\Erw[W_{t+1}|\G_{t}]=0\quad\Erw[W_{t+1}^{2}(i)|\G_{t}]\leq A+B\norm{X_{t}}^{2}_{\infty},
		\end{equation*}
		for some $A,B>0$.
		\item $(\gamma_{t})_{t\in\nat}$ is sequence of non-negative $\G$-adapted RVs and fulfills:
		\begin{equation*}
		\sum_{t=0}^{\infty}\gamma_{t}(i)=\infty\quad\text{and}\quad\sum_{t=0}^{\infty}\gamma_{t}^{2}(i)<\infty\quad \text{a.s.}
		\end{equation*}
		\item there exists $x_{*}\in\real^{D}$ and $\beta\in [0,1)$ s.t.:
		\begin{equation*}
		\norm{HX_{t}-x_{*}}_{\infty}\leq \beta\norm{X_{t}-x_{*}}_{\infty}
		\end{equation*}
		\item $(U_{t})$ is $\G$-adapted, and there exists a sequence $(\theta_{t})_{t\in\nat 0}$ of $\real_{\geq 0}$-valued RV converging to $0$ a.s. such that:
		\begin{equation*}
		\abs{U_{t}(i)}\leq \theta_{t}(\norm{X_{t}}_{\infty}+1),
		\end{equation*}      
	\end{enumerate}
	Then:
	\begin{equation*}
	X_{t}\xrightarrow{t\rightarrow\infty} x_{*}\quad\text{a.s.}
	\end{equation*}
\end{proposition}
\subsubsection{Proof of LAQGI convergence (Theorem \ref{Thm:akakashshssggdhdhdggdgdhdhd})}
\label{Subsubsec:LAQGI}
Our strategy is to write the iterate \eqref{Alg:ohgfdfddrrdrdrdrdrdghhgggdsa} in the form \eqref{Eq:aahshsggsgsfsfsfsssssssss}. To achieve this, we first notice that the iterate of \eqref{Alg:ohgfdfddrrdrdrdrdrdghhgggdsa} can be written as:
\begin{equation}
\label{Eq:aoasjsshhssggsggsgsgsshhss}
Q^{\LQ}_{t+1}=(1-\psi^{\Loc}_{t})\odot Q^{\LQ}_{t}+\psi^{\Loc}_{t}\odot\hatrmT_{t}^{\Loc}Q^{\LQ}_{t},
\end{equation}
where $\hatrmT_{t}^{\Loc}$ is the optimal Bellman operator of the discounted MDP $(\S,\A_{\Loc},\rmr^{\Loc},\delta_{S_{t+1}},\beta_{\Loc})$, and where:
\begin{equation*}
\psi^{\Loc}_{t}(s,a)=\mathbf{1}_{\lrbrace{S_{t}=s,A^{\Loc}_{t}=a}}\gamma^{\Loc}_{t}.
\end{equation*}
Next, by means of the optimal Bellman operator $\rmT_{\Loc,t}$ of the discounted MDP $(\S,\A_{\Loc},\tildrmr_{t}^{\Loc},\tildrmP^{\Loc}_{t},\beta_{\Loc})$, where:
\begin{equation*}
\begin{split}
\tildrmr_{t}^{\Loc}(s,a_{\Loc}):=\Erw_{A^{\Glob}_{t}\sim\eta^{\Glob}_{t}(\cdot|\F_{t})}[\rmr^{\Loc}(s,a_{\Loc},A^{\Glob}_{t})]\quad\text{and}\quad \tildrmP^{\Loc}_{t}(\cdot|s,a_{\Loc}):=\Erw_{A^{\Glob}_{t}\sim\eta^{\Glob}_{t}(\cdot|\F_{t})}[\rmP(\cdot|s,a_{\Loc},A^{\Glob}_{t})],
\end{split}
\end{equation*}

 and the optimal Bellman operator $\tildrmT^{\Loc}$ for the discounted MDP $(\S,\A_{\Loc},\tildrmr^{\Loc},\rmP^{\Loc},\beta_{\Loc})$, we can rewrite \eqref{Eq:aoasjsshhssggsggsgsgsshhss} as:
\begin{equation*}
\begin{split}
Q^{\LQ}_{t+1}&=(1-\psi^{\Loc}_{t})\odot Q^{\LQ}_{t}+\psi^{\Loc}_{t}\odot\left[\tildrmT^{\Loc}Q^{\LQ}_{t}+ U_{t}+W_{t+1}\right] ,
\end{split}
\end{equation*}
where:
\begin{equation}
\label{Eq:ajajshssgsgsgsggsgss}
W_{t+1}=\hatrmT_{t}^{\Loc}Q^{\LQ}_{t}-\rmT_{\Loc,t}Q^{\LQ}_{t}
\end{equation}
\begin{equation}
\label{Eq:ajajagsgsfsffsfsddsssss}
U_{t}=\rmT_{\Loc,t}Q^{\LQ}_{t}-\tildrmT^{\Loc}Q^{\LQ}_{t}
\end{equation}
	\begin{lemma}
	\label{Lem:aajsjshhsggdddhhdd}
	The random sequence $(W_{t})_{t\in\nat}$ defined in \eqref{Eq:ajajshssgsgsgsggsgss} satisfies:
	\begin{equation*}
	\begin{split}
	&\Erw[W_{t+1}|\F_{t}]=0,\\
	&\Erw[(W_{t+1}(s,a))^{2}|\F_{t}]\leq 2\norm{\rmr^{\Loc}}_{\infty}+2\gamma_{\Loc}^{2}\norm{Q^{\LQ}_{t}}
	\end{split}
	\end{equation*}
\end{lemma}
\begin{proof}
	We have from \eqref{Eq:ajkajjhssghsgsgsgsssss}:
	\begin{equation*}
	\begin{split}
	&\Erw[\rmr^{\Loc}(s,a_{\Loc},A^{\Glob}_{t})|\F_{t}]=\sum_{a_{\Glob}\in\A_{\Glob}}\eta^{\Glob}_{t}(a_{\Glob}|\F_{t})\rmr^{\Loc}(s,a_{\Loc},a_{\Glob})=\Erw_{A^{\Glob}_{t}\sim\eta^{\Glob}_{t}(\cdot|\F_{t})}[\rmr^{\Loc}(s,a_{\Loc},A^{\Glob}_{t})]=\tildrmr_{t}^{\Loc}(s,a_{\Loc})
	\end{split}
	\end{equation*}
	Furthermore:
	\begin{equation}
	\label{Eq:jajagsgsgsfsfsfsfssssss}
	\begin{split}
	&\Erw\left[\left.  \max_{a'_{\Loc}\in\A_{\Loc}}Q^{\LQ}_{t}(S_{t+1},a'_{\Loc})\right| \F_{t}\right] =\Erw\left[ \left. \Erw\left[\left.  \max_{a'_{\Loc}\in\A_{\Loc}}\rmQ^{\LQ}_{t}(S_{t+1},a'_{\Loc})\right| \tildF_{t}\right]\right|  \F_{t}\right] =\Erw\left[\left.  \Erw_{S'\sim \rmP(\cdot|S_{t},A^{\Loc}_{t},A^{\Glob}_{t})}\left[ \max_{a'_{\Loc}\in\A_{\Loc}}Q^{\LQ}_{t}(S',a'_{\Loc})\right] \right| \F_{t}\right] \\
	&=\sum_{a_{\Glob}\in\A_{\Glob}}\eta^{\Glob}_{t}(a_{\Glob}|\F_{t})\sum_{s'\in\S}\rmP(s'|S_{t},A^{\Loc}_{t},a_{\Glob})\max_{a'_{\Loc}\in\A_{\Loc}}Q^{\LQ}_{t}(s',a'_{\Loc})=\sum_{s'\in\S}\left[ \sum_{a_{\Glob}\in\A_{\Glob}}\eta^{\Glob}_{t}(a_{\Glob}|\F_{t})\rmP(s'|S_{t},A^{\Loc}_{t},a_{\Glob})\right]\max_{a'_{\Loc}\in\A_{\Loc}} Q^{\LQ}_{t}(s',a'_{\Loc})\\
	&=\sum_{s'\in\S}\Erw_{A^{\Glob}_{t}\sim\eta^{\Glob}_{t}(\cdot|\F_{t})}[\rmP(s'|S_{t},A^{\Loc}_{t},A^{\Glob}_{t})]\max_{a'_{\Loc}\in\A_{\Loc}} Q^{\LQ}_{t}(s',a'_{\Loc})=\sum_{s'\in\S}\tildrmP_{t}^{\Loc}(s'|S_{t},A^{\Loc}_{t})\max_{a'_{\Loc}\in\A_{\Loc}} Q^{\LQ}_{t}(s',a'_{\Loc})\\
	&=\Erw_{S'\sim \tildrmP_{t}^{\Loc}(\cdot|S_{t},A^{\Loc}_{t})}\left[\max_{a'_{\Loc}\in\A_{\Loc}}Q^{\LQ}_{t}(S',a'_{\Loc})\right],
	\end{split}
	\end{equation}
	where the first equality follows from $\tilde{\mathcal{F}}_{t}\subset\mathcal{F}_{t}$ and the tower property for conditional expectation, the third inequality from \eqref{Eq:ajkajjhssghsgsgsgsssss}.
	
	Combining both previous computations, we have:

	\begin{equation}
	\label{Eq:jajagsgsgsfsfsfsfssssss2}
	\begin{split}
	&\Erw[(\hatrmT_{t}^{\Loc}Q^{\LQ}_{t})(S_{t},A^{\Loc}_{t})|\F_{t}]=\Erw\left[\left.  \rmr^{\Loc}(S_{t},A^{\Glob}_{t},A^{\Loc}_{t})+\beta_{\Loc}\max_{a_{\Loc}^{'}\in \A_{\Loc}}Q^{\LQ}_{t}(S_{t+1},a^{'}_{\Loc})\right| \F_{t}\right]\\
	&=\Erw\left[\left.\mathbf{1}_{\lrbrace{S_{t}=s,A^{\Loc}_{t}=a_{\Loc}}}\rmr^{\Loc}(s,a_{\Loc},A^{\Glob}_{t})\right| \F_{t}\right]+\beta_{\Loc}\Erw\left[\left.\max_{a_{\Loc}^{'}\in \A_{\Loc}}Q^{\LQ}_{t}(S_{t+1},a^{'}_{\Loc})\right| \F_{t}\right]\\
	&=\mathbf{1}_{\lrbrace{S_{t}=s,A^{\Loc}_{t}=a_{\Loc}}}\Erw\left[\left.\rmr^{\Loc}(s,a_{\Loc},A^{\Glob}_{t})\right| \F_{t}\right]+\beta_{\Loc}\Erw\left[\left.\max_{a_{\Loc}^{'}\in \A_{\Loc}}Q^{\LQ}_{t}(S_{t+1},a_{\Loc}^{'})\right| \F_{t}\right]\\
	&=\mathbf{1}_{\lrbrace{S_{t}=s,A^{\Loc}_{t}=a_{\Loc}}}\tildrmr_{t}^{\Loc}(s,a_{\Loc})+\beta_{\Loc}\Erw_{S'\sim \rmP^{\Loc}_{t}(\cdot|S_{t},A^{\Loc}_{t})}\left[\max_{a'_{\Loc}\in\A_{\Loc}}\rmQ^{\LQ}_{t}(S',a'_{\Loc})\right]\\
		&=\Erw\left[\left.\tildrmr_{t}^{\Loc}(S_{t},A^{\Loc}_{t})\right| \F_{t}\right]+\beta_{\Loc}\Erw_{S'\sim \rmP^{\Loc}_{t}(\cdot|S_{t},A^{\Loc}_{t})}\left[\max_{a'_{\Loc}\in\A_{\Loc}}\rmQ^{\LQ}_{t}(S',a'_{\Loc})\right]=\Erw[(\rmT_{\Loc,t}Q^{\LQ}_{t})(S_{t},A_{t}^{\Loc})|\F_{t}],
	\end{split}
	\end{equation}
	where the third equality follows from the fact that $\mathbf{1}_{\lrbrace{s=S_{t},a_{\Loc}=A^{\Loc}_{t}}}$ is $\mathcal{F}_{t}$-measurable since $(S_{t},A^{\Loc}_{t})$ is $\mathcal{F}_{t}$-measurable, the fourth from \eqref{Eq:jajagsgsgsfsfsfsfssssss} and \eqref{Eq:jajagsgsgsfsfsfsfssssss2}, and the last two equalities follow from the fact that $\eta^{\Glob}_{t}(\cdot|\F_{t})$ is $\F_{t}$-measureable, and thus also $\tildrmr_{t}^{\Loc}$ and  $\rmP^{\Loc}_{t}(\cdot|S_{t},A^{\Loc}_{t})$. Above computation yields the first statement, since:
	\begin{equation*}
	\begin{split}
	&\Erw[W_{t+1}(s,a_{\Loc})|\F_{t}]=\mathbf{1}_{\lrbrace{S_{t}=s,A_{t}^{\Loc}=a_{\Loc}}}\Erw[(\hatrmT_{t}^{\Loc}Q^{\LQ}_{t})(S_{t},A^{\Loc}_{t})-(\rmT_{\Loc,t}Q^{\LQ}_{t})(S_{t},A_{t}^{\Loc})|\F_{t}]=0
	\end{split}
	\end{equation*}
	
	For the second statement, we compute:
	\begin{equation*}
	\begin{split}
	&\Erw[\left( (\hatrmT_{t}^{\Loc}Q^{\LQ}_{t})(s,a_{\Loc})-(\rmT_{\Loc,t}Q^{\LQ}_{t})(s,a_{\Loc})\right)^2 |\F_{t}]\leq\Erw[ ((\hatrmT_{t}^{\Loc}Q^{\LQ}_{t})(s,a_{\Loc}))^2|\tildF_{t}]+\Erw[(\rmT_{\Loc,t}Q^{\LQ}_{t})(s,a_{\Loc})^2 |\F_{t}],
	\end{split}
	\end{equation*}
	where the inequality follows from $(a-b)^{2}\leq a^{2}+b^{2}$ for any $a,b\geq 0$. Now, we estimate each summand above. First, we have:
	\begin{equation*}
	\Erw[(\rmr^{\Loc}(s,a_{\Loc},A^{\Glob}_{t}))^2|\F_{t}]\leq\norm{\rmr^{\Loc}}_{\infty}^{2}\quad\text{and}\quad \Erw[(\max_{a_{\Loc}'}Q^{\LQ}_{t}(S_{t+1},a_{\Loc}'))^{2}|\F_{t}]\leq \norm{Q^{\LQ}_{t}}_{\infty}^{2}.
	\end{equation*} 
	Consequently:
	\begin{equation*}
	\begin{split}
	&\Erw[\left( (\hatrmT_{t}^{\Loc}Q^{\LQ}_{t})(s,a_{\Loc})\right)^2 |\F_{t}]\leq\Erw[(\rmr^{\Loc}(s,a_{\Loc},A^{\Glob}_{t}))^2|\F_{t}]+\beta_{\Loc}^{2}\Erw[(\max_{a_{\Loc}'\in\A_{\Loc}}Q^{\LQ}_{t}(S_{t+1},a_{\Loc}'))^{2}|\F_{t}]\leq\norm{\rmr^{\Loc}}_{\infty}^{2}+\beta_{\Loc}^{2}\norm{Q^{\LQ}_{t}}_{\infty}^{2}.
	\end{split}
	\end{equation*}
	Similar computation yields:
	\begin{equation*}
	\begin{split}
	&\Erw[\left((\rmT_{\Loc,t}Q^{\LQ}_{t})(s,a_{\Loc})\right)^2 |\F_{t}]\leq\norm{\rmr^{\Loc}}_{\infty}^{2}+\beta_{\Loc}^{2}\norm{Q^{\LQ}_{t}}_{\infty}^{2}.
	\end{split}
	\end{equation*}
	Combining both previous estimates, we obtain the desired statement.
	
\end{proof}

\begin{lemma}
	\label{Lem:jajahshgsggsgsgsgsgsgss}
The random sequence $(U_{t})_{\nat_{0}}$ defined in \eqref{Eq:ajajagsgsfsffsfsddsssss} fulfills:
\begin{equation*}
\abs{U_{t}(s,a)}\leq\theta_{t}(1+\norm{Q^{\LQ}_{t}}_{\infty}),
\end{equation*}
where:
\begin{equation*}
\theta_{t}:=\max\lrbrace{\norm{\rmr^{\Loc}}_{\infty},\beta_{\Loc}}\norm{\eta^{\Glob}_{t}(\cdot|\tildF_{t})-\eta^{\Glob}_{\infty}(\cdot|S_{t},A_{t}^{\Loc})}_{1}.
\end{equation*}
\end{lemma}
\begin{proof}
  Setting $\pi^{(1)}_{\Glob}=\eta_{\infty}^{\Glob}(\cdot|S_{t},A^{\Loc}_{t})$, $\pi^{(2)}_{\Glob}=\eta_{t}^{\Glob}(\cdot|\F_{t})$, $\beta=\beta_{\Loc}$ in Lemma \ref{Lem:ajajajahshssgsgsgsss}, and by noticing that in this case we have $\rmr^{(1)}_{\Loc}=\tildrmr^{\Loc}$, $\rmP^{\Loc}_{(1)}=\tildrmP^{\Loc}$, $\rmr^{(2)}_{\Loc}=\tildrmr_{t}^{\Loc}$, $\rmP^{\Loc}_{(2)}=\rmP^{\Loc}_{t}$,
we obtain as desired:
  \begin{equation*}
  \begin{split}
  \abs{U_{t}(s,a)}&=\abs{(\rmT_{\Loc,t}Q^{\LQ}_{t})(s,a_{\Loc})-(\tildrmT^{\Loc}Q^{\LQ}_{t})(s,a_{\Loc})}\leq\left( \norm{\rmr^{\Loc}}_{\infty}+\beta_{\Loc} \norm{Q^{\LQ}_{t}}_{\infty}\right) \norm{\eta^{\Glob}_{t}(\cdot|\F_{t})-\eta^{\Glob}_{\infty}(\cdot|S_{t},A_{t}^{\Loc})}_{1}
  \end{split}
  \end{equation*} 

\end{proof}
\begin{proof}[Proof of Theorem \ref{Thm:akakashshssggdhdhdggdgdhdhd}]
	The proof that $(W_{t})$ (resp. $(U_{t})$) satisfies the first (resp. the fourth) condition of Proposition \ref{Prop:ajjahhssggsgdddddd} is given in Lemma \ref{Lem:aajsjshhsggdddhhdd} (resp. Lemma 	\ref{Lem:jajahshgsggsgsgsgsgsgss}). The third property follows from the fact that $\tildrmT^{\Loc}$ as a Bellman operator is a contraction with respect to discount factor of the underlying MDP. Therefore, this theorem is shown.
\end{proof}
\subsubsection{Proof of GAQL convergence (Theorem \ref{Thm:kiaiaahshsgsgsgsfffsfsfssss})}
\label{Subsubsec:GAQL}

\begin{proof}[Proof of Theorem \ref{Thm:kiaiaahshsgsgsgsfffsfsfssss}]
	Let $\hatrmT_{a_{\Loc},t}^{\Glob}$ be the optimal Bellman operator of the discounted MDP $(\S,\A_{\Glob},\tildrmr^{\Glob}_{a_{\Loc}},\delta_{S_{t+1}},\beta_{\Glob})$, and $\tildrmT^{\Glob}_{a_{\Loc}}$ be the optimal Bellman operator of the discounted MDP $(\S,\A_{\Glob},\tildrmr^{\Glob}_{a_{\Loc}},\tildrmP_{a_{\Loc}}^{\Glob},\beta_{\Glob})$.
	We can write the iterate of Algorithm 	\ref{Alg:ohgfdfddrrdrdrdrdrdghhgggdsadxxx} in the form:
	\begin{equation}
	\label{Eq:aoasjsshhssggsggsgsgsshhssfaaadddd}
	Q^{\GQ}_{a_{\Loc},t+1}=(1-\psi^{\Glob}_{a_{\Loc},t})\odot Q^{\GQ}_{a_{\Loc},t}+\psi^{\Glob}_{a_{\Loc},t}\odot\left[ \tildrmT^{\Glob}_{a_{\Loc}}Q^{\GQ}_{a_{\Loc},t}+W^{\Glob}_{a_{\Loc},t+1}\right] ,
	\end{equation}
	where:
	\begin{equation*}
	W^{\Glob}_{a_{\Loc},t+1}:=\mathbf{1}_{\lrbrace{S_{t}=s,A_{t}=a}}\left(  \hatrmT^{\Glob}_{a_{\Loc},t}Q^{\GQ}_{t}-\tildrmT_{a_{\Loc}}^{\Glob}Q^{\GQ}_{t}\right)\quad\text{and}\quad \psi_{a_{\Loc},t}^{\Glob}(s,a_{\Loc})=\mathbf{1}_{\lrbrace{S_{t}=s,A^{\Loc}_{t}=a_{\Loc},A^{\Glob}_{t}=a^{\Glob}_{t}}}\gamma_{t}^{\Glob}.
	\end{equation*}
	Notice that \eqref{Eq:aoasjsshhssggsggsgsgsshhssfaaadddd} has the form \eqref{Eq:aahshsggsgsfsfsfsssssssss} with $U_{t}=0$. So by checking the conditions (except the fourth condition) in Proposition \ref{Prop:ajjahhssggsgdddddd}, we can use the latter for showing the desired statement. 
	
	First, we check the third condition. Let be $a_{\Loc}\in\A_{\Loc}$ arbitrary. Let $\tildrmQ^{\Glob}_{a_{\Loc}}$ be the optimal Q-function of the discounted MDP $(\S,\A_{\Glob},\tildrmr^{\Glob}_{a_{\Loc}},\tildrmP_{a_{\Loc}}^{\Glob},\beta_{\Glob})$. We have for all $\rmQ\in\real^{\S\times\A_{\Glob}}$:
	\begin{equation}
	\label{Eq:akakashshssggdhdhdggdgdhdhd}
	\begin{split}
	\norm{\tildrmT_{a_{\Loc}}^{\Glob}\rmQ-\tildrmQ^{\Glob}_{a_{\Loc}}}_{\infty}&=\norm{\tildrmT_{a_{\Loc}}^{\Glob}Q-\tildrmT_{a_{\Loc}}^{\Glob}\tildrmQ^{\Glob}_{a_{\Loc}}}_{\infty}\leq\beta_{\Glob}\norm{\rmQ-\tildrmQ^{\Glob}_{a_{\Loc}}}_{\infty},
	\end{split}
	\end{equation}
	where the equality follows from the fact that $\tildrmQ^{\Glob}_{a_{\Loc}}$ is the fixed point of $\tildrmT_{a_{\Loc}}^{\Glob}$, and the inequality follows from the fact that the Bellman operator is a contraction w.r.t. the discount factor of the underlying discounted MDP. Thus, the third condition in Proposition \ref{Prop:ajjahhssggsgdddddd} is shown.

	Now, we check the first condition in Proposition \ref{Prop:ajjahhssggsgdddddd}. We have:
	\begin{equation*}
	\begin{split}
	&\Erw[(\hatrmT^{\Glob}_{A^{\Loc}_{t},t}Q^{\GQ}_{A^{\Loc}_{t},t})(S_{t},A_{t})|\tildF_{t}]=\tildrmr^{\Glob}_{A^{\Loc}_{t}}(S_{t},A^{\Glob}_{t})+\beta_{\Glob}\Erw[\max_{a^{'}_{\Glob}\in\A_{\Glob}}Q^{\GQ}_{t}(S_{t+1},A_{t}^{\Loc},a^{'}_{\Glob})|\tildF_{t}]\\
	&=\tildrmr^{\Glob}_{A^{\Loc}_{t}}(S_{t},A^{\Glob}_{t})+\beta_{\Glob}\sum_{s'\in\S}\rmP(s'|S_{t},A^{\Loc}_{t},A^{\Glob}_{t})\max_{a_{\Glob}'\in\A_{\Glob}}Q^{\GQ}_{A_{t}^{\Loc},t}(s',a_{\Glob}')\\
	&=\tildrmr^{\Glob}_{A^{\Loc}_{t}}(S_{t},A^{\Glob}_{t})+\beta_{\Glob}\Erw_{S'\sim\tildrmP^{\Glob}_{A^{\Loc}_{t}}(\cdot|S_{t},A^{\Glob}_{t})}[\max_{a_{\Glob}'\in\A_{\Glob}}Q^{\GQ}_{A_{t}^{\Loc},t}(S',a_{\Glob}')]\\
	&=\tildrmT^{\Glob}_{A_{t}^{\Loc}}Q^{\GQ}_{A_{t}^{\Loc},t}(S_{t},A^{\Glob}_{t})=\Erw[(\tildrmT^{\Glob}_{A_{t}^{\Loc}}Q^{\GQ}_{A_{t}^{\Loc},t})(S_{t},A^{\Glob}_{t})|\tildF_{t}],
	\end{split}
	\end{equation*}
	where the first equality follows from the fact that $\rmr^{\Glob}_{A^{\Loc}_{t}}(S_{t},A^{\Glob}_{t})$ only depends on $(S_{t},A^{\Loc}_{t},A^{\Glob}_{t})$ and thus $\tildF_{t}$-measureable, and where the last equality follows from the fact that $\tildrmT^{\Glob}_{A_{t}^{\Loc}}Q^{\Glob}_{A_{t}^{\Loc},t}$ only depends on $(S_{\tau},A^{\Loc}_{\tau},A^{\Glob}_{\tau})$, $\tau\in [t]_{0}$, and thus $\tildF_{t}$ measureable.
	Consequently, we have:
	\begin{equation*}
	\begin{split}
	&\Erw[W^{\Glob}_{a_{\Loc},t+1}(s,a)|\tildF_{t}]=\mathbf{1}_{\lrbrace{S_{t}=s,A^{\Loc}_{t}=a_{\Loc},A^{\Glob}_{t}=a_{\Glob}}}\Erw[(\hatrmT^{\Glob}_{A^{\Loc}_{t},t}Q^{\GQ}_{A^{\Loc}_{t},t})(S_{t},A^{\Glob}_{t})-\tildrmT^{\Glob}Q^{\GQ}_{A^{\Loc}_{t},t}(S_{t},A^{\Glob}_{t})|\tildF_{t}]=0.
	\end{split}
	\end{equation*}
	To show that the first condition in Proposition \ref{Prop:ajjahhssggsgdddddd} holds, it remains to derive the corresponding second moment bound. Similar argumentation as in the proof of Lemma \ref{Lem:aajsjshhsggdddhhdd} yields as desired:
	\begin{equation*}
	\Erw[(W_{a_{\Loc}t+1}^{\Glob}(s,a_{\Glob}))^{2}|\tildF_{t}]\leq 2\left( \norm{\rmr^{\Glob}}_{\infty}^{2}+\beta_{\Loc}^{2}\norm{Q^{\Glob}_{t}}_{\infty}\right) 
	\end{equation*}
	
\end{proof}
\subsection{Missing Proofs in Section \ref{Sec:OptProof}}
\label{Subsec:ajajajhshsggssfsfsfsssss}
\begin{proof}[Proof of Lemma \ref{Lem:ajajsshssgsgssfsfss}]
	Let $\pi_{\Loc}\in\Delta_{\S}(\A_{\Loc})$ be any LA's policy. Define the policy $\pi\in\Delta_{\S}(\A_{\Loc}\times\A_{\Glob})$ by:
	\begin{equation}
	\label{Eq:aajajasgsgsfsfffsfsss}
	\pi(a_{\Loc},a_{\Glob}|s):=\pi_{\Loc}(a_{\Loc}|s)\eta^{\Glob}_{\infty}(a_{\Glob}|s,a_{\Loc}),~ (a_{\Loc},a_{\Glob})\in\A_{\Loc}\times\A_{\Glob},~s\in\S.
	\end{equation}
	By above definition, it follows that  
	$\rmV^{\Loc}_{\pi_{\Loc},\eta^{\Glob}_{\infty}}$ is the value function of $\pi_{\Loc}$ in $(\S,\A_{\Loc}\times\A_{\Glob},\rmr^{\Loc},\rmP,\beta_{\Loc})$. By means of \eqref{Eq:aajshssgsgsgsffsfss}, we obtain:
	\begin{equation}
	\label{Eq:aiaishssgsggsgsfsfss}
	\begin{split}
	&\rmV^{\Loc}_{\pi_{\Loc},\eta^{\Glob}_{\infty}}(s)=\Erw_{(A_{\Loc},A_{\Glob})\sim\pi(\cdot|s)}\left[\rmr^{\Loc}(s,A)+\beta_{\Loc}\Erw_{S'\sim \rmP(\cdot|s,A)}\left[ \rmV^{\Loc}_{\pi_{\Loc},\eta^{\infty}_{\Glob}}(S')\right]  \right] \\
	&=\Erw_{A_{\Loc}\sim\pi_{\Loc}}\left[\Erw_{A_{\Glob}\sim\eta^{\Glob}_{\infty}(\cdot|s,A_{\Loc})}\left[ \rmr^{\Loc}(s,A_{\Loc},A_{\Glob})\right] +\beta^{\Loc}\Erw_{A_{\Glob}\sim\eta^{\Glob}_{\infty}(\cdot|s,A_{\Loc})}\left[\Erw_{S'\sim \rmP(\cdot|s,A)}\left[ \rmV^{\Loc}_{\pi_{\Loc},\eta_{\infty}^{\Glob}}(S')\right] \right]  \right]\\
	&=\Erw_{A_{\Loc}\sim\pi_{\Loc}}\left[\tildrmr^{\Loc}(s,A_{\Loc})+\beta^{\Loc}\Erw_{S'\sim \tildrmP_{\Loc}(\cdot|s,A_{\Loc})}\left[ \rmV^{\Loc}_{\pi_{\Loc},\eta_{\infty}^{\Glob}}(S')\right]  \right],
	\end{split}
	\end{equation}
	where the second equality follows by writing out the definition of $\pi$ and the third inequality from the definition of $\tildrmr^{\Loc}$ and $\tildrmP^{\Loc}$. From above computation and \eqref{Eq:aajshssgsgsgsffsfss}, we have that $\rmV^{\Loc}_{\pi_{\Loc},\eta^{\Glob}_{\infty}}$ is the value function of $\pi_{\Loc}$ in $(\S,\A_{\Loc},\tildrmr^{\Loc},\tildrmP^{\Loc},\beta^{\Loc})$. Now, Theorem \ref{Thm:akakashshssggdhdhdggdgdhdhd} asserts that $\pi_{\Loc}^{\LAQGI}$ is the optimal policy of the MDP $(\S,\A_{\Loc},\tildrmr_{\Loc},\tildrmP_{\Loc},\beta_{\Loc})$ and therefore, its value function in $(\S,\A_{\Loc},\tildrmr_{\Loc},\tildrmP_{\Loc},\beta_{\Loc})$  dominates the value function of $\pi_{\Loc}$, which is, as shown before, equal to $\rmV^{\Loc}_{\pi_{\Loc},\eta^{\Glob}_{\infty}}$. Finally, we obtain the desired statement by noticing that the value function of $\pi_{\Loc}^{\LAQGI}$ in $(\S,\A_{\Loc},\tildrmr_{\Loc},\tildrmP_{\Loc},\beta_{\Loc})$ is equal to $\rmV_{\pi_{\Loc}^{\LAQGI},\eta^{\Glob}_{\infty}}$ by the similar argumentation as in \eqref{Eq:aiaishssgsggsgsfsfss}
\end{proof}
\begin{proof}[Proof of Lemma \ref{Lem:ahhasgsgsfsfsfsffssss}]
	Let $\pi_{\Glob}\in\Delta_{\S\times\A_{\Loc}}(\A_{\Glob})$ be arbitrary. By the similar argumentation as in \eqref{Eq:aiaishssgsggsgsfsfss},
	we have that:
	\begin{equation}
	\label{Eq:ajajajssggssggsfsfsfdddddd}
	\begin{split}
	\rmV_{\pi_{\Loc},\pi_{\Glob}}^{\Glob}(s)&=\Erw_{A_{\Glob}\sim\pi_{\Glob}(\cdot|s,\pi_{\Loc}(s))}\left[ \overrmr^{\Glob}_{\pi_{\Loc}(s)}(s,A_{\Glob}) +\beta^{\Glob}\Erw_{S'\sim\tildrmP^{\Glob}_{\pi_{\Loc}(s)}(\cdot|s,A_{\Glob})}\left[\rmV_{\pi_{\Loc},\pi_{\Glob}}^{\Glob}(S') \right] \right],
	\end{split}
	\end{equation}
	where:
	\begin{equation*}
	\overrmr^{\Glob}(s,a_{\Glob}):=\tildrmr^{\Glob}_{\pi_{\Loc}(s)}(s,a_{\Glob})~\text{and}~\overrmP^{\Glob}(\cdot|s,a_{\Glob}):=\tildrmP^{\Glob}_{\pi_{\Loc}(s)}(\cdot|s,a_{\Glob}),
	\end{equation*}
	with $\tildrmr^{\Glob}_{a_{\Loc}}$ and $\tildrmP^{\Glob}_{a_{\Loc}}$, $a_{\Loc}\in\A_{\Loc}$ is defined in \eqref{Eq:ajjahsgsgsgsffffsffsssssss}.
	We obtain from \eqref{Eq:ajajajssggssggsfsfsfdddddd} that:
	$\rmV_{\pi_{\Loc},\pi_{\Glob}}^{\Glob}$ is the value function of the policy $\tilde{\pi}_{\Glob}\in\Delta_{\S}(\A_{\Glob})$, with $\tilde{\pi}_{\Glob}(a_{\Glob}|s):=\pi_{\Glob}(a_{\Glob}|s,\pi_{\Loc}(s))$, in the discounted MDP $(\S,\A_{\Glob},\overrmr^{\Glob},\overrmP^{\Glob},\beta^{\Glob})$.
	
	Now, the fact that for any $a_{\Loc}\in\A_{\Loc}$, $\tildrmQ_{a_{\Loc}}^{\Glob}$ is the optimal Q-function for the discounted MDP $(\S,\A_{\Glob},\tildrmr^{\Glob}_{a_{\Loc}},\tildrmP^{\Glob}_{a_{\Loc}},\beta_{\Glob})$ and the relation \eqref{Eq:akjajsjshgdgdgdggdgdgdgdd} yields that the function $\overrmQ^{\Glob}(s,a_{\Glob}):=\tildrmQ_{\pi_{\Loc}(s)}^{\Glob}(s,a_{\Glob})$ satisfies:
	\begin{equation}
	\label{Eq:jaajshhsgsgsggsffssggsfssgs}
	\overrmQ^{\Glob}(s,a_{\Glob})=\overrmr^{\Glob}(s,a_{\Glob})+\beta_{\Glob}\Erw_{S'\sim\overrmP^{\Glob}(\cdot|s,a_{\Glob})}\left[ \max_{a'_{\Glob}}\overrmQ^{\Glob}(S',a'_{\Glob})\right], 
	\end{equation}
	and consequently $\overrmQ^{\Glob}$ is the optimal Q-function of the discounted MDP $(\S,\A_{\Glob},\overrmr^{\Glob},\overrmP^{\Glob},\beta^{\Glob})$. Now, let:
	\begin{equation*}
	 \overline{\rmV}^{\Glob}(s):=\max_{a_{\Glob}\in\A_{\Glob}}\overline{\rmQ}^{\Glob}(s,a_{\Glob})
	 \end{equation*}
	  be the corresponding value function. By the optimality of $\overrmQ^{\Glob}$ and the fact that $\rmV_{\pi_{\Loc},\pi_{\Glob}}^{\Glob}$ is the value function of a policy $\pi_{\Glob}\in\Delta_{\S\times\A_{\Loc}}(\A_{\Glob})$ in $(\S,\A_{\Glob},\overrmr^{\Glob},\overrmP^{\Glob},\beta^{\Glob})$ (see the previous paragraph), we have that $\overline{\rmV}^{\Glob}\geq\rmV_{\pi_{\Loc},\pi_{\Glob}}^{\Glob}$. 
	
	It remains now to show that $\overline{\rmV}^{\Glob}=\rmV^{\Glob}_{\pi_{\Loc},\pi^{\GAQL}_{\Glob}}$. Since:
	\begin{equation*}
	\overline{\rmV}^{\Glob}(s)=\max_{a_{\Glob}\in\A_{\Glob}}\overline{\rmQ}^{\Glob}(s,a_{\Glob})=\Erw_{A_{\Glob}\sim\pi_{\Glob}^{\GAQL}(\cdot|s,\pi_{\Loc}(s))}[\overline{\rmQ}^{\Glob}(s,A_{\Glob})],
	\end{equation*}
	and since:
	\begin{equation*}
	\rmV^{\Glob}_{\pi_{\Loc},\pi^{\GAQL}_{\Glob}}(s)=\Erw_{A_{\Glob}\sim\pi_{\Glob}^{\GAQL}(\cdot|s,\pi_{\Loc}(s))}\left[ \rmQ_{\pi}(s,\pi_{\Loc}(s),A_{\Glob})\right],
	\end{equation*}
	where $\rmQ_{\pi}$ is the Q-function of the policy $\pi\in\Delta_{\S}(\A_{\Loc}\times\A_{\Glob})$, with $\pi(a_{\Loc},a_{\Glob}|s):=\pi_{\Loc}(a_{\Loc}|s)\pi_{\Glob}(a_{\Glob}|s,a_{\Loc})$, in $(\S,\A_{\Loc}\times\A_{\Glob},\rmr^{\Glob},\rmP,\beta_{\Glob})$, it is sufficient to show that $\rmQ_{\pi}(s,\pi_{\Loc}(s),a_{\Glob})=\overline{\rmQ}^{\Glob}(s,A_{\Glob})$. Toward this end, it is straightforward to see that by \eqref{Eq:aaksjsjshhsgsgsggshsgsgsggsgs} that $\rmQ_{\pi}(s,\pi_{\Loc}(s),a_{\Glob})$ fulfills:
	\begin{equation*}
	\begin{split}
	\rmQ_{\pi}(s,\pi_{\Loc}(s),a_{\Glob})&=\rmr^{\Glob}(s,\pi_{\Loc}(s),a_{\Glob})+\beta_{\Glob}\Erw_{S'\sim\rmP(\cdot|s,\pi_{\Loc}(s),a_{\Glob})}\left[\Erw_{A_{\Glob}\sim\pi_{\Glob}^{\GAQL}(\cdot|S',\pi_{\Loc}(S'))}\left[\rmQ_{\pi}(S',\pi_{\Loc}(S'),A_{\Glob}) \right]  \right] \\
	&=\overrmr^{\Glob}(s,a_{\Glob})+\beta_{\Glob}\Erw_{S'\sim\overrmP^{\Glob}(\cdot|s,a_{\Glob})}\left[\Erw_{A_{\Glob}\sim\tilde{\pi}_{\Glob}(\cdot|S')}\left[ \rmQ_{\pi}(S',\pi_{\Loc}(S'),a_{\Glob})\right]  \right].
	\end{split}
	\end{equation*}
	where $\tilde{\pi}_{\Glob}\in\Delta_{\S}(\A_{\Glob})$, with $\tilde{\pi}_{\Glob}(a_{\Glob}|s):=\pi_{\Glob}^{\GAQL}(a_{\Glob}|s,\pi_{\Loc}(s))$. Moreover, we can write \eqref{Eq:jaajshhsgsgsggsffssggsfssgs} by definition of $\pi_{\Glob}^{\GAQL}$ as follows:
	\begin{equation*}
	\begin{split}
	\overrmQ^{\Glob}(s,a_{\Glob})&=\overrmr^{\Glob}(s,a_{\Glob})+\beta_{\Glob}\Erw_{S'\sim\overrmP^{\Glob}(\cdot|s,a_{\Glob})}\left[\Erw_{A_{\Glob}\sim\tilde{\pi}_{\Glob}(\cdot|S')}\left[  \overrmQ^{\Glob}(S',A_{\Glob})\right] \right]. 
	\end{split}
	\end{equation*}
	Consequently, $\rmQ_{\pi}(s,\pi_{\Loc}(s),a_{\Glob})$ and $\overrmQ^{\Glob}(s,a_{\Glob})$ are solutions for the Bellman equation for Q-function of $\tilde{\pi}_{\Glob}$ in $(\S,\A_{\Glob},\overrmr^{\Glob},\overrmP^{\Glob})$ (see \eqref{Eq:aaksjsjshhsgsgsggshsgsgsggsgs}). Uniqueness of the solution of a Bellman equation yields finally the desired statement.  
\end{proof}
\begin{proof}[Proof of Theorem 	\ref{Thm:ajajhshssggsgssgsgffsfssssss}]
	Since $\pi_{\Loc}^{\LAQGI}\in\Delta_{\S}(\A_{\Loc})$ is a deterministic policy, $\rmV_{\pi^{\LAQGI}_{\Loc},\pi^{\GAQL}_{\Glob}}^{\Glob}\geq \rmV_{\pi_{\Loc}^{\LAQGI},\pi_{\Glob}}^{\Glob}$ follows from Lemma \ref{Lem:ajajsshssgsgssfsfss}. So, it remains to show $\rmV_{\pi^{\LAQGI}_{\Loc},\pi^{\GAQL}_{\Glob}}^{\Loc}\geq \rmV_{\pi_{\Loc},\pi^{\GAQL}_{\Glob}}^{\Loc}-\epsilon$. 
	
	First, notice that given a deterministic $\pi_{\Loc}\in\Delta_{\S}(\A_{\Loc})$, one can show that:
	\begin{equation*}
	\rmV_{\pi_{\Loc},\pi^{\GAQL}_{\Glob}}^{\Loc}=\rmV^{\Loc}_{\pi_{\Loc},\tilde{\pi}^{\GAQL}_{\Glob}},
	\end{equation*}
	where $\tilde{\pi}^{\GAQL}_{\Glob}(\cdot|s,a_{\Loc})$ is uniformly distributed in  $\argmax_{a_{\Glob}\in\A_{\Glob}}\tildrmQ_{a_{\Loc}}(s,a_{\Glob})$. Thus we need only to check the desired inequality with $\tilde{\pi}^{\GAQL}_{\Glob}$ in place of $\pi^{\GAQL}_{\Glob}$. For this sake, notice first that by the similar argumentation as in \eqref{Eq:aiaishssgsggsgsfsfss}, it holds that $\rmV_{\tilde{\pi}^{\LAQGI}_{\Loc},\pi^{\GAQL}_{\Glob}}^{\Loc}$ is the value function of the policy $\tilde{\pi}^{\LAQGI}_{\Loc}$ in  $(\S,\A_{\Loc},\check{\rmr}^{\Loc},\check{\rmP}^{\Loc},\beta_{\Loc})$, where:
	\begin{equation*}
	\begin{split}
	&\check{\rmr}^{\Loc}(s,a_{\Loc}):=\Erw_{A_{\Glob}\sim\tilde{\pi}^{\GAQL}_{\Glob}(\cdot|s,a_{\Loc})}\left[\rmr^{\Loc}(s,a_{\Loc},A_{\Glob}) \right]\\
	& \check{\rmP}^{\Loc}(\cdot|s,a_{\Loc}):=\Erw_{A_{\Glob}\sim\tilde{\pi}^{\GAQL}_{\Glob}(\cdot|s,a_{\Loc})}\rmP(\cdot|s,a_{\Loc},A_{\Glob}).
	\end{split}
	\end{equation*}
	Furthermore, we have that $\rmV_{\tilde{\pi}^{\LAQGI}_{\Loc},\eta^{\GQ}_{\infty}}^{\Loc}$ is the value function of the policy $\tilde{\pi}^{\LAQGI}_{\Glob}$ in $(\S,\A_{\Loc},\tilde{\rmr}^{\Loc},\tilde{\rmP}^{\Loc},\beta_{\Loc})$, where $\tilde{\rmr}^{\Loc}$ and $\tilde{\rmP}^{\Loc}$ is given in \eqref{Eq:ajajahshssggssgshhsgsgshss}. Consequently by \eqref{Eq:ajajsgsgsffsfddsfsdsdsfsdss} in Lemma 	\ref{Lem:ajajajahshssgsgsgsss}, we have:
	\begin{equation*}
	\begin{split}
	\norm{\rmV_{\pi^{\LAQGI}_{\Loc},\tilde{\pi}^{\GAQL}_{\Glob}}^{\Loc}-\rmV_{\pi^{\LAQGI}_{\Loc},\eta^{\GQ}_{\infty}}^{\Loc}}_{\infty}\leq \frac{\norm{\rmr^{\Loc}}_{\infty}}{(1-\beta_{\Loc})^{2}}\max_{s,a_{\Loc}} \norm{\tilde{\pi}^{\GAQL}_{\Glob}(\cdot|s,a_{\Loc})-\eta^{\GQ}_{\infty}(\cdot|s,a_{\Loc})}_{1}.
	\end{split}
	\end{equation*}
	Furthermore as $\eta_{\infty}^{\GQ}$ is the Boltzmann strategy \eqref{Eq:jajajjashhshsgsgsgshshsss}, we have by Lemma \ref{Lem:jaajjsshsggsfsfsffsgsfss}:
	\begin{equation*}
	\begin{split}
	\norm{\tilde{\pi}_{\Glob}^{\GAQL}(\cdot|s,a_{\Loc})-\eta^{\GQ}_{\infty}(\cdot|s,a_{\Loc})}_{1}\leq D_{s,a_{\Loc}}\exp\left(-\frac{C_{s,a_{\Loc}}}{2\tau} \right), 
	\end{split}
	\end{equation*}
	Consequently, we have by combining both previous estimates:
	\begin{equation*}
	\norm{\rmV_{\pi^{\LAQGI}_{\Loc},\tilde{\pi}^{\GAQL}_{\Glob}}^{\Loc}-\rmV_{\pi^{\LAQGI}_{\Loc},\eta^{\GQ}_{\infty}}^{\Loc}}_{\infty}\leq\frac{\norm{\rmr^{\Loc}}_{\infty}D}{(1-\beta_{\Loc})^{2}}\exp\left(-\frac{C}{\tau} \right).
	\end{equation*}
	By similar argumentation, we obtain:
	\begin{equation*}
	\norm{\rmV_{\pi_{\Loc},\pi^{\GAQL}_{\Loc}}^{\Loc}-\rmV_{\pi_{\Loc},\tilde{\eta}^{\GQ}_{\infty}}^{\Loc}}_{\infty}\leq \frac{\norm{\rmr^{\Loc}}_{\infty}D}{(1-\beta_{\Loc})^{2}}\exp\left(-\frac{C}{\tau} \right)
	\end{equation*}
	Consequently:
	\begin{equation*}
	\rmV_{\pi^{\LAQGI}_{\Glob},\pi^{\GAQL}_{\Glob}}^{\Loc}+\frac{\epsilon}{2}\geq\rmV_{\pi^{\LAQGI}_{\Loc},\tilde{\eta}^{\GQ}_{\infty}}^{\Loc}\geq\rmV_{\pi_{\Loc},\tilde{\eta}^{\GQ}_{\infty}}^{\Loc}\geq\rmV_{\pi_{\Loc},\pi^{\GAQL}_{\Glob}}^{\Loc}-\frac{\epsilon}{2},
	\end{equation*}
	as desired
\end{proof}
\subsection{Missing Proofs in Section \ref{Sec: Nash}}
\label{Subsec:ProofNash}
\begin{proof}[Proof of Lemma \ref{Lem:ajhajasggffsfsfsfsss}]
$\rmH^{\Loc}_{\pi_{\Loc},\pi_{\Glob}}\leq\rmH^{\Loc}_{*,\pi_{\Glob}}$ is trivial by definition. 
To prove $\rmH^{\Glob}_{\pi_{\Loc},\pi_{\Glob}}\leq\rmH^{\Glob}_{\pi_{\Loc},*}$, let $\rmV_{\Glob}\in\real^{\S}$ be arbitrary. Notice that:
\begin{equation*}
    \Erw_{A_{\Glob}\sim\pi_{\Glob}(\cdot|s,A_{\Loc})}[\rmr^{\Glob}(s,A_{\Loc},A_{\Glob}) + \beta_{\Glob}\Erw_{S'\sim \rmP(\cdot|s,A_{\Loc},A_{\Glob})}[\rmV_{\Glob}(S')]] \leq \max_{\pi_{\Glob} \in \Delta} \Erw_{A_{\Glob}\sim\pi_{\Glob}(\cdot|s,A_{\Loc})}[\rmr^{\Glob}(s,A_{\Loc},A_{\Glob}) + \beta_{\Glob}\Erw_{S'\sim \rmP(\cdot|s,A_{\Loc},A_{\Glob})}[\rmV_{\Glob}(S')]]
\end{equation*}
Taking the expectation on both sides w.r.t. $A_{\Loc} \sim \pi_{\Loc}(\cdot|s)$ and by the monotonicity of the expectation operator we get
\begin{equation*}
\begin{split}
   (\rmH^{\Glob}_{\pi_{\Loc},\pi_{\Glob}}\rmV_{\Glob})(s)&= \Erw_{\substack{A_{\Loc}\sim\pi_{\Loc}(\cdot|s)\\A_{\Glob}\sim\pi_{\Glob}(\cdot|s,A_{\Loc})}}\left[\rmr^{\Glob}(s,A_{\Loc},A_{\Glob}) + \beta_{\Glob}\Erw_{S'\sim \rmP(\cdot|s,A_{\Loc},A_{\Glob})}[\rmV_{\Glob}(S')]\right] \\
   &\leq \Erw_{A_{\Loc}\sim\pi_{\Loc}(\cdot|s)}\left[\max_{\pi_{\Glob} \in \Delta} \Erw_{A_{\Glob}\sim\pi_{\Glob}(\cdot|s,A_{\Loc})}[\rmr^{\Glob}(s,A_{\Loc},A_{\Glob}) + \beta_{\Glob}\Erw_{S'\sim \rmP(\cdot|s,A_{\Loc},A_{\Glob})}[\rmV_{\Glob}(S')]]\right]\leq (\rmH^{\Glob}_{*,\pi_{\Glob}}\rmV_{\Glob})(s)
   \end{split}
\end{equation*}
concluding the proof.
\end{proof}
\begin{proof}[Proof of Lemma \ref{Lem:aidgfgfgefefeeeeeeeeee}]
For any agent $i \in \{\Loc,\Glob\}$ and state $s \in \mathcal{S}$, we have:
\begin{equation*}
    \begin{split}
        \rmV_i(s) & = (\rmH_{\pi_{\Loc},\pi_{\Glob}}^{(i)} \rmV_i)(s)= \Erw_{\substack{A_{\Loc}\sim\pi_{\Loc}(\cdot|s)\\A_{\Glob}\sim\pi_{\Glob}(\cdot|s,A_{\Loc})}}\left[\rmr^{(i)}(s,A_{\Loc},A_{\Glob}) + \beta_{i}\Erw_{S'\sim \rmP(\cdot|s,A_{\Loc},A_{\Glob})}[\rmV_{i}(S')]\right]
    \end{split}
\end{equation*}
Taking the norm $\|\cdot\|_{\infty} := \max_{s,i}|\cdot|$ on both sides yields:
\begin{equation*}
  \begin{split}
      \norm{\rmV}_{\infty}=\max_{s,i}|\rmV_i(s)| & = \max_{s,i}\left| \Erw_{\substack{A_{\Loc}\sim\pi_{\Loc}(\cdot|s)\\A_{\Glob}\sim\pi_{\Glob}(\cdot|s,A_{\Loc})}}\left[\rmr^{(i)}(s,A_{\Loc},A_{\Glob}) + \beta_{i}\Erw_{S'\sim \rmP(\cdot|s,A_{\Loc},A_{\Glob})}[\rmV_{i}(S')]\right]\right|\\
      & \leq \max_{s,i} \left[ \max_{(a_{\Loc},a_{\Glob})} \left|\rmr^{(i)}(a_{\Loc},a_{\Glob}) + \beta_{i}\Erw_{S'\sim \rmP(\cdot|s,a_{\Loc},a_{\Glob})}[\rmV_{i}(S')]\right| \right]\\
      & \leq \max_{s,i,(a_{\Loc},a_{\Glob})} |\rmr^{(i)}(s,a_{\Loc},a_{\Glob})| + \beta \max_{s,i,(a_{\Loc},a_{\Glob})}|\Erw_{S'\sim \rmP(\cdot|s,a_{\Loc},a_{\Glob})}[\rmV_{i}(S')]|\\
      & \leq \max_{s,i,(a_{\Loc},a_{\Glob})} |\rmr^{(i)}(s,a_{\Loc},a_{\Glob})| + \beta \max_{s',i}|\rmV_i(s')|=\norm{\rmr}_{\infty}+\beta \norm{\rmV}_{\infty}.
  \end{split}  
\end{equation*}
Thus, we have $(1-\beta)\norm{\rmV}\leq \norm{\rmr}_{\infty}$ yielding the desired statement.
\end{proof}

\begin{proof}[Proof of Lemma \ref{Lem:ahjahshsgsgsfffsgfgssssss}]
	Consider the agent $i \in \{\Loc,\Glob\}$.\\
	\textit{Lipschitz in $V$:}\\
	\begin{equation*}
		\begin{split} 
		    \|\rmH_{\pi_{\Loc},\pi_{\Glob}} \rmV - \rmH_{\pi_{\Loc},\pi_{\Glob}} \tilde{\rmV}\|_{\infty} & =
			\max_{s,i}|(\rmH^{(i)}_{\pi_{\Loc},\pi_{\Glob}} \rmV_i)(s) - (\rmH^{(i)}_{\pi_{\Loc},\pi_{\Glob}} \tilde{\rmV}_i)(s)|\\
			& = \max_{s,i} \beta_{i}\left|\Erw_{\substack{A_{\Loc}\sim\pi_{\Loc}(\cdot|s)\\A_{\Glob}\sim\pi_{\Glob}(\cdot|s,A_{\Loc})}}\left[\Erw_{S'\sim \rmP(\cdot|s,A_{\Loc},A_{\Glob})}[\rmV_i(S') - \tilde{\rmV}_i(S')]\right]\right|\\
			& \leq \beta \max_{s,i,(a_{\Loc},a_{\Glob})} \left(\Erw_{S'\sim \rmP(\cdot|s,(a_{\Loc},a_{\Glob}))}[|\rmV_i(S') - \tilde{\rmV}_i(S')|]\right)\\
			& \leq \beta \max_{s,i,(a_{\Loc},a_{\Glob})} \left(\|\rmP(\cdot|s,(a_{\Loc},a_{\Glob}))\|_1 \max_{s'}|\rmV_i(s') - \tilde{\rmV}_i(s')| \right)\\
			& \leq \beta \max_{s',i}|\rmV_i(s') - \tilde{\rmV}_i(s')|\\
			& \leq \beta \|\rmV - \tilde{\rmV}\|_{\infty}
		\end{split}
	\end{equation*}
	\newline
	\textit{Lipschitz in $\pi_{\Loc}$:}\\
	\begin{equation*}
		\begin{split} 
		     \|\rmH_{\pi_{\Loc},\pi_{\Glob}} \rmV - \rmH_{\tilde{\pi}_{\Loc},\pi_{\Glob}} \rmV\|_{\infty} & =
			\max_{s,i}|(\rmH^{(i)}_{\pi_{\Loc},\pi_{\Glob}} \rmV_i)(s) - (\rmH^{(i)}_{\tilde{\pi}_{\Loc},\pi_{\Glob}} \rmV_i)(s)|\\
			& = \max_{s,i}\left|\sum_{a_{\Loc}}\left(\sum_{a_{\Glob}}\pi_{\Glob}(a_{\Glob}|a_{\Loc},s)(\rmr^{(i)}(s,a_{\Loc},a_{\Glob})+ \beta_i\Erw_{S'\sim \rmP(\cdot|s,a)}[\rmV_i(S')])\right) (\pi_{\Loc}(a_{\Loc}|s)-\tilde{\pi}_{\Loc}(a_{\Loc}|s))\right|\\
			& \leq \max_{s,i} \left|\langle \sum_{a_{\Glob}}\pi_{\Glob}(a_{\Glob}|\cdot,s)(\rmr^{(i)}(s,\cdot,a_{\Glob})+ \beta_i\Erw_{S'\sim \rmP(\cdot|s,\cdot,a_{\Glob})}[\rmV_i(S')],\pi_{\Loc}(\cdot|s)-\tilde{\pi}_{\Loc}(\cdot|s)\rangle\right|,
		\end{split}
    \end{equation*}
    where $\inn{\cdot}{\cdot}$ denotes the usual inner product.
    We apply Hölders inequality and define $\|\pi_{\Loc}-\tilde{\pi}_{\Loc}\|_1:=\max_{s}\|\pi_{\Loc}(\cdot|s)-\tilde{\pi}_{\Loc}(\cdot|s)\|_1$. 
	\begin{equation*}
	    \begin{split}
			\|\rmH_{\pi_{\Loc},\pi_{\Glob}} \rmV - \rmH_{\tilde{\pi}_{\Loc},\pi_{\Glob}} \rmV\|_{\infty} & \leq \max_{s,i,a_{\Loc}} \left|\sum_{a_{\Glob}}\pi_{\Glob}(a_{\Glob}|a_{\Loc},s)\left(\rmr^{(i)}(s,a_{\Loc},a_{\Glob})+ \beta_i\Erw_{S'\sim \rmP(\cdot|s,a_{\Loc},a_{\Glob})}[\rmV_i(S')]\right)\right| \cdot \|\pi_{\Loc}-\tilde{\pi}_{\Loc}\|_1 \\
			& \leq \max_{s,i,(a_{\Loc},a_{\Glob})}\left| \rmr^{(i)}(s,a_{\Loc},a_{\Glob})+ \beta_i\Erw_{S'\sim \rmP(\cdot|s,a)}[\rmV_i(s')] \right| \cdot \|\pi_{\Loc}-\tilde{\pi}_{\Loc}\|_1 \\
			& \leq \left( \max_{s,i,(a_{\Loc},a_{\Glob})} |\rmr^{(i)}(s,a_{\Loc},a_{\Glob})| +  \beta\max_{s,i,a} \left|\Erw_{S'\sim \rmP(\cdot|s,a)}[\rmV_i(S')]\right|\right)\cdot\|\pi_{\Loc}-\tilde{\pi}_{\Loc}\|_1
		\end{split}
	\end{equation*}
	Since $\max_{s,i,a} \left|\Erw_{S'\sim \rmP(\cdot|s,a)}[\rmV_i(S')]\right| \leq \|\rmV\|_\infty$ and by applying Lemma \ref{Lem:aidgfgfgefefeeeeeeeeee} to $\norm{\rmV}_{\infty}$ it follows
	\begin{equation*}
		 \|\rmH_{\pi_{\Loc},\pi_{\Glob}} \rmV - \rmH_{\tilde{\pi}_{\Loc},\pi_{\Glob}} \rmV\|_{\infty} \leq  \frac{\|\rmr\|_\infty}{1-\beta} \|\pi_{\Loc}-\tilde{\pi}_{\Loc}\|_1
	\end{equation*}
	\newline
	\textit{Lipschitz in $\pi_{\Glob}$:}\\
	\begin{equation*}
		\begin{split}
		    \|\rmH_{\pi_{\Loc},\pi_{\Glob}} \rmV - \rmH_{\pi_{\Loc},\tilde{\pi}_{\Glob}} \rmV\|_{\infty} &=
			\max_{s,i}|(\rmH^{(i)}_{\pi_{\Loc},\pi_{\Glob}} \rmV_i)(s) - (\rmH^{(i)}_{\pi_{\Loc},\tilde{\pi}_{\Glob}} \rmV_i)(s)|\\
			& = \max_{s,i}\left|\sum_{a_{\Loc}}\pi_{\Loc}(a_{\Loc}|s)\left(\sum_{a_{\Glob}}\left(\rmr^{(i)}(s,a_{\Loc},a_{\Glob})+ \beta_i\Erw_{S'\sim \rmP(\cdot|s,a)}[\rmV_i(S')]\right) (\pi_{\Glob}(a_{\Glob}|a_{\Loc},s)-\tilde{\pi}_{\Glob}(a_{\Glob}|a_{\Loc},s))\right)\right|\\
			& \leq \max_{s,i,a_{\Loc}} \left| \langle \rmr^{(i)}(s,a_{\Loc},\cdot)+ \beta_i\Erw_{S'\sim \rmP(\cdot|s,a_{\Loc},\cdot)}[\rmV_i(S')], \pi_{\Glob}(\cdot|a_{\Loc},s)-\tilde{\pi}_{\Glob}(\cdot|a_{\Loc},s) \rangle \right|\\
			\end{split}
    \end{equation*}
    We apply Hölders inequality and define $\|\pi_{\Glob}-\tilde{\pi}_{\Glob}\|_1:=\max_{s,,a_{\Loc}}\|\pi_{\Glob}(\cdot|a_{\Loc},s)-\tilde{\pi}_{\Glob}(\cdot|a_{\Loc},s)\|_1$. 
    \begin{equation*}
		    \|\rmH_{\pi_{\Loc},\pi_{\Glob}} \rmV - \rmH_{\pi_{\Loc},\tilde{\pi}_{\Glob}} \rmV\|_{\infty} \leq \max_{s,i,(a_{\Loc},a_{\Glob})} \left| \rmr^{(i)}(s,a_{\Loc},a_{\Glob})+ \beta_i\Erw_{S'\sim \rmP(\cdot|s,a)}[\rmV_i(S')] \right| \cdot  \|\pi_{\Glob}-\tilde{\pi}_{\Glob}\|_1\\
	\end{equation*}
	Following the same steps as beforehand we get
	\begin{equation*}
		 \|\rmH_{\pi_{\Loc},\pi_{\Glob}} \rmV - \rmH_{\pi_{\Loc},\tilde{\pi}_{\Glob}} \rmV\|_{\infty} \leq  \frac{\|\rmr\|_\infty}{1-\beta} \|\pi_{\Glob}-\tilde{\pi}_{\Glob}\|_1
	\end{equation*}
	concluding the proof.
\end{proof}
\begin{proof}[Proof of Lemma \ref{Lem:Uppersem}]Let us denote $\pi:=(\pi_{\Loc},\pi_{\Glob})$ and $\pi:=(\pi^{*}_{\Loc},\pi^{*}_{\Glob})$
Consider the sequence $\{\Sigma(\pi^n_{\Loc},\pi_{\Glob}^n)\}_{n \in \mathbb{N}}$ which converges to $\rmV = [\rmV_{\Loc},\rmV_{\Glob}]^T$ for $n \to \infty$. It holds
\begin{equation*}
    \|\rmH_{\pi,\pi^*}\rmV - \rmV\|_{\infty} \leq \underbrace{\|\rmH_{\pi,\pi^*}\rmV -\rmH_{(\pi,\pi^{*})_n}\Sigma(\pi^n_{\Loc},\pi_{\Glob}^n)\|_{\infty}}_{:= A} + \underbrace{\|\rmH_{(\pi,\pi^{*})_n}\Sigma(\pi^n_{\Loc},\pi_{\Glob}^n) - \rmV\|_{\infty}}_{:= B}
\end{equation*}
due to the triangle inequality. By Lemma \ref{Lem:ahjahshsgsgsfffsgfgssssss}, we know that the operator is Lipschitz-continuous in each argument. We can therefore derive the following upper bound for $A$:
\begin{equation*}
    \begin{split}
        A & \leq \beta \|\rmV - \Sigma(\pi^n_{\Loc},\pi_{\Glob}^n)\|_{\infty} + \frac{\|\rmr\|_\infty}{1-\beta} \max\{\|\pi^{*}_{\Loc}-\pi^{*,n}_{\Loc}\|_1, \|\pi_{\Loc}-\pi^n_{\Loc}\|_1\} \\
        & + \frac{\|\rmr\|_\infty}{1-\beta}\max\{\|\pi_{\Glob}-\pi^n_{\Glob}\|_1,\|\pi^{*}_{\Glob}-\pi^{*,n}_{\Glob}\|_1\}\xrightarrow{n\rightarrow\infty} 0,
    \end{split}
\end{equation*}
where $\beta>0$ is a constant.
For $B$, we use the fact, that $(\pi^{*,n}_{\Loc},\pi^{*,n}_{\Glob}) \in \phi(\pi^n_{\Loc},\pi_{\Glob}^n)$. By \eqref{phidef} we have $\rmH_{(\pi,\pi^{*})_n}\Sigma(\pi^n_{\Loc},\pi_{\Glob}^n) = \Sigma(\pi^n_{\Loc},\pi_{\Glob}^n)$. Thus
\begin{equation*}
    B = \|\Sigma(\pi^n_{\Loc},\pi_{\Glob}^n) - \rmV\|_{\infty} \xrightarrow{n\rightarrow\infty} 0
\end{equation*}
and due to the positive definiteness of the norm we conclude
\begin{equation*}
    \|\rmH_{\pi,\pi^*}\rmV - \rmV\|_{\infty} = 0 \implies \rmH_{\pi,\pi^*}\rmV = \rmV
\end{equation*}
$\rmV$ is the fixed point such that $\rmH_{\pi,\pi^*}\rmV = \rmV = \rmH_{\pi_{\Loc},\pi_{\Glob}}^* \rmV$ is satisfied. Thus $(\pi^{*}_{\Loc},\pi^{*}_{\Glob}) \in \phi(\pi_{\Loc},\pi_{\Glob})$ and the correspondence  $\phi: \Delta \to \wp(\Delta)$ is upper semi-continuous.
\end{proof}

\subsection{Auxiliary Statements}
\begin{lemma}
	\label{Lem:ajajasshshsgsgsgsfsfss}
Let $\rmV_{i}$ be the value function of the policy $\pi\in\Delta_{\S}(\A)$ in the discounted MDP $(\S,\A,\rmr_{i},\rmP_{i},\beta)$. It holds:
\begin{equation*}
\begin{split}
&\norm{\rmV_{1}-\rmV_{2}}_{\infty}\leq\frac{ \norm{\rmr_{1}-\rmr_{2}}_{\infty}+\frac{\beta\norm{\rmr_{2}}_{\infty}}{1-\beta}\max_{s,a}\norm{\rmP_{1}(\cdot|s,a)-\rmP_{2}(\cdot|s,a)}_{1}}{1-\beta}. 
\end{split}
\end{equation*}
\end{lemma}
\begin{proof}
We denote:
\begin{equation*}
\rmf_{i}(s,a):=\Erw_{S'\sim \rmP_{i}(\cdot|s,a)}[V_{i}(S')].
\end{equation*}
It holds:
\begin{equation}
\label{Eq:jssshdgdgdgdfsfssfsss}
\begin{split}
&\abs{\rmV_{1}(s)-\rmV_{2}(s)}\\
&=\abs{\inn{\pi(\cdot|s)}{\rmr_{1}(s,\cdot)-\rmr_{2}(s,\cdot)+\beta[\rmf_{1}(s,\cdot)-\rmf_{2}(s,\cdot)]}}\\
&\leq \norm{\pi(\cdot|s)}_{1}\norm{\rmr_{1}(s,\cdot)-\rmr_{2}(s,\cdot)+\beta[\rmf_{1}(s,\cdot)-\rmf_{2}(s,\cdot)]}_{\infty}\\
&\leq\norm{\rmr_{1}(s,\cdot)-\rmr_{2}(s,\cdot)}_{\infty}+\beta\norm{\rmf_{1}(s,\cdot)-\rmf_{2}(s,\cdot)]}_{\infty},
\end{split}
\end{equation}
where the first inequality follows from H\"older's inequality and the second inequality follows from the fact that $\pi(\cdot|s)$ is a probability distribution and from the triangle inequality. Taking the maximum over $s\in\S$ on the both sides of above inequality, it yields:
\begin{equation*}
\norm{\rmV_{1}-\rmV_{2}}_{\infty}\leq\norm{\rmr_{1}-\rmr_{2}}_{\infty}+\beta\norm{\rmf_{1}-\rmf_{2}}_{\infty}
\end{equation*}
 Now, we compute:
\begin{equation}
\label{Eq:ajajsgsgsfsfsffsdddd}
\begin{split}
&\abs{\rmf_{1}(s,a)-\rmf_{2}(s,a)]}\\
&\leq\abs{\inn{\rmP_{1}(\cdot|s,a)}{V_{1}-V_{2}}}+\abs{\inn{\rmP_{1}(\cdot|s,a)-\rmP_{2}(\cdot|s,a)}{V_{2}}}\\
&\leq\norm{\rmP_{1}(\cdot|s,a)}_{1}\norm{V_{1}-V_{2}}_{\infty}+\norm{V_{2}}_{\infty}\norm{\rmP_{1}(\cdot|s,a)-\rmP_{2}(\cdot|s,a)}_{1}.
\end{split}
\end{equation}
Clearly, we have that $\norm{\rmP_{1}(\cdot|s,a)}_{1}=1$. Furthermore, it yields:
\begin{equation*}
\begin{split}
\abs{V_{2}(s)}&=\abs{\Erw_{\pi}\left[ \sum_{t=0}^{\infty}\beta^{t}\rmr_{2}(S_{t},A_{t})\right] }\leq\Erw_{\pi}\left[ \sum_{t=0}^{\infty}\beta^{t}\abs{\rmr_{2}(S_{t},A_{t})}\right] \\
&\leq\norm{\rmr_{2}}_{\infty}\sum_{t=0}^{\infty}\beta^{t}=\frac{\norm{\rmr_{2}}_{\infty}}{1-\beta},
\end{split}
\end{equation*}
and thus $\norm{V_{2}}_{\infty}\leq\norm{\rmr_{2}}_{\infty}/(1-\beta)$. By previous observations, we can continue the estimate \eqref{Eq:ajajsgsgsfsfsffsdddd}:
\begin{equation*}
\begin{split}
&\norm{\rmf_{1}-\rmf_{2}}_{\infty}\\
&\leq\norm{V_{1}-V_{2}}_{\infty}+\frac{\norm{\rmr_{2}}_{\infty}}{1-\beta}\max_{s,a}\norm{\rmP_{1}(\cdot|s,a)-\rmP_{2}(\cdot|s,a)}_{1},
\end{split}
\end{equation*}
and obtain:

Setting this into \eqref{Eq:jssshdgdgdgdfsfssfsss} and taking the maximum over $s$, we have:
\begin{equation*}
\begin{split}
&\norm{\rmV_{1}-\rmV_{2}}_{\infty}\leq \norm{\rmr_{1}-\rmr_{2}}_{\infty}\\
&+\beta\left[ \norm{V_{1}-V_{2}}_{\infty}+\frac{\norm{\rmr_{2}}_{\infty}}{1-\beta}\max_{s,a}\norm{\rmP_{1}(\cdot|s,a)-\rmP_{2}(\cdot|s,a)}_{1}\right]. 
\end{split}
\end{equation*}
Therefore:
\begin{equation*}
\begin{split}
&\norm{\rmV_{1}-\rmV_{2}}_{\infty}\\
&\leq\frac{ \norm{\rmr_{1}-\rmr_{2}}_{\infty}+\frac{\beta\norm{\rmr_{2}}_{\infty}}{1-\beta}\max_{s,a}\norm{\rmP_{1}(\cdot|s,a)-\rmP_{2}(\cdot|s,a)}_{1}}{1-\beta}. 
\end{split}
\end{equation*}
\end{proof}

\begin{lemma}
	\label{Lem:jaajjsshsggsfsfsffsgsfss}
Let $\X$ be a finite set, $f:\X\rightarrow\real$, and $\X^{*}_{f}:=\argmax f$. Consider the Boltzmann distribution with the inverse temperature parameter $\tau>0$ and the potential $f$:
\begin{equation*}
(\Phi_{\tau}(f))(x)=\frac{\exp\left( \frac{f(x)}{\tau}\right) }{\sum_{x'}\exp\left( \frac{f(x')}{\tau}\right)},
\end{equation*}
and the uniform distribution $\Unif(\X^{*}_{f})$ on the set of maximizer of $f$. It holds:
\begin{equation*}
\norm{\Phi_{\tau}(f)-\Unif(\X^{*}_{f})}_{1}\leq D\exp\left(-\frac{C}{\tau} \right), 
\end{equation*}
where $C,D>0$ given by:
\begin{equation*}
C:=\min_{x'\notin \X_{f}^{*}}\left[ \max f-f(x')\right] \quad\text{and}\quad D:=\sqrt{2\frac{\abs{\X}-\abs{\X_{f}^{*}}}{\abs{\X_{f}^{*}}}}
\end{equation*}
\end{lemma}
\begin{proof}
First we compute the Kullback-Leibler divergence from $\Phi_{\tau}(f)$ to $\Unif(\X^{*}_{f})$:
\begin{equation*}
\begin{split}
&\rmD(\Unif(\X^{*}_{f})|\Phi_{\tau}(f))=\Erw_{X\sim\Unif(\X^{*}_{f})}\left[ \log\left(\frac{\Unif(X)}{\Phi_{\tau}(f)(X)} \right) \right] \\
&=\frac{1}{\abs{\X^{*}_{f}}}\sum_{x\in\X^{*}_{f}}\log\left( \frac{\sum_{x'}\exp\left( \frac{f(x')}{\tau}\right)}{\abs{\X^{*}_{f}}\exp\left( \frac{f(x)}{\tau}\right)}\right)\\
&=\frac{1}{\abs{\X^{*}_{f}}}\sum_{x\in\X^{*}_{f}}\log\left( \frac{A+B}{\abs{\X^{*}_{f}}\exp\left( \frac{f(x)}{\tau}\right)}\right)\\
&=-\log(\abs{\X_{f}^{*}})-\frac{\sum_{x\in\X^{*}_{f}}f(x)}{\tau\abs{\X_{f}^{*}}}+\log\left( A\left[ \frac{B}{A}+1\right] \right) \\
&=-\log(\abs{\X_{f}^{*}})-\frac{\max f}{\tau}+\log(A)+\log\left( \frac{B}{A}+1 \right)\\
&\leq-\log(\abs{\X_{f}^{*}})-\frac{\max f}{\tau}+\log(A)+ \frac{B}{A}.
\end{split}
\end{equation*}
where:
\begin{equation*}
A:=\sum_{x'\in\X^{*}_{f}}\exp\left( \frac{f(x')}{\tau}\right)\quad\text{and}\quad B:=\sum_{x'\notin\X^{*}_{f}}\exp\left( \frac{f(x)}{\tau}\right).
\end{equation*}
The last inequality in above computation follows from the inequality $\log(x)\leq x-1$, for all $x>0$. Notice that $A=\abs{X_{f}*}\exp(\max f/\tau)$. Thus we continue above estimation:
\begin{equation*}
\begin{split}
&\rmD(\Unif(\X^{*}_{f})|\Phi_{\tau}(f))\\
&\leq -\log(\abs{\X_{f}^{*}})-\frac{\max f}{\tau}+\log(\abs{\X^{*}_{f}})+\frac{\max f}{\tau}+ \frac{B}{A}\\
&=\frac{B}{A}=\frac{\sum_{x'\notin\X^{*}_{f}}\exp\left( \frac{f(x')}{\tau}\right)}{\abs{\X_{f}^{*}}\exp(\max f/\tau)}=\frac{\sum_{x'\notin\X^{*}_{f}}\exp\left( \frac{f(x')-\max f}{\tau}\right)}{\abs{\X_{f}^{*}}}\\
&\leq \frac{\abs{\X}-\abs{\X_{f}^{*}}}{\abs{\X_{f}^{*}}}\exp\left(\frac{\max_{x'\notin \X_{f}^{*}}\left[ f(x')-\max f\right]}{\tau}  \right). 
\end{split}
\end{equation*}  
Now we apply the Pinsker's inequality to obtain the desired statement:
\begin{equation*}
\begin{split}
&\norm{\Unif(\X^{*}_{f})-\Phi_{\tau}(f)}_{1}\leq\sqrt{2\rmD(\Unif(\X^{*}_{f})|\Phi_{\tau}(f))}\\
&\leq \sqrt{2\frac{\abs{\X}-\abs{\X_{f}^{*}}}{\abs{\X_{f}^{*}}}}\exp\left(\frac{\max_{x'\notin \X_{f}^{*}}\left[ f(x')-\max f\right]}{2\tau}  \right)
\end{split}
\end{equation*}
\end{proof}
\begin{lemma}
	\label{Lem:ajajajahshssgsgsgsss}
Let be $\pi^{(1)}_{\Glob},\pi^{(2)}_{\Glob}\in\Delta(\A_{\Glob})$, and define for any $i\in\lrbrace{1,2}$:
\begin{equation*}
\begin{split}
&\rmr^{(i)}_{\Loc}(s,a_{\Loc}):=\Erw_{A_{\Glob}\sim\pi^{(i)}_{\Glob}}\left[\rmr^{\Loc}(s,a_{\Loc},A_{\Glob}) \right]\\
&\rmP^{(i)}_{\Loc}(s'|s,a_{\Loc}):=\Erw_{A_{\Glob}\sim\pi^{(i)}_{\Glob}}\left[ \rmP(s'|s,a_{\Loc},A_{\Glob})\right] 
\end{split} 
\end{equation*}
 Then it holds:
 \begin{align}
 &\norm{\rmr^{(1)}_{\Loc}-\rmr^{(2)}_{\Loc}}_{\infty}\leq \norm{\rmr^{\Loc}}_{\infty}\norm{\pi^{(1)}_{\Glob}-\pi^{(2)}_{\Glob}}_{1}\label{Eq:aajjajsgsgsgsgsfdfdfdfd1}\\
 &\norm{\rmP_{\Loc}^{(1)}(\cdot|s,a_{\Loc})-\rmP_{\Loc}^{(2)}(\cdot|s,a_{\Loc})}_{1}\leq
 \norm{\pi^{(1)}_{\Glob}-\pi^{(2)}_{\Glob}}_{1}\label{Eq:aajjajsgsgsgsgsfdfdfdfd2}
 \end{align}
Moreover, let be $\pi\in\Delta_{\S}(\A_{\Loc})$, and let $\rmT^{(i)}$
 be the Bellman operator of a policy $\pi\in\Delta_{\S}(\A_{\Loc})$ in the discounted MDP $(\S,\A_{\Loc},\rmr_{\Loc}^{(i)},\rmP_{\Loc}^{(i)},\beta)$ with the policy $\pi\in\Delta_{\S}(\A_{\Loc})$. Then it holds:
\begin{equation}
\label{Eq:aajjajsgsgsgsgsfdfdfdfd3}
\norm{\rmT^{(1)}\rmQ-\rmT^{(2)}\rmQ}_{\infty}\leq\left( \norm{\rmr_{\Loc}}_{\infty}+\beta \norm{Q}_{\infty}\right) \norm{\pi^{(1)}_{\Glob}-\pi^{(2)}_{\Glob}}_{1}.
\end{equation}
Let be $\rmV^{(i)}_{\Loc}$ be the value function of the policy $\pi_{\Loc}\in\Delta_{\S}(\A_{\Loc})$ in the discounted MDP $(\S,\A_{\Loc},\rmr^{(i)}_{\Loc},\rmP^{(i)}_{\Loc})$. Then:
\begin{equation}
\label{Eq:ajajsgsgsffsfddsfsdsdsfsdss}
\norm{\rmV^{(1)}_{\Loc}-\rmV^{(2)}_{\Loc}}_{\infty}\leq\frac{\norm{\rmr_{\Loc}}_{\infty}}{(1-\beta)^2}\norm{\pi_{\Glob}^{(1)}-\pi_{\Glob}^{(2)}}_{\infty}.
\end{equation}
Furthermore, the previous statements hold true if $\rmT^{(i)}$ is the optimal Bellman operator of $(\S,\A_{\Loc},\rmr_{\Loc}^{(i)},\rmP_{\Loc}^{(i)},\beta)$.
\end{lemma}
\begin{proof}
The inequality \eqref{Eq:aajjajsgsgsgsgsfdfdfdfd1} follows from H\"older's inequality:
\begin{equation*}
\begin{split}
&\abs{\rmr_{\Loc}^{(1)}(s,a_{\Loc})-\rmr_{\Loc}^{(2)}(s,a_{\Loc})}\leq \abs{\inn{\pi^{(1)}_{\Glob}-\pi^{(2)}_{\Glob}}{\rmr^{\Loc}(s,a_{\Loc},\cdot)}}\\
&\leq \norm{\rmr^{\Loc}(s,a_{\Loc},\cdot)}_{\infty}\norm{\pi^{(1)}_{\Glob}-\pi^{(2)}_{\Glob}}_{1}\leq \norm{\rmr^{\Loc}}_{\infty}\norm{\pi^{(1)}_{\Glob}-\pi^{(2)}_{\Glob}}_{1}
\end{split}
\end{equation*}
Similarly, we obtain the inequality \eqref{Eq:aajjajsgsgsgsgsfdfdfdfd2} by the following computation:
\begin{equation*}
\begin{split}
&\abs{\rmP_{\Loc}^{(1)}(s'|s,a_{\Loc})-\rmP_{\Loc}^{(2)}(s'|s,a_{\Loc})}= \abs{\inn{\pi^{(1)}_{\Glob}-\pi^{(2)}_{\Glob}}{\rmP(s'|s,a_{\Loc},\cdot)}}\\
&\leq\max_{a_{\Glob}} \rmP(s'|s,a_{\Loc},a_{\Glob}) \norm{\pi^{(1)}_{\Glob}-\pi^{(2)}_{\Glob}}_{1}\leq\norm{\pi^{(1)}_{\Glob}-\pi^{(2)}_{\Glob}}_{1}.
\end{split}
\end{equation*}

Now, we show \eqref{Eq:aajjajsgsgsgsgsfdfdfdfd3}.
For any $\rmQ$:
\begin{equation*}
\begin{split}
&\abs{(\rmT^{(1)}\rmQ)(s,a_{\Loc})-(\rmT^{(2)}Q)(s,a_{\Loc})}\\
&\leq \abs{\rmr^{(1)}_{\Loc}(s,a_{\Loc})-\rmr^{(2)}_{\Loc}(s,a_{\Loc})}\\
&+\beta\abs{\sum_{s'}\left( \rmP_{\Loc}^{(1)}(s'|s,a_{\Loc})-\rmP_{\Loc}^{(2)}(s'|s,a_{\Loc})\right) \Erw_{A'_{\Loc}\sim\pi(\cdot|s')}[\rmQ(s',A'_{\Loc})]}
\end{split}
\end{equation*}
By means of H\"older's inequality we can estimate the second summand in the right hand side of above inequality:
\begin{equation*}
\begin{split}
&\abs{\sum_{s'}\left( \rmP_{\Loc}^{(1)}(s'|s,a_{\Loc})-\rmP_{\Loc}^{(2)}(s'|s,a_{\Loc})\right) \Erw_{A'_{\Loc}\sim\pi(\cdot|s')}[\rmQ(s',A'_{\Loc})]}\\
&\leq \norm{\rmP_{\Loc}^{(1)}(\cdot|s,a_{\Loc})-\rmP_{\Loc}^{(2)}(\cdot|s,a_{\Loc})}_{1}\max_{s}\abs{\Erw_{A'_{\Loc}\sim\pi(\cdot|s)}[\rmQ(s,A'_{\Loc})]}\\
&\leq \norm{\rmP_{\Loc}^{(1)}(\cdot|s,a_{\Loc})-\rmP_{\Loc}^{(2)}(\cdot|s,a_{\Loc})}_{1}\norm{Q}_{\infty}.
\end{split}
\end{equation*}
Setting this estimate into the previous inequality and taking maximum over $(s,a_{\Loc})$, we obtain:
\begin{equation*}
\begin{split}
&\norm{\rmT^{(1)}\rmQ-\rmT^{(2)}\rmQ}_{\infty}\\
&\leq \norm{\rmr^{(1)}_{\Loc}-\rmr^{(2)}_{\Loc}}_{\infty}+\beta\norm{Q}_{\infty}\max_{s,a_{\Loc}}\norm{\rmP_{\Loc}^{(1)}(\cdot|s,a_{\Loc})-\rmP_{\Loc}^{(2)}(\cdot|s,a_{\Loc})}_{1}
\end{split}
\end{equation*}
The desired statement yields by inserting the inequalities \eqref{Eq:aajjajsgsgsgsgsfdfdfdfd1} and \eqref{Eq:aajjajsgsgsgsgsfdfdfdfd2} into above estimate.

To show the last inequality \eqref{Eq:ajajsgsgsffsfddsfsdsdsfsdss}, notice that by Lemma \ref{Lem:ajajasshshsgsgsgsfsfss}, we have:
\begin{equation*}
\begin{split}
&\norm{\rmV^{(1)}_{\Loc}-\rmV^{(2)}_{\Loc}}_{\infty}\\
&\leq\frac{ \norm{\rmr^{(1)}_{\Loc}-\rmr^{(2)}_{\Loc}}_{\infty}+\frac{\norm{\rmr^{(2)}_{\Loc}\beta}_{\infty}}{1-\beta}\max_{s,a}\norm{\rmP^{(1)}_{\Loc}(\cdot|s,a_{\Loc})-\rmP^{(2)}_{\Loc}(\cdot|s,a_{\Loc})}_{1}}{1-\beta}. 
\end{split}
\end{equation*}
Clearly, $\norm{\rmr^{(2)}_{\Loc}}_{\infty}\leq\norm{\rmr^{\Loc}}$. Setting this estimate and the inequalities \eqref{Eq:aajjajsgsgsgsgsfdfdfdfd1} and \eqref{Eq:aajjajsgsgsgsgsfdfdfdfd2} into above inequality, we obtain the desired statement. Finally, the last statement concerning to the optimal Bellman operator can easily be shown by similar way as above.  

\end{proof}



\bibliographystyle{IEEEtran}
\bibliography{BibRL}

\end{document}